\documentclass{article}

\PassOptionsToPackage{numbers, compress}{natbib}

\usepackage{titletoc}

 \usepackage[preprint]{neurips_2026}

\usepackage[utf8]{inputenc} 
\usepackage[T1]{fontenc}    
\usepackage{hyperref}       
\usepackage{url}            
\usepackage{booktabs}       
\usepackage{amsfonts}       
\usepackage{nicefrac}       
\usepackage{microtype}      
\usepackage[table]{xcolor}  
\usepackage{amsmath}
\usepackage{amssymb}
\usepackage{amsthm}
\usepackage{mathtools}
\usepackage{algorithm}
\usepackage{algorithmic}
\usepackage{bm}
\usepackage{graphicx}
\usepackage{enumitem}
\usepackage{array}
\usepackage{multirow}
\usepackage{subcaption}

\newtheorem{theorem}{Theorem}
\newtheorem{proposition}[theorem]{Proposition}

\newtheorem{observation}[theorem]{Observation}

\newcommand{\R}{\mathbb{R}}
\newcommand{\E}{\mathbb{E}}

\newcommand{\N}{\mathcal{N}}

\newcommand{\loss}{\mathcal{L}}
\newcommand{\pdata}{p_{\mathrm{data}}}
\newcommand{\dd}{\mathrm{d}}

\definecolor{lightpink}{RGB}{255, 230, 230}
\newcommand{\epsrow}{\rowcolor{lightpink}}
\newcommand{\epscell}[1]{\cellcolor{lightpink}#1}

\usepackage{tikz}
\usetikzlibrary{arrows.meta,positioning,calc,fit,backgrounds,shapes.geometric}
\usepackage{xcolor}

\title{Exact Posterior Score Estimation \\for Solving Linear Inverse Problems}

\author{
  Abbas Mammadov\thanks{Equal contribution.} \\
  University of Oxford \\
  \And
  Ozgur Kara\footnotemark[1] \\
  UIUC \\
  \And
  Kaan Oktay \\
  fal \\
  \And
  Iskander Azangulov \\
  University of Oxford \\
  \AND
  Adil Kaan Akan \\
  fal \\
  \And
  Hyungjin Chung \\
  EverEx \\
  \And
  James Matthew Rehg \\
  UIUC \\
  \And
  Yee Whye Teh \\
  University of Oxford \\
}

\begin{document}

\maketitle

\begin{abstract}
Diffusion and flow-based models learn powerful data priors by training a denoiser to reverse Gaussian corruption. To use this prior to solve a linear inverse problem, one needs to sample from the posterior, but the score that the prior provides is the unconditional score, not the posterior score. Existing methods either steer a fixed pretrained denoiser with approximate measurement-matching corrections, or train a conditional restoration model that abandons the denoising structure of the prior. We derive the exact posterior score in closed form for linear Gaussian inverse problems under general Gaussian interpolants, and show that posterior sampling reduces to a denoising problem at an operator-dependent shifted pivot under an anisotropic noise covariance. We turn this identity into \textbf{Exact Posterior Score (EPS)}, a denoising training objective that preserves the input/output structure of standard pretraining and can therefore be trained from scratch or fine-tuned from a pretrained denoiser. At inference, EPS uses the same sampler as the underlying backbone, with no likelihood gradients or projections. We evaluate EPS on five linear inverse problems across FFHQ and ImageNet, where it outperforms training-free and training-based baselines on fidelity, perceptual, and distributional metrics, while using roughly an order of magnitude fewer denoiser evaluations than gradient-based posterior samplers.
\end{abstract}

\section{Introduction}
\label{sec:intro}

Linear inverse problems, in which an unknown signal $x_0$ must be recovered from a noisy linear measurement $y = A x_0 + \eta$ with known forward operator $A$ and observation noise $\eta$, are pervasive across imaging and the sciences, including compressive sensing \cite{donoho2006compressed,candes2006robust}, accelerated medical imaging \cite{lustig2007sparse,knoll2020fastmri}, super-resolution \cite{dong2014srcnn,wang2021realesrgan}, deblurring \cite{nah2017deepdeblur,kupyn2018deblurgan}, and inpainting in computational photography \cite{bertalmio2000image,pathak2016context}. The forward operator $A$ is typically ill-conditioned or rank-deficient, so many candidate signals are consistent with the same observation, and the right object to recover is the posterior $p(x_0 | y)$ rather than any single point estimate. The posterior captures uncertainty over reconstructions, supports downstream decisions, and exposes the trade-off between data fidelity and prior plausibility.

Diffusion and flow-based generative models offer powerful, expressive data priors for this task, learning a denoising trajectory from noise back to clean samples \cite{sohldickstein2015deep,ho2020denoising,song2020score,lipman2022flow,albergo2023stochastic,liu2022flow,karras2022elucidating}. The central question is how to turn this trajectory into a sampler from $p(x_0 | y)$. The reverse-time sampler needs the \emph{posterior} score $\nabla_{x_t} \log p(x_t | y)$, not the unconditional prior score $\nabla_{x_t} \log p(x_t)$ that diffusion training provides. Replacing the former by an approximation introduces bias at every step, which compounds into oversmoothing, hallucinated structure, or poorly calibrated uncertainty. Existing methods fall into two broad camps.

\begin{figure}[t]
\vspace{-1mm}
\centering
\includegraphics[width=0.9\textwidth]{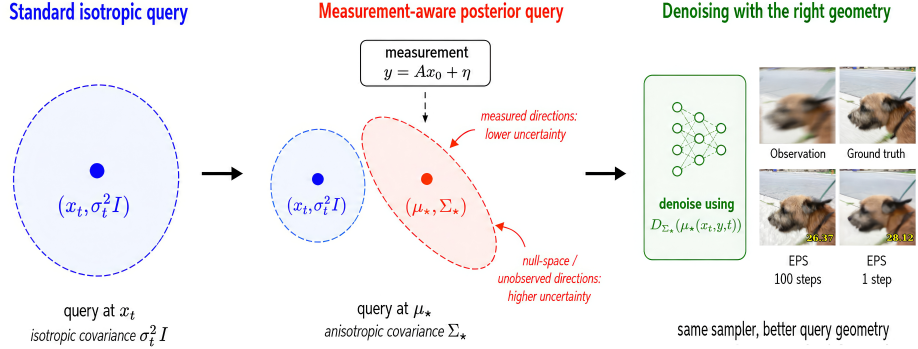}
\caption{
\textbf{EPS turns posterior sampling into denoising with the right query geometry.}
Instead of denoising an isotropic query at $x_t$, the measurement shifts the query to the posterior pivot $\mu_\star$ and reshapes the noise into an anisotropic covariance $\Sigma_\star$. Measured directions become more certain, while unobserved directions remain uncertain. EPS trains a denoiser for this anisotropic geometry and reuses the backbone's unconditional sampler unchanged. The first step of the resulting sampler corresponds to an estimate of the posterior mean $\mathbb{E}[x_0|y]$, which typically has higher PSNR but is over-smoothed, while the sample produced at the end (in this case, 100 steps) has more details.
}
\vspace{-6mm}
\label{fig:eps-main}
\end{figure}

\emph{Training-free methods} keep a pretrained denoising backbone fixed and add a measurement-matching update at each reverse step. The prototypical example is Diffusion Posterior Sampling (DPS)~\cite{chung2022diffusion}, which differentiates a measurement loss through the unconditional denoiser, with variants using projections, denoised estimates,  or task-specific correction rules~\cite{chung2022diffusion,wang2022zero,song2023pseudoinverseguided,kawar2022denoising,zhang2025improving,he2023manifoldpreservingguideddiffusion,chung2022improving,mardani2023variational}. This route is attractive because it is zero-shot and inherits the strong unconditional prior of a pretrained model. However, the added update is only an approximation to the true measurement-matching score, and even moment-matching variants that track anisotropic uncertainty in $p(x_0 | x_t)$ \cite{rozet2024learning,rout2024beyond,daras2024survey} only refine the unconditional denoising query. Asymptotically exact alternatives based on sequential Monte Carlo~\cite{wu2023practical, cardoso2023monte, dou2024diffusion} avoid this approximation, but at the cost of running many particle trajectories per observation.

\emph{Training-based methods} sidestep the approximation question by training a new model specifically for the inverse problem, with the measurement $y$ as input. This family includes conditional diffusion models that learn a measurement-conditional score~\cite{saharia2022palette,batzolis2021conditional,elata2025invfusion}, bridge-based methods that build a trajectory directly from $y$ to the data~\cite{delbracio2023inversion,liu20232}, and methods that distill a posterior sampler from a pretrained diffusion prior~\cite{mammadov2024amortized, feng2023score}. In all cases, the network is exposed to the raw measurement rather than to the geometry of the exact posterior denoising query, so it must learn the operator dependence end-to-end.

In contrast, we observe that for linear Gaussian inverse problems the exact posterior score has a closed form, with a simple structural meaning. As Figure~\ref{fig:eps-main} illustrates, posterior sampling is still a denoising problem, but with a measurement-aware input and an operator-dependent anisotropic noise covariance. We use this identity to define \textbf{Exact Posterior Score (EPS)}, a denoising training objective whose target and loss match those of standard pretraining, with the input replaced by a measurement-dependent pivot. EPS can therefore be trained from scratch or fine-tuned efficiently from a pretrained denoiser. At inference, it runs the underlying backbone's sampler unchanged, with no likelihood gradients, projections, or inner optimization.

Our contributions are as follows.
\begin{itemize}[leftmargin=*,itemsep=1pt,topsep=2pt]
    \item We derive the exact posterior score for linear Gaussian inverse problems under general Gaussian interpolants, and show that posterior sampling reduces to denoising at an operator-dependent shifted pivot under an anisotropic covariance. We also pinpoint where existing approximate-guidance methods deviate from this exact identity.
    \item We turn the identity into EPS, a denoising training objective and sampling algorithm that preserves the structure of standard pretraining while incorporating the exact posterior geometry. EPS can be trained from scratch or fine-tuned from a pretrained checkpoint, and it uses the underlying backbone's sampler at inference.
    \item We evaluate EPS on five linear inverse problems across FFHQ and ImageNet, reporting pointwise fidelity, perceptual quality, and distributional calibration metrics, and find consistent improvements over both training-free and training-based baselines at substantially smaller sampling budgets.
\end{itemize}

\section{Background}
\label{sec:background}

\subsection{Generative Models as Denoising Trajectories}
\label{sec:bg:interpolants}
Let $x_0\sim\pdata$ and $\epsilon\sim\N(0,I_d)$. A stochastic interpolant defines
\begin{equation}
    x_t = \alpha_t x_0 + \beta_t \epsilon,
    \qquad t\in[0,T],
    \label{eq:interpolant}
\end{equation}
with $(\alpha_0,\beta_0)=(1,0)$ and $(\alpha_T,\beta_T)=(0,1)$ in the data-to-noise convention. Different choices of $(\alpha_t,\beta_t)$ recover variance-preserving diffusion \cite{ho2020denoising,song2020score}, rectified flow \cite{liu2022flow}, and EDM \cite{karras2022elucidating}, with corresponding network parameterizations as a score, noise predictor, velocity, or EDM-preconditioned denoiser; the conversion between them is deterministic once $(\alpha_t,\beta_t)$ are fixed, so we write all backbones through a learned denoiser $D_\theta$. The forward kernel is
\begin{equation}
    p(x_t|x_0)=\N(x_t;\alpha_t x_0,\beta_t^2 I_d).
\end{equation}
The marginal score $s_t(x)=\nabla_x\log p(x)$ is equivalent to the optimal denoiser through Tweedie's identity,
\begin{equation}
    D_t(x) \coloneqq \E[x_0|x_t=x]
    = \frac{1}{\alpha_t}\big(x+\beta_t^2 s_t(x)\big),
    \label{eq:tweedie-uncond}
\end{equation}
and the reverse velocity follows from the same identity:
\begin{equation}
    v_t(x)=\E[\dot\alpha_t x_0+\dot\beta_t\epsilon\mid x_t=x]
    =\frac{\dot\beta_t}{\beta_t}\,x + \left(\dot\alpha_t-\frac{\alpha_t\dot\beta_t}{\beta_t}\right)D_t(x).
    \label{eq:velocity-uncond}
\end{equation}
This viewpoint sets up EPS: if the posterior denoiser is known, the posterior score and posterior velocity follow immediately from the same identities, and the base model's sampler can be reused.

\subsection{Inverse Problems and Approximate Posterior Sampling}
\label{sec:bg:inverse}
We observe
\begin{equation}
    y=A x_0+\eta,\qquad \eta\sim\N(0,\sigma_y^2 I_m),
    \label{eq:ip}
\end{equation}
for a known linear operator $A\in\R^{m\times d}$. This notation covers masks, downsampling, and convolutional blur operators, and includes rank-deficient settings where many signals are consistent with the same observation. The target is the posterior $p(x_0|y)\propto p(y|x_0)\pdata(x_0)$. A reverse sampler should therefore use
\begin{equation}
    \nabla_{x_t}\log p(x_t|y)=\nabla_{x_t}\log p(x_t)+\nabla_{x_t}\log p(y|x_t),
    \label{eq:bayes-score}
\end{equation}
where the second term is the measurement-matching score. Zero-shot solvers approximate it with the template $\nabla_{x_t}\log p(y|x_t)\approx -L_t M_t / G_t$ \cite{daras2024survey}, where $M_t$ is a measurement residual, $L_t$ lifts it back to sample space, and $G_t$ is the guidance strength.

\section{Exact Posterior Score (EPS)}
\label{sec:method}

We now derive the posterior score and convert it into a training objective. The derivation has two pieces. First, a Gaussian product identifies the correct pivot and covariance. Second, an anisotropic Tweedie identity turns the resulting smoothed density into a denoiser.

\subsection{Anisotropic Tweedie Identity}
\label{sec:method:tweedie}

For a positive definite covariance matrix $\Sigma$, define the Gaussian-smoothed data density
\begin{equation}
    p_{\mathrm{data}}^{\Sigma}(\mu)
    = \int \N(\mu;x_0,\Sigma)\,\pdata(x_0)\,\dd x_0,
\end{equation}
and the corresponding optimal denoiser
\begin{equation}
    D_{\Sigma}(\mu)
    = \E[x_0| x_0+\xi=\mu],\qquad \xi\sim\N(0,\Sigma).
\end{equation}
Then, by the anisotropic form of Tweedie's formula \cite{robbins1992empirical,efron2011tweedie},
\begin{equation}
    D_{\Sigma}(\mu)=\mu+\Sigma\nabla_\mu\log p_{\mathrm{data}}^{\Sigma}(\mu),
    \label{eq:anisotropic-tweedie}
\end{equation}
which is the usual Tweedie formula when $\Sigma$ is a scalar multiple of the identity. EPS uses this identity with a covariance that is not chosen by hand, but rather is derived from the inverse problem formulation.

\subsection{Closed-Form Posterior Score}
\label{sec:method:theory}

The posterior marginal can be written as
\begin{equation}
    p(x_t|y)\;\propto\;\int p(x_t|x_0)p(y|x_0)\pdata(x_0)\,\dd x_0.
    \label{eq:posterior-marginal-intro}
\end{equation}
Both factors inside the integral are Gaussian in $x_0$. Completing the square gives the following result.

\begin{theorem}[Exact posterior score]
\label{thm:main}
Under the linear Gaussian inverse problem \eqref{eq:ip} and the interpolant \eqref{eq:interpolant}, the posterior score at time $t$ is
\begin{equation}
    \nabla_{x_t}\log p(x_t|y)
    =\frac{1}{\beta_t^2}\Big(\alpha_t D_{\Sigma_\star(t)}(\mu_\star(x_t,y,t))-x_t\Big),
    \label{eq:posterior-score-main}
\end{equation}
where
\begin{equation}
    \Sigma_\star(t)
    =\left(\frac{\alpha_t^2}{\beta_t^2}I_d+\frac{1}{\sigma_y^2}A^\top A\right)^{-1},
    \qquad
    \mu_\star(x_t,y,t)
    =\Sigma_\star(t)\left(\frac{\alpha_t}{\beta_t^2}x_t+\frac{1}{\sigma_y^2}A^\top y\right).
    \label{eq:sigma-mu-star}
\end{equation}
Equivalently, $D_{\Sigma_\star(t)}(\mu_\star(x_t,y,t))=\E[x_0|x_t,y]$.
\end{theorem}
The proof is given in Appendix~\ref{app:derivation:main-proof}.
The theorem says that posterior sampling is still a denoising problem, but not the isotropic one seen in unconditional pretraining. The denoiser must be queried at a measurement-aware input $\mu_\star$ under a measurement-aware anisotropic noise covariance $\Sigma_\star$.

\begin{proposition}[Posterior velocity]
\label{prop:posterior-velocity}
The posterior velocity associated with the interpolant \eqref{eq:interpolant} is
\begin{equation}
    v_t^y(x_t)
    =\E[\dot\alpha_t x_0+\dot\beta_t\epsilon| x_t,y]
    =\frac{\dot\beta_t}{\beta_t}x_t+\left(\dot\alpha_t-\frac{\alpha_t\dot\beta_t}{\beta_t}\right)D_{\Sigma_\star(t)}(\mu_\star(x_t,y,t)).
    \label{eq:posterior-velocity}
\end{equation}
Thus estimating the exact posterior denoiser is equivalent to estimating the exact posterior flow.
\end{proposition}
The proof is given in Appendix~\ref{app:derivation:velocity-proof}.

\paragraph{Posterior pivot.} We call $\mu_\star$ the \emph{posterior pivot} because the proof of Theorem~\ref{thm:main} (Appendix~\ref{app:derivation:main-proof}) shows that the joint quadratic form in $(x_t,y,x_0)$ pivots about $\mu_\star(x_t,y,t)$: Completing the square sends the entire dependence on $x_0$ into a single Gaussian centered at $\mu_\star$ with covariance $\Sigma_\star(t)$, while $x_t$ and $y$ enter only through this pivot and a multiplicative normalizer. Equivalently, $\mu_\star$ is the precision-weighted Bayesian fusion of the current state and the measurement under the two Gaussian likelihoods, before the data prior $\pdata$ is folded in by the denoiser $D_{\Sigma_\star(t)}$. The pivot is therefore the only summary statistic of $(x_t,y)$ that the posterior denoiser needs to see.

\paragraph{Computing $\mu_\star$.} Although \eqref{eq:sigma-mu-star} involves inverting a $d \times d$ matrix in general, every operator we consider admits a fast structured solve. For binary inpainting masks $A^\top A$ is diagonal, for downsampling it is block-diagonal, and for circular convolutions used in deblurring it is diagonalized by the FFT. The per-step cost of computing $\mu_\star$ is therefore negligible relative to a denoiser forward pass. For more general operators, $\mu_\star$ can still be obtained efficiently via conjugate gradient applied to the symmetric positive-definite system $(\alpha_t^2/\beta_t^2\, I + \sigma_y^{-2} A^\top A)\,\mu_\star = (\alpha_t/\beta_t^2)\, x_t + \sigma_y^{-2} A^\top y$, which only requires matrix-vector products with $A$ and $A^\top$. We measure these costs directly in Appendix~\ref{app:additional-experiments:runtime}, where the structured $\mu_\star$ solve adds only sub-millisecond overhead per sampling step.

\subsection{What was missing in Training-Free Methods}
\label{sec:method:exactness}

Theorem~\ref{thm:main} also pinpoints what existing training-free methods miss. Combining \eqref{eq:posterior-score-main} with the unconditional Tweedie identity \eqref{eq:tweedie-uncond}, the measurement-matching score can be written as a difference of two denoisers,
\begin{equation}
\small
    \nabla_{x_t}\log p(y | x_t)
    \;=\; \nabla_{x_t}\log p(x_t | y) - \nabla_{x_t}\log p(x_t)
    \;=\; \frac{\alpha_t}{\beta_t^2}\Big(\, D_{\Sigma_\star(t)}(\mu_\star(x_t,y,t)) \;-\; D_t(x_t) \,\Big).
    \label{eq:exact-guidance}
\end{equation}
The exact guidance is the gap between the \emph{posterior} denoiser evaluated at the pivot $\mu_\star$ under the anisotropic covariance $\Sigma_\star(t)$, and the unconditional denoiser evaluated at $x_t$. Methods that follow the template from~\cite{daras2024survey}, including DPS, DDNM, $\Pi$GDM, and moment-matching variants \cite{chung2022diffusion,wang2022zero,song2023pseudoinverseguided,rozet2024learning,rout2024beyond}, all approximate the first denoiser using only the second, by differentiating a measurement loss through $D_t(x_t)$, projecting $D_t(x_t)$ onto an affine subspace, or fitting a Gaussian to $p(x_0 | x_t)$ and comparing it to $y$. They thus evaluate the network at a \emph{different input} than the exact identity, querying $D_\theta$ at $x_t$ rather than at the pivot $\mu_\star$, which itself depends on $y$. Even moment-matching methods that use anisotropic information approximate $p(x_0 | x_t)$, the denoising distribution \emph{before} the measurement is incorporated, whereas the exact object is $p(x_0 | x_t, y)$, the denoising distribution \emph{after} the measurement is fused into the kernel. The two coincide only in degenerate cases such as isotropic $A^\top A$ or high-noise limits, but not generically.

\subsection{The EPS Training Objective}
\label{sec:method:objective}

Theorem~\ref{thm:main} reduces posterior sampling to a single object, the anisotropic posterior denoiser $D_{\Sigma_\star(t)}(\mu_\star(x_t,y,t))$. Two of the three quantities involved are analytically tractable. Given $(x_t, y, t)$ and the operator parameters $(A, \sigma_y)$, the pivot $\mu_\star(x_t,y,t)$ and covariance $\Sigma_\star(t)$ are deterministic, closed-form functions defined by \eqref{eq:sigma-mu-star}, and for the structured operators we consider, both can be computed in essentially the cost of an FFT or an element-wise solve (Section~\ref{sec:method:theory}). What we cannot compute analytically is the denoiser itself. The expression $D_{\Sigma_\star(t)}(\mu) = \E[x_0 | x_0 + \xi = \mu]$ with $\xi \sim \N(0, \Sigma_\star(t))$ requires the data distribution $\pdata$, which is accessible only through data samples, while the pretrained unconditional denoiser was trained with isotropic noise so is not directly applicable (but can be used to initialize the EPS training). 

Following standard approaches for training diffusion models, we therefore learn it by regression. 
To enable efficient noising of $x_0$ using the anistropic noise covariance, we note the following result: 

\begin{proposition}[Isotropic simulation of the anisotropic pivot]
\label{prop:isotropic-pivot}
Let $x_t=\alpha_t x_0+\beta_t\epsilon$ with $\epsilon\sim\N(0,I_d)$ and
$y=Ax_0+\sigma_y\eta$ with $\eta\sim\N(0,I_m)$, independently. Define
$\mu_\star(x_t,y,t)$ and $\Sigma_\star(t)$ as in \eqref{eq:sigma-mu-star}. Then,
conditional on $x_0$,
\begin{equation}
    \mu_\star(x_t,y,t)|x_0
    \sim \N(x_0,\Sigma_\star(t)).
    \label{eq:isotropic-pivot-law}
\end{equation}
Thus the anisotropic corruption required by the exact posterior denoiser is induced
by the closed-form pivot construction itself; it is not necessary to sample
anisotropic noise directly.
\end{proposition}
The proof is given in Appendix~\ref{app:derivation:isotropic-pivot}. This result is related to existing GP and linear regression literature (Appendix~\ref{app:derivation:gp}), and shows that the required anistropic noising can be computed efficiently.

\paragraph{The objective.} EPS regresses a denoising network $D_\theta$ onto clean targets given the pivot input,
\begin{equation}
    \loss_{\mathrm{EPS}}(\theta)
    =\E_{x_0,y,t,\epsilon}\!\left[w(t)\,\left\|D_\theta(\mu_\star(x_t,y,t),\,y,\,t)-x_0\right\|^2\right],
    \label{eq:eps-loss}
\end{equation}
where $x_t=\alpha_t x_0+\beta_t\epsilon$ and $y \sim \N(A x_0, \sigma_y^2 I_m)$. Standard arguments show that the squared-loss minimizer of \eqref{eq:eps-loss} is $\E[x_0 | \mu_\star, y, t]$, which by Theorem~\ref{thm:main} equals $D_{\Sigma_\star(t)}(\mu_\star) = \E[x_0 | x_t, y]$. Once trained, the posterior score and posterior velocity follow without further effort from \eqref{eq:posterior-score-main} and \eqref{eq:posterior-velocity}.
While strictly unnecessary, note that EPS also passes $y$ to the learned denoiser so that the network is explicitly conditioned on the observed measurement while learning the posterior denoising map; for a fixed operator, the closed-form pivot is a sufficient statistic, but conditioning on $y$ makes the dependence on the particular inverse-problem instance explicit in the learned model, and slightly improves results (Appendix~\ref{app:additional-experiments:inputs}).

\paragraph{The upshot.} The minimizer of \eqref{eq:eps-loss} has the same target type (a clean image $x_0 \in \R^d$) and the same squared-loss regression structure as the standard pretrained denoiser $D_{\theta_0}(x_t, t)$. The structural changes are (i) the input is the pivot $\mu_\star$ rather than $x_t$, and (ii) implicitly, through the pivot construction, the input is corrupted by anisotropic noise of covariance $\Sigma_\star(t)$ rather than isotropic noise of scale $\beta_t$. Why does this  matter? At the same input $\mu_\star$, the unconditional denoiser $D_{\theta_0}$ would return a biased estimate, because it implicitly assumes its input is corrupted by isotropic noise, whereas the true noise on $\mu_\star$ is operator-dependent and anisotropic. The network must therefore learn how to denoise under this measurement-induced anisotropic geometry.

Empirically, when warm-started from the pretrained unconditional denoiser, EPS converges in a small fraction of the iterations needed by other training-based posterior solvers (Appendix~\ref{app:additional-experiments:warmstart}). We attribute this to its proximity to the pretraining task. Conditional  methods such as Palette~\citep{saharia2022palette} and InvFusion~\citep{elata2025invfusion} preserve the noise schedule and noisy intermediates of pretraining but condition the score on the raw measurement, so the network must spend capacity learning the operator dependence on top of the pretrained denoising prior. Bridge methods such as InDI~\citep{delbracio2023inversion} and I2SB~\citep{liu20232} go further by replacing the noise-to-data forward process with a measurement-to-data one, so the network must learn a different conditional mapping from scratch. EPS preserves the forward process and the per-time denoising query, and only adapts to operator-induced anisotropic geometry $\Sigma_\star(t)$.

\begin{algorithm}[t]
\caption{EPS training}
\label{alg:eps-train}
\begin{algorithmic}[1]
\REQUIRE Pretrained denoiser $D_{\theta_0}$ (or random init), data distribution $\pdata$, operator distribution $p(A)$, observation noise $\sigma_y$, noise schedule $(\alpha_t,\beta_t)$.
\STATE Initialize $\theta\leftarrow\theta_0$ (or randomly).
\WHILE{not converged}
    \STATE Sample $x_0\sim\pdata$, $A\sim p(A)$, $t\sim p(t)$, $\epsilon\sim\N(0,I_d)$, and $\eta\sim\N(0,I_m)$.
    \STATE Form $y\leftarrow A x_0+\sigma_y\eta$ and $x_t\leftarrow \alpha_t x_0+\beta_t\epsilon$.
    \STATE Compute $\Sigma_\star(t)$ and $\mu_\star(x_t,y,t)$ from \eqref{eq:sigma-mu-star} via the structured solve for $A$.
    \STATE Evaluate the posterior denoising loss $\loss=w(t)\|D_\theta(\mu_\star,y,t)-x_0\|^2$.
    \STATE Update $\theta$ by gradient descent on $\loss$, and update EMA weights if used by the base sampler.
\ENDWHILE
\RETURN Trained denoiser $D_\theta$.
\end{algorithmic}
\end{algorithm}

\subsection{Sampling}
\label{sec:method:one-step}

At inference, EPS uses the same deterministic or stochastic sampler as the underlying diffusion backbone, replacing every denoiser call by $D_\theta(\mu_\star(x_t,y,t),y,t)$. No likelihood gradient, projection, or inner optimization is required during sampling. Since $\mu_\star$ is obtained by a structured linear solve, the per-step overhead is negligible relative to a denoiser forward pass (see Section~\ref{sec:method:theory}).

\paragraph{The High-Noise Limit.} Theorem~\ref{thm:main} also characterizes the sampler at the start of the reverse trajectory, where the noise scale $\sigma_t$ is largest. In this regime both the pivot and its anisotropic denoiser take a particularly simple form:

\begin{observation}[High-noise posterior-mean limit]
\label{obs:one-step-limit}
With EDM parameterization $\alpha_t=1$, $\beta_t=\sigma_t$, and let
$P_{\mathcal N(A)}=I-A^\dagger A$ be the orthogonal projector onto the nullspace
of $A$. Then, as $\sigma_t\to\infty$,
\begin{equation}
    \mu_\star(x_t,y,t)
    \longrightarrow
    A^\dagger y + P_{\mathcal N(A)}x_t .
    \label{eq:high-noise-pivot}
\end{equation}
Moreover, the corresponding anisotropic denoiser satisfies
\begin{equation}
    D_{\Sigma_\star(t)}(\mu_\star(x_t,y,t))
    \longrightarrow
    \E[x_0|y].
    \label{eq:high-noise-denoiser}
\end{equation}
Thus a single high-noise EPS denoiser evaluation is a posterior-mean estimator.
\end{observation}
The proof is given in Appendix~\ref{app:derivation:one-step}. The pivot limit \eqref{eq:high-noise-pivot} is the EPS-specific part of the statement: at the start of sampling, the network is queried at the pseudo-inverse reconstruction $A^\dagger y$ plus pure noise in the nullspace of $A$. The posterior-mean limit \eqref{eq:high-noise-denoiser}, by contrast, is generic rather than a contribution of EPS: by Theorem~\ref{thm:main} it is equivalent to $\mathbb{E}[x_0|x_t,y]\to\mathbb{E}[x_0|y]$, which holds simply because $x_t$ carries vanishing information about $x_0$ as $\sigma_t\to\infty$. Under perfect learning, the same is therefore true of any training-based method whose denoiser targets $\mathbb{E}[x_0|x_t,y]$ along the same path, e.g., Palette~\cite{saharia2022palette}; Appendix~\ref{app:additional-experiments:palette-1step} confirms this empirically. We nevertheless find it a useful way to read the sampling path of all such methods: it starts at the posterior mean $\mathbb{E}[x_0|y]$ and ends at a sample from $p(x_0|y)$.

\begin{figure}[t]
\centering
\includegraphics[width=\linewidth]{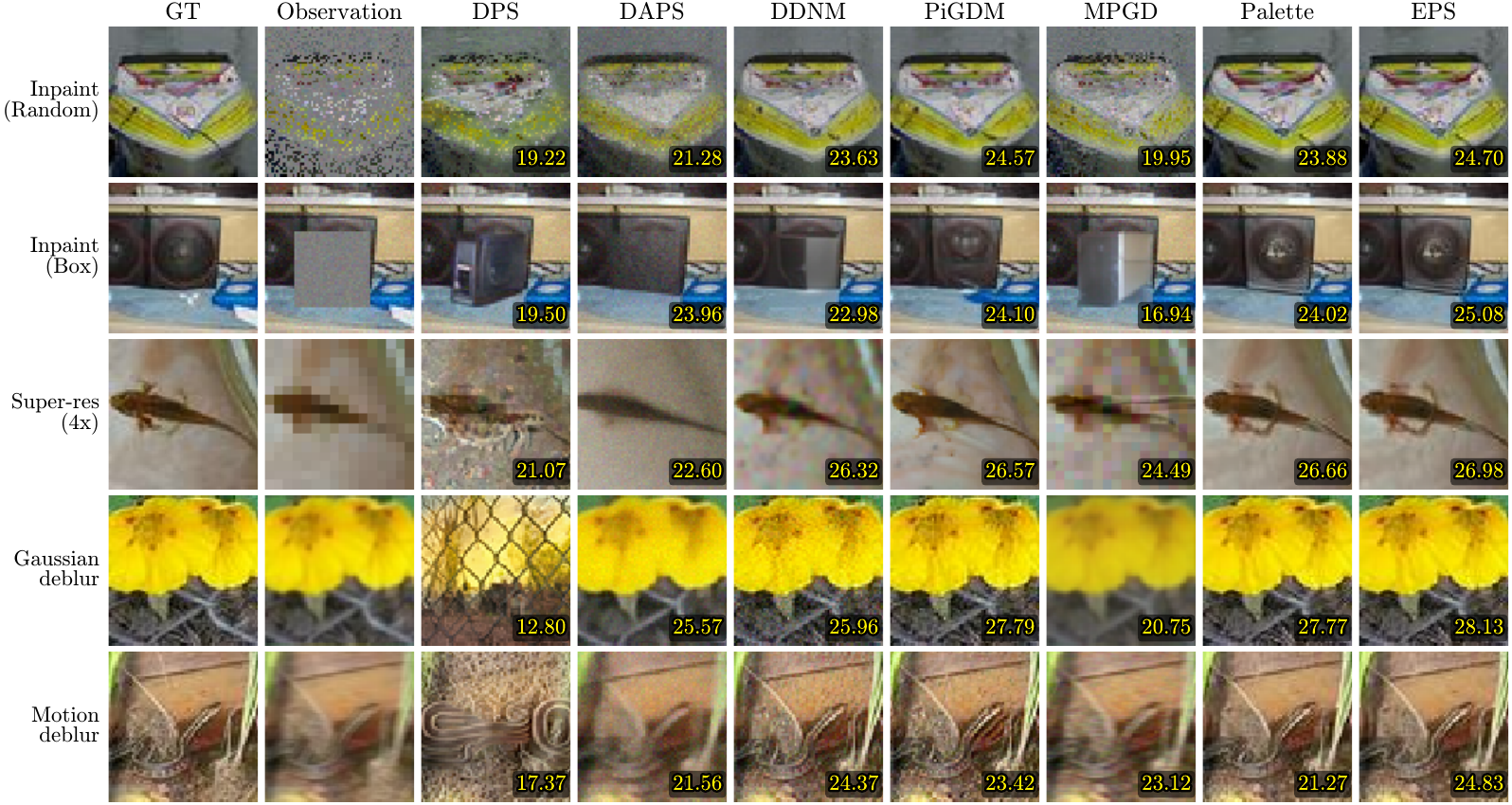}
\caption{\textbf{Qualitative reconstructions across the five inverse problems.} Numbers indicate PSNR values.}
\label{fig:qualitative}
\end{figure}

\begin{figure}[t]
    \centering
    \includegraphics[width=\linewidth]{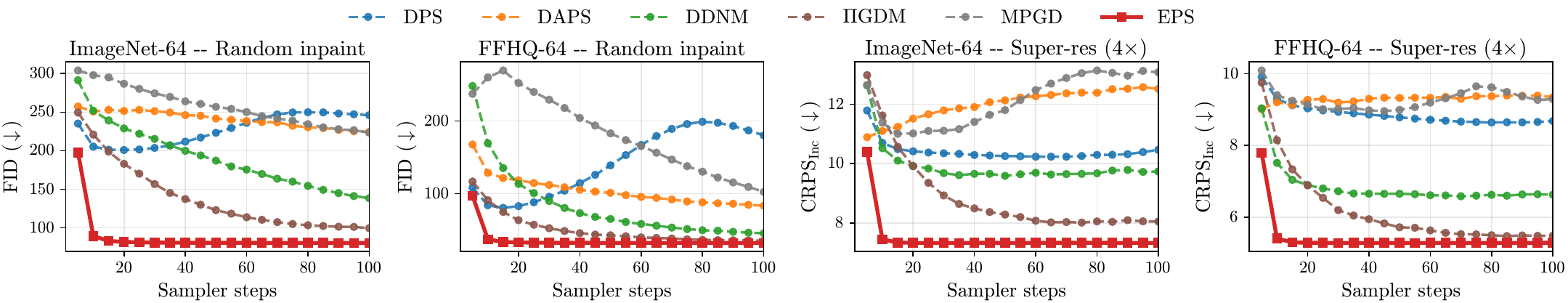}
    \caption{\textbf{EPS converges in $\sim$20 NFE; baselines never catch up.} Sampling-step sensitivity for FID and CRPS-Inception on random inpainting and $4{\times}$ super-resolution, across both ImageNet-64 and FFHQ-64. EPS plateaus within $\sim$20 NFE on every panel, while DPS, DAPS, DDNM, $\Pi$GDM, and MPGD continue to improve out to 100 NFE without reaching the EPS asymptote.}
    \label{fig:main-steps}
\end{figure}

\section{Experiments}
\label{sec:experiments}
We evaluate EPS on five linear inverse problems across two datasets. The main comparison is at $64{\times}64$, where every baseline is run under the same backbone, task, and evaluation protocol and where distributional metrics can be computed fairly. Additional $256{\times}256$ results are in Appendix~\ref{app:additional-experiments:256}.

\subsection{Experimental Setup}
\label{sec:experiments:setup}
\paragraph{Datasets, backbones, and tasks.} We use FFHQ and ImageNet. For ImageNet we use the publicly available class-conditional EDM \cite{karras2022elucidating} checkpoint released with the original codebase. For FFHQ we train an EDM checkpoint from scratch because the released FFHQ model does not reserve images for validation. All methods that require training or fine-tuning use the same backbone as EPS, and all zero-shot solvers use the same pretrained denoiser. We consider five linear inverse problems: random inpainting with $70\%$ missing pixels, centered box inpainting, $4\times$ super-resolution, Gaussian deblurring, and motion deblurring. In all cases we add Gaussian observation noise with standard deviation $\sigma_y=0.05$. Operator details and randomization protocols are in Appendix~\ref{app:implementation}.

\paragraph{Baselines and metrics.} The sampling-based family comprises DPS \cite{chung2022diffusion}, DAPS \cite{zhang2025improving}, DDNM \cite{wang2022zero}, $\Pi$GDM \cite{song2023pseudoinverseguided}, and MPGD \cite{he2023manifoldpreservingguideddiffusion}. The training-based family is represented by Palette \cite{saharia2022palette}, implemented under the same EDM backbone and compute budget; Palette can be viewed as the EPS pipeline with $x_t$ replacing $\mu_\star$, isolating the contribution of the shifted pivot. We report PSNR and SSIM \cite{wang2004image}, LPIPS \cite{zhang2018unreasonable}, and FID \cite{heusel2017gans} for pointwise and perceptual quality, and CRPS \cite{gneiting2014} and MMD \cite{gretton2012kernel} in pixel and Inception feature space \cite{mammadov2026variational,debortoli2025distributionaldiffusionmodelsscoring} for distributional calibration. We use raw Inception features rather than the L2-normalized features of \citet{mammadov2026variational}, so absolute values differ but relative comparisons across methods are preserved. Definitions are in Appendix~\ref{app:metrics}.

\subsection{Main Results}
\label{sec:experiments:main}
Tables~\ref{tab:imagenet} and \ref{tab:ffhq} report all five tasks, all baselines, and all metrics in a single view per dataset. We include EPS at $100$, $20$, and $1$ NFE: the first two correspond to the posterior-sampling regime at different budgets, and the $1$-NFE row tests the posterior-mean prediction from Section~\ref{sec:method:one-step}.
\begin{table*}[t]
  \centering
  \scriptsize
  \setlength{\tabcolsep}{4pt}
  \caption{\textbf{Quantitative comparison on ImageNet-64.} Five linear inverse problems, 100 images $\times$ 10 seeds. Baselines follow the sampler and hyperparameters from their respective papers; Palette and EPS use the EDM Euler sampler at 1 NFE per step. Best in \textbf{bold}, second-best \underline{underlined}; EPS rows highlighted in light pink. \dag\ The NFE$=$1 row applies a single Tweedie evaluation $D_\theta(\mu_\star,\sigma_{\max})$, returning the conditional posterior mean $\mathbb{E}[x_0\mid y]$ in one shot rather than a posterior sample.}
  \label{tab:imagenet}
  \resizebox{\textwidth}{!}{%
  \begin{tabular}{l | l | c | ccc ccc cc}
    \toprule
    \textbf{Task} & \textbf{Method} & \textbf{NFE} & \textbf{PSNR ($\uparrow$)} & \textbf{SSIM ($\uparrow$)} & \textbf{LPIPS ($\downarrow$)} & \textbf{FID ($\downarrow$)} & \textbf{MMD\textsubscript{pix} ($\downarrow$)} & \textbf{MMD\textsubscript{Inc} ($\downarrow$)} & \textbf{CRPS\textsubscript{pix} ($\downarrow$)} & \textbf{CRPS\textsubscript{Inc} ($\downarrow$)} \\
    \midrule
    \multirow{9}{*}{\shortstack[l]{Inpaint\\(random)}}
        & DPS      & 250 & 21.56 & 0.6573 & 0.2158 & 183.94 & -4.88e-03 &  2.51e-02 & 5.88 &  9.89 \\
        & DAPS     & 100 & 20.89 & 0.5555 & 0.3168 & 224.39 & -2.22e-03 &  4.94e-02 & 6.91 & 12.09 \\
        & DDNM     & 100 & 22.63 & 0.7127 & 0.1758 & 138.53 & -5.78e-03 &  7.29e-03 & 5.23 &  7.99 \\
        & $\Pi$GDM & 100 & 23.95 & 0.7780 & 0.1198 &  99.60 & -6.40e-03 & -2.22e-03 & 4.54 &  6.32 \\
        & MPGD     & 100 & 19.62 & 0.5447 & 0.3151 & 223.43 &  2.56e-03 &  4.92e-02 & 7.81 & 11.94 \\
        & Palette  & 100 & 24.09 & 0.7869 & 0.1011 & 81.88  & -6.50e-03 & -4.41e-03 & 4.16 & 5.52 \\
    \cmidrule{2-11}
    \rowcolor{lightpink} \cellcolor{white} & EPS (ours) & 100 & 24.34 & 0.7948 & 0.0979 & \underline{79.60} & \textbf{-6.52e-03} & \textbf{-4.51e-03} & \underline{4.04} & \textbf{5.41} \\
    \rowcolor{lightpink} \cellcolor{white} & EPS (ours) &  20 & \underline{24.87} & \underline{0.8122} & \textbf{0.0910} & \textbf{77.06} & \textbf{-6.52e-03} & \textbf{-4.51e-03} & \textbf{4.02} & \underline{5.44} \\
    \rowcolor{lightpink} \cellcolor{white} & EPS (ours)\dag & 1 & \textbf{26.60} & \textbf{0.8580} & \underline{0.0933} & 88.59 & -6.15e-03 & -4.84e-04 & 4.98 & 8.32 \\
    \midrule
    \multirow{9}{*}{\shortstack[l]{Inpaint\\(box)}}
        & DPS      & 250 & 19.44 & 0.6587 & 0.1891 & 142.05 & -5.63e-03 &  1.01e-02 & 7.25 & 8.19 \\
        & DAPS     & 100 & 20.62 & 0.6600 & 0.2100 & 161.62 & -5.00e-03 &  2.38e-02 & 7.36 & 9.57 \\
        & DDNM     & 100 & 20.38 & 0.7049 & 0.1824 & 125.82 & -5.37e-03 &  4.18e-03 & 7.00 & 7.54 \\
        & $\Pi$GDM & 100 & 20.53 & 0.7292 & 0.1490 & 104.50 & -5.87e-03 & -2.69e-03 & 6.59 & 6.50 \\
        & MPGD     & 100 & 19.35 & 0.6467 & 0.2463 & 167.19 & -3.44e-03 &  1.93e-02 & 8.67 & 9.68 \\
        & Palette  & 100 & 21.12 & 0.7541 & 0.1218 & 92.73  & \underline{-6.10e-03} & -4.12e-03 & 5.92 & 5.93 \\
    \cmidrule{2-11}
    \rowcolor{lightpink} \cellcolor{white} & EPS (ours) & 100 & 21.24 & 0.7569 & \underline{0.1196} & \underline{91.07} & \textbf{-6.11e-03} & \textbf{-4.23e-03} & \underline{5.87} & \textbf{5.84} \\
    \rowcolor{lightpink} \cellcolor{white} & EPS (ours) &  20 & \underline{21.72} & \underline{0.7667} & \textbf{0.1166} & \textbf{90.12} & -6.08e-03 & \underline{-4.14e-03} & \textbf{5.86} & \underline{5.90} \\
    \rowcolor{lightpink} \cellcolor{white} & EPS (ours)\dag & 1 & \textbf{23.60} & \textbf{0.7908} & 0.1514 & 129.11 & -5.60e-03 &  7.34e-03 & 7.16 & 10.38 \\
    \midrule
    \multirow{9}{*}{\shortstack[l]{Super-res.\\($4{\times}$)}}
        & DPS      & 250 & 19.59 & 0.4511 & 0.3367 & 200.98 & -5.06e-03 &  2.67e-02 & 7.77 & 10.36 \\
        & DAPS     & 100 & 18.46 & 0.3324 & 0.4772 & 258.25 &  6.09e-03 &  8.54e-02 & 10.01 & 12.52 \\
        & DDNM     & 100 & \underline{21.10} & 0.5523 & 0.3055 & 169.94 & -5.10e-03 &  1.41e-02 & 7.32 & 9.74 \\
        & $\Pi$GDM & 100 & 20.25 & 0.5318 & 0.2499 & 144.36 & -5.78e-03 &  4.61e-03 & 6.76 & 8.05 \\
        & MPGD     & 100 & 20.57 & 0.5103 & 0.3426 & 221.50 & -3.20e-03 &  6.97e-02 & 9.02 & 13.08 \\
        & Palette  & 100 & 20.24 & 0.5364 & 0.2220 & \textbf{128.76} & \textbf{-5.86e-03} & \underline{-2.79e-03} & \textbf{6.50} & \textbf{7.33} \\
    \cmidrule{2-11}
    \rowcolor{lightpink} \cellcolor{white} & EPS (ours) & 100 & 20.25 & 0.5369 & \underline{0.2207} & \underline{128.80} & \textbf{-5.86e-03} & \textbf{-2.84e-03} & 6.52 & \underline{7.35} \\
    \rowcolor{lightpink} \cellcolor{white} & EPS (ours) &  20 & 20.90 & \underline{0.5692} & \textbf{0.2127} & 129.13 & -5.72e-03 & -1.98e-03 & \underline{6.51} & 7.54 \\
    \rowcolor{lightpink} \cellcolor{white} & EPS (ours)\dag & 1 & \textbf{22.78} & \textbf{0.6530} & 0.2455 & 182.92 & -4.97e-03 &  1.98e-02 & 8.06 & 13.47 \\
    \midrule
    \multirow{9}{*}{\shortstack[l]{Gaussian\\deblur}}
        & DPS      & 250 & 25.05 & 0.7804 & 0.1511 & 158.46 & -5.24e-03 &  1.57e-02 & 4.48 & 8.96 \\
        & DAPS     & 100 & 25.97 & 0.7680 & 0.1466 & 152.47 & -6.50e-03 &  2.78e-02 & 3.69 & 9.58 \\
        & DDNM     & 100 & 26.28 & 0.7903 & 0.1227 & 101.61 & -6.71e-03 &  9.01e-03 & 2.92 & 7.18 \\
        & $\Pi$GDM & 100 & 28.13 & 0.8785 & 0.0742 & 66.10  & -6.70e-03 & -2.49e-03 & 2.91 & 5.68 \\
        & MPGD     & 100 & 20.74 & 0.6375 & 0.2396 & 175.42 &  2.24e-02 &  2.13e-02 & 8.01 & 10.20 \\
        & Palette  & 100 & 29.15 & 0.9010 & 0.0491 & 46.62  & \textbf{-6.83e-03} & \textbf{-5.56e-03} & 2.26 & \textbf{4.11} \\
    \cmidrule{2-11}
    \rowcolor{lightpink} \cellcolor{white} & EPS (ours) & 100 & \underline{29.18} & 0.9015 & \underline{0.0486} & \underline{46.55} & \textbf{-6.83e-03} & \underline{-5.55e-03} & \underline{2.25} & \textbf{4.11} \\
    \rowcolor{lightpink} \cellcolor{white} & EPS (ours) &  20 & \textbf{29.68} & \underline{0.9108} & \textbf{0.0449} & \textbf{44.25} & \textbf{-6.83e-03} & \underline{-5.55e-03} & \textbf{2.24} & \underline{4.13} \\
    \rowcolor{lightpink} \cellcolor{white} & EPS (ours)\dag & 1 & 28.82 & \textbf{0.9194} & 0.0606 & 56.53 & -5.14e-03 & -3.46e-03 & 4.09 & 7.73 \\
    \midrule
    \multirow{9}{*}{\shortstack[l]{Motion\\deblur}}
        & DPS      & 250 & 25.63 & 0.7981 & 0.1398 & 162.66 & -6.41e-03 &  2.58e-02 & 4.42 & 9.82 \\
        & DAPS     & 100 & 24.07 & 0.6859 & 0.1863 & 190.39 & -6.06e-03 &  3.99e-02 & 4.68 & 11.19 \\
        & DDNM     & 100 & 25.73 & 0.7758 & 0.1299 & 109.80 & -6.68e-03 &  7.60e-03 & 3.16 & 7.21 \\
        & $\Pi$GDM & 100 & 25.89 & 0.8205 & 0.1031 & 93.75  & -6.46e-03 & -6.33e-04 & 3.72 & 6.63 \\
        & MPGD     & 100 & 25.90 & 0.8114 & 0.1525 & 183.11 & -6.09e-03 &  3.34e-02 & 5.28 & 12.49 \\
        & Palette  & 100 & 27.02 & 0.8582 & 0.0680 & 62.86  & \underline{-6.73e-03} & -5.02e-03 & \underline{2.93} & 4.81 \\
    \cmidrule{2-11}
    \rowcolor{lightpink} \cellcolor{white} & EPS (ours) & 100 & \underline{27.62} & \underline{0.8661} & \underline{0.0647} & \underline{61.29} & \textbf{-6.77e-03} & \textbf{-5.11e-03} & \textbf{2.69} & \textbf{4.70} \\
    \rowcolor{lightpink} \cellcolor{white} & EPS (ours) &  20 & \textbf{28.17} & \textbf{0.8791} & \textbf{0.0606} & \textbf{59.11} & \textbf{-6.77e-03} & \underline{-5.05e-03} & \textbf{2.69} & \underline{4.76} \\
    \rowcolor{lightpink} \cellcolor{white} & EPS (ours)\dag & 1 & 26.56 & 0.8613 & 0.1079 & 97.59 & -6.03e-03 & -7.30e-04 & 5.25 & 9.91 \\
    \bottomrule
  \end{tabular}%
  }
\end{table*}

\paragraph{Quantitative comparison and sampling efficiency.} EPS attains the best or second-best score on every metric and task in Tables~\ref{tab:imagenet} and \ref{tab:ffhq} , across pointwise fidelity (PSNR, SSIM), perceptual quality (LPIPS, FID), and distributional calibration (CRPS, MMD). The closest competitor is consistently Palette, which shares the EPS pipeline except that the network input is $x_t$ rather than $\mu_\star$; holding the backbone, compute budget, and training data fixed, EPS still outperforms Palette on every task, isolating the contribution of the shifted pivot. Sampling-based baselines that approximate the measurement-matching score ($\Pi$GDM, DPS, DDNM, DAPS, MPGD) lag further behind, even at substantially larger NFE budgets ($\Pi$GDM, DDNM, DAPS, MPGD at $100$ NFE; DPS at $250$), and the gap widens where the operator is most ill-conditioned, namely random inpainting with $70\%$ missing pixels and $4{\times}$ super-resolution. Figure~\ref{fig:main-steps} confirms the same trend across budgets: EPS plateaus within $\sim\!20$ NFE on every panel, while every baseline keeps improving out to $100$ NFE without reaching the EPS asymptote. The $1$-NFE row is the strongest pointwise estimator on PSNR and SSIM (e.g., $26.60$ PSNR on ImageNet random inpainting against $24.34$ at $100$ NFE), consistent with Observation~\ref{obs:one-step-limit}, and trades distributional calibration for that pointwise sharpness in line with the perception-distortion trade-off~\cite{blau2018perception}.

\paragraph{Qualitative comparison and ablations.} Figure~\ref{fig:qualitative} shows reconstructions on the five tasks. The largest visual gaps appear under aggressive inpainting and deblurring, where DPS, DAPS, and MPGD oversmooth or introduce texture inconsistent with the measurement, and DDNM and $\Pi$GDM match measured directions but leave the unmeasured subspace blurry. EPS preserves sharp prior structure while matching the observation, since the pivot $\mu_\star$ explicitly separates measured and unmeasured directions and $\Sigma_\star(t)$ specifies how much to denoise along each. Appendix~\ref{app:additional-experiments} studies the input pivot, zero-shot behavior, warm-start convergence, sampling-step sensitivity, amortization across tasks, and $256{\times}256$ scaling, and confirms the two central mechanisms: the shifted pivot is the right input, and preserving pretrained denoising marginals explains the fast convergence.

\section{Related Work}
\label{sec:related}
\paragraph{Posterior sampling with pretrained generative priors.} A large literature uses pretrained diffusion or flow models as priors for inverse problems. Explicit guidance methods, including Score-SDE/ALD, RePaint, DDNM, DDRM, DPS, $\Pi$GDM, DAPS, MPGD, and FlowDPS \cite{song2019generative,lugmayr2022repaint,wang2022zero,kawar2022denoising,chung2022diffusion,song2023pseudoinverseguided,zhang2025improving,he2023manifoldpreservingguideddiffusion,kim2025flowdps,daras2024survey}, approximate the measurement-matching score $\nabla_{x_t}\log p(y|x_t) \approx -L_t M_t / G_t$ \cite{daras2024survey}, where $M_t$ is a measurement residual, $L_t$ lifts it back to the sample space, and $G_t$ controls the guidance strength. They differ in the form of $M_t$, $L_t$, and $G_t$, instantiating the template via projections, denoised estimates, Jacobian corrections, or moment approximations of $p(x_0|x_t)$. Moment-matching variants \cite{rozet2024learning,rout2024beyond} go beyond first-order Tweedie by approximating $p(x_0|x_t)$ with an anisotropic Gaussian. Section~\ref{sec:method:exactness} makes precise the gap between all of these and the exact posterior score, namely that each method substitutes the unconditional denoising query for the posterior denoising query, querying the network at $x_t$ rather than at the pivot $\mu_\star$. GLASS \cite{holderrieth2025glassflowstransitionsampling} is the closest training-free relative, and its equivalent-time formula coincides with EPS in the special case where $A^\top A$ is a scalar multiple of the identity, equivalently when $\Sigma_\star(t)$ is isotropic (Appendix~\ref{app:derivation:equivalent}). EPS handles the general operator-dependent anisotropic case at the cost of a training step.

\paragraph{Conditional training and restoration bridges.} Palette and conditional image-to-image diffusion models train a network directly on $(x_t, y)$ pairs \cite{saharia2022palette,batzolis2021conditional,mammadov2024amortized,elata2025invfusion}, and bridge-based restoration methods such as InDI \cite{delbracio2023inversion} and image-to-image Schr\"odinger bridges \cite{liu20232} construct trajectories from the measurement distribution to the data distribution. These methods are expressive and avoid hand-designed corrections, but they expose the network to a conditional path whose intermediate marginals do not match the unconditional denoising marginals of the prior, so they cannot leverage a pretrained denoiser as a warm start in a structurally aligned way. EPS instead derives the conditional path induced by the exact linear-Gaussian posterior kernel, which preserves the input/output type of standard denoising pretraining (Section~\ref{sec:method:objective}). This makes both random initialization and warm-starting from a pretrained checkpoint natural training options. We view bridge-based methods as complementary, since they define useful restoration dynamics for general degradations, while EPS identifies the specific bridge that solves the linear-Gaussian posterior exactly.

\section{Conclusion}
\label{sec:conclusion}

We derived the exact posterior score for linear Gaussian inverse problems and showed that posterior sampling reduces to a denoising problem at a measurement-induced pivot $\mu_\star$ under an operator-dependent anisotropic covariance $\Sigma_\star$. We turned this identity into EPS, a denoising training objective whose input/output structure matches standard pretraining, and which can therefore be either trained from scratch or fine-tuned efficiently from a pretrained checkpoint. EPS samples with the underlying backbone's sampler, requires no measurement gradients or projections, and admits a one-evaluation posterior-mean estimator in the high-noise limit. Empirically, EPS improves both reconstruction fidelity and distributional calibration over sampling-based and conditional-training baselines on five linear inverse problems on FFHQ and ImageNet, while exposing an explicit sampling-budget trade-off through the same sampler used by the backbone.

\paragraph{Limitations.} The exact derivation assumes a linear forward operator and Gaussian observation noise. Nonlinear operators can be approached by local linearization or by training against the true likelihood, but the closed form of Theorem~\ref{thm:main} no longer applies directly. Pixel-space inverse problems with latent diffusion backbones also require care because the decoder makes a linear pixel-space operator nonlinear in latent space.

\begin{ack}
AM is supported by the Clarendon Fund Scholarship, University of Oxford.
We thank fal for the compute grants that supported this research. This work was supported in part by the Engineering and Physical Sciences Research Council (EPSRC) through the AI Hub in Generative Models [grant number EP/Y028805/1].
The authors acknowledge the use of resources provided by the Isambard-AI National AI Research Resource (AIRR)~\cite{mcintoshsmith2024isambardaileadershipclasssupercomputer}. Isambard-AI is operated by the University of Bristol and is funded by the UK Government’s Department for Science, Innovation and Technology (DSIT) via UK Research and Innovation; and the Science and Technology Facilities Council [ST/AIRR/I-A-I/1023].
\end{ack}

\newpage
\bibliography{references}

@article{chung2022diffusion,
  title={Diffusion posterior sampling for general noisy inverse problems},
  author={Chung, Hyungjin and Kim, Jeongsol and Mccann, Michael T and Klasky, Marc L and Ye, Jong Chul},
  journal={arXiv preprint arXiv:2209.14687},
  year={2022}
}

@misc{debortoli2025distributionaldiffusionmodelsscoring,
      title={Distributional Diffusion Models with Scoring Rules}, 
      author={Valentin De Bortoli and Alexandre Galashov and J. Swaroop Guntupalli and Guangyao Zhou and Kevin Murphy and Arthur Gretton and Arnaud Doucet},
      year={2025},
      eprint={2502.02483},
      archivePrefix={arXiv},
      primaryClass={cs.LG},
      url={https://arxiv.org/abs/2502.02483}, 
}

@article{kim2025flowdps,
  title={Flow{DPS}: Flow-driven posterior sampling for inverse problems},
  author={Kim, Jeongsol and Kim, Bryan Sangwoo and Ye, Jong Chul},
  journal={arXiv preprint arXiv:2503.08136},
  year={2025}
}

@article{albergo2023stochastic,
  title={Stochastic interpolants: A unifying framework for flows and diffusions},
  author={Albergo, Michael and Boffi, Nicholas M and Vanden-Eijnden, Eric},
  journal={Journal of Machine Learning Research},
  volume={26},
  number={209},
  pages={1--80},
  year={2025}
}

@inproceedings{zhang2018unreasonable,
  title={The unreasonable effectiveness of deep features as a perceptual metric},
  author={Zhang, Richard and Isola, Phillip and Efros, Alexei A and Shechtman, Eli and Wang, Oliver},
  booktitle={Proceedings of the IEEE conference on computer vision and pattern recognition},
  pages={586--595},
  year={2018}
}

@inproceedings{blau2018perception,
  title={The perception-distortion tradeoff},
  author={Blau, Yochai and Michaeli, Tomer},
  booktitle={Proceedings of the IEEE conference on computer vision and pattern recognition},
  pages={6228--6237},
  year={2018}
}

@inproceedings{zhang2025improving,
  title={Improving diffusion inverse problem solving with decoupled noise annealing},
  author={Zhang, Bingliang and Chu, Wenda and Berner, Julius and Meng, Chenlin and Anandkumar, Anima and Song, Yang},
  booktitle={Proceedings of the Computer Vision and Pattern Recognition Conference},
  pages={20895--20905},
  year={2025}
}

@misc{he2023manifoldpreservingguideddiffusion,
      title={Manifold Preserving Guided Diffusion}, 
      author={Yutong He and Naoki Murata and Chieh-Hsin Lai and Yuhta Takida and Toshimitsu Uesaka and Dongjun Kim and Wei-Hsiang Liao and Yuki Mitsufuji and J. Zico Kolter and Ruslan Salakhutdinov and Stefano Ermon},
      year={2023},
      eprint={2311.16424},
      archivePrefix={arXiv},
      primaryClass={cs.LG},
      url={https://arxiv.org/abs/2311.16424}, 
}

@article{song2019generative,
  title={Generative modeling by estimating gradients of the data distribution},
  author={Song, Yang and Ermon, Stefano},
  journal={Advances in neural information processing systems},
  volume={32},
  year={2019}
}

@article{liu2022flow,
  title={Flow straight and fast: Learning to generate and transfer data with rectified flow},
  author={Liu, Xingchao and Gong, Chengyue and Liu, Qiang},
  journal={arXiv preprint arXiv:2209.03003},
  year={2022}
}

@inproceedings{heusel2017gans,
  title={Gans trained by a two time-scale update rule converge to a local nash equilibrium},
  author={Heusel, Martin and Ramsauer, Hubert and Unterthiner, Thomas and Nessler, Bernhard and Hochreiter, Sepp},
  booktitle={Advances in neural information processing systems},
  volume={30},
  year={2017}
}

@article{lipman2022flow,
  title={Flow matching for generative modeling},
  author={Lipman, Yaron and Chen, Ricky TQ and Ben-Hamu, Heli and Nickel, Maximilian and Le, Matt},
  journal={arXiv preprint arXiv:2210.02747},
  year={2022}
}

@inproceedings{ho2020denoising,
  title={Denoising diffusion probabilistic models},
  author={Ho, Jonathan and Jain, Ajay and Abbeel, Pieter},
  booktitle={Advances in neural information processing systems},
  volume={33},
  pages={6840--6851},
  year={2020}
}

@article{song2020score,
  title={Score-based generative modeling through stochastic differential equations},
  author={Song, Yang and Sohl-Dickstein, Jascha and Kingma, Diederik P and Kumar, Abhishek and Ermon, Stefano and Poole, Ben},
  journal={arXiv:2011.13456},
  year={2020}
}

@article{gneiting2014,
  title={Probabilistic forecasting},
  author={Gneiting, Tilmann and Katzfuss, Matthias},
  journal={Annual Review of Statistics and Its Application},
  volume={1},
  pages={125--151},
  year={2014},
  publisher={Annual Reviews}
}

@article{liu20232,
  title={{$I^{2}SB$}: Image-to-Image {S}chr{\"o}dinger Bridge},
  author={Liu, Guan-Horng and Vahdat, Arash and Huang, De-An and Theodorou, Evangelos A and Nie, Weili and Anandkumar, Anima},
  journal={arXiv preprint arXiv:2302.05872},
  year={2023}
}

@article{gretton2012kernel,
  title={A kernel two-sample test},
  author={Gretton, Arthur and Borgwardt, Karsten M and Rasch, Malte J and Sch{\"o}lkopf, Bernhard and Smola, Alexander},
  journal={The journal of machine learning research},
  volume={13},
  number={1},
  pages={723--773},
  year={2012},
  publisher={JMLR. org}
}

@article{sohldickstein2015deep,
  title={Deep Unsupervised Learning using Nonequilibrium Thermodynamics}, 
  author={Jascha Sohl-Dickstein and Eric A. Weiss and Niru Maheswaranathan and Surya Ganguli},
  year={2015},
  journal={arXiv:1503.03585},
}

@article{karras2022elucidating,
      title={Elucidating the Design Space of Diffusion-Based Generative Models}, 
      author={Tero Karras and Miika Aittala and Timo Aila and Samuli Laine},
      year={2022},
      journal={arXiv:2206.00364},
}

@misc{holderrieth2025glassflowstransitionsampling,
      title={GLASS Flows: Transition Sampling for Alignment of Flow and Diffusion Models}, 
      author={Peter Holderrieth and Uriel Singer and Tommi Jaakkola and Ricky T. Q. Chen and Yaron Lipman and Brian Karrer},
      year={2025},
      eprint={2509.25170},
      archivePrefix={arXiv},
      primaryClass={cs.LG},
      url={https://arxiv.org/abs/2509.25170}, 
}

@article{wang2022zero,
  title={Zero-shot image restoration using denoising diffusion null-space model},
  author={Wang, Yinhuai and Yu, Jiwen and Zhang, Jian},
  journal={arXiv preprint arXiv:2212.00490},
  year={2022}
}

@inproceedings{
song2023pseudoinverseguided,
title={Pseudoinverse-Guided Diffusion Models for Inverse Problems},
author={Jiaming Song and Arash Vahdat and Morteza Mardani and Jan Kautz},
booktitle={International Conference on Learning Representations},
year={2023},
url={https://openreview.net/forum?id=9_gsMA8MRKQ}
}

@article{kawar2022denoising,
  title={Denoising diffusion restoration models},
  author={Kawar, Bahjat and Elad, Michael and Ermon, Stefano and Song, Jiaming},
  journal={Advances in neural information processing systems},
  volume={35},
  pages={23593--23606},
  year={2022}
}

@article{chung2022improving,
  title={Improving diffusion models for inverse problems using manifold constraints},
  author={Chung, Hyungjin and Sim, Byeongsu and Ryu, Dohoon and Ye, Jong Chul},
  journal={Advances in Neural Information Processing Systems},
  volume={35},
  pages={25683--25696},
  year={2022}
}

@article{mardani2023variational,
  title={A variational perspective on solving inverse problems with diffusion models},
  author={Mardani, Morteza and Song, Jiaming and Kautz, Jan and Vahdat, Arash},
  journal={arXiv preprint arXiv:2305.04391},
  year={2023}
}

@article{mammadov2024amortized,
  title={Amortized posterior sampling with diffusion prior distillation},
  author={Mammadov, Abbas and Chung, Hyungjin and Ye, Jong Chul},
  journal={arXiv preprint arXiv:2407.17907},
  year={2024}
}

@misc{mcintoshsmith2024isambardaileadershipclasssupercomputer,
      title={Isambard-AI: a leadership class supercomputer optimised specifically for Artificial Intelligence}, 
      author={Simon McIntosh-Smith and Sadaf R Alam and Christopher Woods},
      year={2024},
      eprint={2410.11199},
      archivePrefix={arXiv},
      primaryClass={cs.DC},
      url={https://arxiv.org/abs/2410.11199}, 
}

@article{rozet2024learning,
  title={Learning diffusion priors from observations by expectation maximization},
  author={Rozet, Fran{\c{c}}ois and Andry, G{\'e}r{\^o}me and Lanusse, Fran{\c{c}}ois and Louppe, Gilles},
  journal={Advances in Neural Information Processing Systems},
  volume={37},
  pages={87647--87682},
  year={2024}
}

@inproceedings{rout2024beyond,
  title={Beyond first-order tweedie: Solving inverse problems using latent diffusion},
  author={Rout, Litu and Chen, Yujia and Kumar, Abhishek and Caramanis, Constantine and Shakkottai, Sanjay and Chu, Wen-Sheng},
  booktitle={Proceedings of the IEEE/CVF Conference on Computer Vision and Pattern Recognition},
  pages={9472--9481},
  year={2024}
}

@article{daras2024survey,
  title={A survey on diffusion models for inverse problems},
  author={Daras, Giannis and Chung, Hyungjin and Lai, Chieh-Hsin and Mitsufuji, Yuki and Ye, Jong Chul and Milanfar, Peyman and Dimakis, Alexandros G and Delbracio, Mauricio},
  journal={arXiv preprint arXiv:2410.00083},
  year={2024}
}

@article{delbracio2023inversion,
  title={Inversion by direct iteration: An alternative to denoising diffusion for image restoration},
  author={Delbracio, Mauricio and Milanfar, Peyman},
  journal={arXiv preprint arXiv:2303.11435},
  year={2023}
}

@article{elata2025invfusion,
  title={InvFusion: Bridging Supervised and Zero-shot Diffusion for Inverse Problems},
  author={Elata, Noam and Chung, Hyungjin and Ye, Jong Chul and Michaeli, Tomer and Elad, Michael},
  journal={arXiv preprint arXiv:2504.01689},
  year={2025}
}

@article{mammadov2026variational,
  title={Variational Flow Maps: Make Some Noise for One-Step Conditional Generation},
  author={Mammadov, Abbas and Takao, So and Chen, Bohan and Baptista, Ricardo and Mardani, Morteza and Teh, Yee Whye and Berner, Julius},
  journal={arXiv preprint arXiv:2603.07276},
  year={2026}
}

@article{batzolis2021conditional,
  title={Conditional image generation with score-based diffusion models},
  author={Batzolis, Georgios and Stanczuk, Jan and Sch{\"o}nlieb, Carola-Bibiane and Etmann, Christian},
  journal={arXiv preprint arXiv:2111.13606},
  year={2021}
}

@inproceedings{saharia2022palette,
  title={Palette: Image-to-image diffusion models},
  author={Saharia, Chitwan and Chan, William and Chang, Huiwen and Lee, Chris and Ho, Jonathan and Salimans, Tim and Fleet, David and Norouzi, Mohammad},
  booktitle={ACM SIGGRAPH 2022 conference proceedings},
  pages={1--10},
  year={2022}
}

@article{wang2004image,
  title={Image quality assessment: from error visibility to structural similarity},
  author={Wang, Zhou and Bovik, Alan C and Sheikh, Hamid R and Simoncelli, Eero P},
  journal={IEEE transactions on image processing},
  volume={13},
  number={4},
  pages={600--612},
  year={2004},
  publisher={IEEE}
}

@article{efron2011tweedie,
  title={Tweedie’s formula and selection bias},
  author={Efron, Bradley},
  journal={Journal of the American Statistical Association},
  volume={106},
  number={496},
  pages={1602--1614},
  year={2011},
  publisher={Taylor \& Francis}
}

@incollection{robbins1992empirical,
  title={An empirical Bayes approach to statistics},
  author={Robbins, Herbert E},
  booktitle={Breakthroughs in Statistics: Foundations and basic theory},
  pages={388--394},
  year={1992},
  publisher={Springer}
}

@article{donoho2006compressed,
  title={Compressed sensing},
  author={Donoho, David L},
  journal={IEEE Transactions on information theory},
  volume={52},
  number={4},
  pages={1289--1306},
  year={2006},
  publisher={IEEE}
}

@article{candes2006robust,
  title={Robust uncertainty principles: Exact signal reconstruction from highly incomplete frequency information},
  author={Cand{\`e}s, Emmanuel J and Romberg, Justin and Tao, Terence},
  journal={IEEE Transactions on information theory},
  volume={52},
  number={2},
  pages={489--509},
  year={2006},
  publisher={IEEE}
}

@article{lustig2007sparse,
  title={Sparse MRI: The application of compressed sensing for rapid MR imaging},
  author={Lustig, Michael and Donoho, David and Pauly, John M},
  journal={Magnetic Resonance in Medicine: An Official Journal of the International Society for Magnetic Resonance in Medicine},
  volume={58},
  number={6},
  pages={1182--1195},
  year={2007},
  publisher={Wiley Online Library}
}

@article{knoll2020fastmri,
  title={fastMRI: A publicly available raw k-space and DICOM dataset of knee images for accelerated MR image reconstruction using machine learning},
  author={Knoll, Florian and Zbontar, Jure and Sriram, Anuroop and Muckley, Matthew J and Bruno, Mary and Defazio, Aaron and Parente, Marc and Geras, Krzysztof J and Katsnelson, Joe and Chandarana, Hersh and others},
  journal={Radiology: Artificial Intelligence},
  volume={2},
  number={1},
  pages={e190007},
  year={2020},
  publisher={Radiological Society of North America}
}

@inproceedings{dong2014srcnn,
  title={Learning a deep convolutional network for image super-resolution},
  author={Dong, Chao and Loy, Chen Change and He, Kaiming and Tang, Xiaoou},
  booktitle={European conference on computer vision},
  pages={184--199},
  year={2014},
  organization={Springer}
}

@inproceedings{wang2021realesrgan,
  title={Real-esrgan: Training real-world blind super-resolution with pure synthetic data},
  author={Wang, Xintao and Xie, Liangbin and Dong, Chao and Shan, Ying},
  booktitle={Proceedings of the IEEE/CVF international conference on computer vision},
  pages={1905--1914},
  year={2021}
}

@inproceedings{nah2017deepdeblur,
  title={Deep multi-scale convolutional neural network for dynamic scene deblurring},
  author={Nah, Seungjun and Hyun Kim, Tae and Mu Lee, Kyoung},
  booktitle={Proceedings of the IEEE conference on computer vision and pattern recognition},
  pages={3883--3891},
  year={2017}
}

@inproceedings{kupyn2018deblurgan,
  title={Deblurgan: Blind motion deblurring using conditional adversarial networks},
  author={Kupyn, Orest and Budzan, Volodymyr and Mykhailych, Mykola and Mishkin, Dmytro and Matas, Ji{\v{r}}{\'\i}},
  booktitle={Proceedings of the IEEE conference on computer vision and pattern recognition},
  pages={8183--8192},
  year={2018}
}

@inproceedings{bertalmio2000image,
  title={Image inpainting},
  author={Bertalmio, Marcelo and Sapiro, Guillermo and Caselles, Vincent and Ballester, Coloma},
  booktitle={Proceedings of the 27th annual conference on Computer graphics and interactive techniques},
  pages={417--424},
  year={2000}
}

@inproceedings{pathak2016context,
  title={Context encoders: Feature learning by inpainting},
  author={Pathak, Deepak and Krahenbuhl, Philipp and Donahue, Jeff and Darrell, Trevor and Efros, Alexei A},
  booktitle={Proceedings of the IEEE conference on computer vision and pattern recognition},
  pages={2536--2544},
  year={2016}
}

@book{williams2006gaussian,
  title={Gaussian processes for machine learning},
  author={Williams, Christopher KI and Rasmussen, Carl Edward},
  volume={2},
  pages={68},
  year={2006},
  publisher={MIT Press Cambridge, MA}
}

@article{wu2023practical,
  title={Practical and asymptotically exact conditional sampling in diffusion models},
  author={Wu, Luhuan and Trippe, Brian and Naesseth, Christian and Blei, David and Cunningham, John P},
  journal={Advances in Neural Information Processing Systems},
  volume={36},
  pages={31372--31403},
  year={2023}
}

@article{cardoso2023monte,
  title={Monte carlo guided diffusion for bayesian linear inverse problems},
  author={Cardoso, Gabriel and Idrissi, Yazid Janati El and Corff, Sylvain Le and Moulines, Eric},
  journal={arXiv preprint arXiv:2308.07983},
  year={2023}
}

@inproceedings{dou2024diffusion,
  title={Diffusion posterior sampling for linear inverse problem solving: A filtering perspective},
  author={Dou, Zehao and Song, Yang},
  booktitle={The Twelfth International Conference on Learning Representations},
  year={2024}
}

@inproceedings{feng2023score,
  title={Score-based diffusion models as principled priors for inverse imaging},
  author={Feng, Berthy T and Smith, Jamie and Rubinstein, Michael and Chang, Huiwen and Bouman, Katherine L and Freeman, William T},
  booktitle={Proceedings of the IEEE/CVF international conference on computer vision},
  pages={10520--10531},
  year={2023}
}

@inproceedings{lugmayr2022repaint,
  title={Repaint: Inpainting using denoising diffusion probabilistic models},
  author={Lugmayr, Andreas and Danelljan, Martin and Romero, Andres and Yu, Fisher and Timofte, Radu and Van Gool, Luc},
  booktitle={Proceedings of the IEEE/CVF conference on computer vision and pattern recognition},
  pages={11461--11471},
  year={2022}
}
\bibliographystyle{unsrtnat}

\newpage
\appendix

\startcontents[appendices]
\noindent\rule{\linewidth}{0.4pt}
\begin{center}
{\large\bfseries Appendix Contents}
\end{center}
\noindent\rule{\linewidth}{0.4pt}
\printcontents[appendices]{l}{1}{\setcounter{tocdepth}{2}}
\noindent\rule{\linewidth}{0.4pt}
\vspace{1em}

\clearpage
\section{Theory and Derivations}
\label{app:derivation}

This appendix gives the full derivation of the EPS score identity, the posterior velocity, the equivalent-time special case, and the one-step posterior-mean limit. We keep the assumptions of the main text: $x_t=\alpha_t x_0+\beta_t\epsilon$ with $\epsilon\sim\N(0,I_d)$ and $y=Ax_0+\eta$ with $\eta\sim\N(0,\sigma_y^2I_m)$.

\subsection{Anisotropic Tweedie Formula}
\label{app:derivation:tweedie}

For any positive definite $\Sigma$, define
\begin{equation}
    p_{\mathrm{data}}^\Sigma(\mu)=\int \N(\mu;x_0,\Sigma)\pdata(x_0)\,\dd x_0,
    \qquad
    D_\Sigma(\mu)=\E[x_0| x_0+\xi=\mu],\quad \xi\sim\N(0,\Sigma).
\end{equation}
Then
\begin{equation}
    D_\Sigma(\mu)=\mu+\Sigma\nabla_\mu\log p_{\mathrm{data}}^\Sigma(\mu).
    \label{eq:app:tweedie}
\end{equation}
Indeed, differentiating under the integral gives
\begin{align}
    \nabla_\mu p_{\mathrm{data}}^\Sigma(\mu)
    &=\int \nabla_\mu \N(\mu;x_0,\Sigma)\pdata(x_0)\,\dd x_0 \\
    &=\int \N(\mu;x_0,\Sigma)\Sigma^{-1}(x_0-\mu)\pdata(x_0)\,\dd x_0.
\end{align}
Dividing by $p_{\mathrm{data}}^\Sigma(\mu)$ yields
\begin{equation}
    \nabla_\mu\log p_{\mathrm{data}}^\Sigma(\mu)
    =\Sigma^{-1}\left(\E[x_0| x_0+\xi=\mu]-\mu\right),
\end{equation}
which proves \eqref{eq:app:tweedie}. The key point for EPS is that $\Sigma$ is a full covariance matrix fixed by the inverse problem, not a scalar noise level chosen to match an unconditional diffusion time.

\subsection{Proof of Theorem~\ref{thm:main}}
\label{app:derivation:main-proof}

\paragraph{Theorem~\ref{thm:main} restated.}
\emph{Under the linear Gaussian inverse problem \eqref{eq:ip} and the interpolant \eqref{eq:interpolant}, the posterior score at time $t$ is
\begin{equation}
    \nabla_{x_t}\log p(x_t|y)
    =\frac{1}{\beta_t^2}\Big(\alpha_tD_{\Sigma_\star(t)}(\mu_\star(x_t,y,t))-x_t\Big),
\end{equation}
where
\begin{equation}
    \Sigma_\star(t)=\left(\frac{\alpha_t^2}{\beta_t^2}I_d+\frac{1}{\sigma_y^2}A^\top A\right)^{-1},
    \qquad
    \mu_\star(x_t,y,t)=\Sigma_\star(t)\left(\frac{\alpha_t}{\beta_t^2}x_t+\frac{1}{\sigma_y^2}A^\top y\right).
\end{equation}
Equivalently, $D_{\Sigma_\star(t)}(\mu_\star(x_t,y,t))=\E[x_0|x_t,y]$.}

\begin{proof}
By conditional independence of $x_t$ and $y$ given $x_0$,
\begin{equation}
    p(x_t|y)=\frac{1}{p(y)}\int p(x_t|x_0)p(y|x_0)\pdata(x_0)\,\dd x_0.
    \label{eq:app:posterior-integral-full}
\end{equation}
The two Gaussian factors are
\begin{align}
    p(x_t|x_0)&=(2\pi\beta_t^2)^{-d/2}\exp\left(-\frac{\|x_t-\alpha_t x_0\|^2}{2\beta_t^2}\right),\\
    p(y|x_0)&=(2\pi\sigma_y^2)^{-m/2}\exp\left(-\frac{\|y-Ax_0\|^2}{2\sigma_y^2}\right).
\end{align}
Collect the exponent terms that depend on $x_0$:
\begin{align}
    &-\frac{\|x_t-\alpha_t x_0\|^2}{2\beta_t^2}
    -\frac{\|y-Ax_0\|^2}{2\sigma_y^2} \\
    &\quad= -\frac{1}{2}x_0^\top
    \underbrace{\left(\frac{\alpha_t^2}{\beta_t^2}I_d+\frac{1}{\sigma_y^2}A^\top A\right)}_{\Lambda_t}
    x_0
    +
    \underbrace{\left(\frac{\alpha_t}{\beta_t^2}x_t+\frac{1}{\sigma_y^2}A^\top y\right)^\top}_{b_t^\top}
    x_0
    -\frac{\|x_t\|^2}{2\beta_t^2}-\frac{\|y\|^2}{2\sigma_y^2}.
    \label{eq:app:quadratic-collect}
\end{align}
Because $\beta_t>0$, the matrix $\Lambda_t$ is positive definite even when $A$ is rank deficient. Let $\Sigma_\star=\Lambda_t^{-1}$ and $\mu_\star=\Sigma_\star b_t$. Completing the square in \eqref{eq:app:quadratic-collect},
\begin{equation}
    -\frac{1}{2}x_0^\top\Lambda_t x_0+b_t^\top x_0
    =-\frac{1}{2}(x_0-\mu_\star)^\top\Sigma_\star^{-1}(x_0-\mu_\star)
    +\frac{1}{2}b_t^\top\Sigma_\star b_t.
\end{equation}
Therefore the product of the two likelihoods can be factored as
\begin{equation}
    p(x_t|x_0)p(y|x_0)=C_t(x_t,y)\,\N(x_0;\mu_\star,\Sigma_\star),
    \label{eq:app:kernel-factor}
\end{equation}
where all dependence on $x_0$ is contained in the displayed Gaussian and
\begin{equation}
    \log C_t(x_t,y)
    =\mathrm{const}(t)-\frac{\|x_t\|^2}{2\beta_t^2}-\frac{\|y\|^2}{2\sigma_y^2}
    +\frac{1}{2}b_t^\top\Sigma_\star b_t.
    \label{eq:app:normalizer}
\end{equation}
Substituting \eqref{eq:app:kernel-factor} into \eqref{eq:app:posterior-integral-full} gives
\begin{equation}
    p(x_t|y)=\frac{C_t(x_t,y)}{p(y)}\,
    p_{\mathrm{data}}^{\Sigma_\star}(\mu_\star),
    \qquad
    p_{\mathrm{data}}^{\Sigma_\star}(\mu_\star)=\int \N(x_0;\mu_\star,\Sigma_\star)\pdata(x_0)\,\dd x_0.
    \label{eq:app:posterior-smoothed-full}
\end{equation}
Now differentiate \eqref{eq:app:posterior-smoothed-full} with respect to $x_t$. Since $\Sigma_\star$ does not depend on $x_t$ and $\partial b_t/\partial x_t=(\alpha_t/\beta_t^2)I_d$,
\begin{equation}
    \frac{\partial \mu_\star}{\partial x_t}=\frac{\alpha_t}{\beta_t^2}\Sigma_\star.
    \label{eq:app:mu-derivative}
\end{equation}
The normalizer derivative is
\begin{align}
    \nabla_{x_t}\log C_t(x_t,y)
    &=-\frac{x_t}{\beta_t^2}
    +\frac{1}{2}\nabla_{x_t}\left(b_t^\top\Sigma_\star b_t\right) \\
    &=-\frac{x_t}{\beta_t^2}+\frac{\alpha_t}{\beta_t^2}\Sigma_\star b_t
    =-\frac{x_t}{\beta_t^2}+\frac{\alpha_t}{\beta_t^2}\mu_\star.
    \label{eq:app:normalizer-derivative}
\end{align}
The smoothed-density derivative is, by the chain rule and \eqref{eq:app:mu-derivative},
\begin{equation}
    \nabla_{x_t}\log p_{\mathrm{data}}^{\Sigma_\star}(\mu_\star)
    =\frac{\alpha_t}{\beta_t^2}\Sigma_\star
    \nabla_\mu\log p_{\mathrm{data}}^{\Sigma_\star}(\mu)\big|_{\mu=\mu_\star}.
    \label{eq:app:smoothed-derivative}
\end{equation}
Using the anisotropic Tweedie identity \eqref{eq:app:tweedie},
\begin{equation}
    \Sigma_\star\nabla_\mu\log p_{\mathrm{data}}^{\Sigma_\star}(\mu_\star)
    =D_{\Sigma_\star}(\mu_\star)-\mu_\star.
    \label{eq:app:tweedie-substitute}
\end{equation}
Combining \eqref{eq:app:normalizer-derivative}, \eqref{eq:app:smoothed-derivative}, and \eqref{eq:app:tweedie-substitute}, the pivot terms cancel:
\begin{align}
    \nabla_{x_t}\log p(x_t|y)
    &=-\frac{x_t}{\beta_t^2}+\frac{\alpha_t}{\beta_t^2}\mu_\star
    +\frac{\alpha_t}{\beta_t^2}\left(D_{\Sigma_\star}(\mu_\star)-\mu_\star\right)\\
    &=\frac{1}{\beta_t^2}\left(\alpha_tD_{\Sigma_\star}(\mu_\star)-x_t\right).
\end{align}
Finally, under the joint density proportional to $\N(x_0;\mu_\star,\Sigma_\star)\pdata(x_0)$, the posterior mean of $x_0$ is exactly $D_{\Sigma_\star}(\mu_\star)$, so $D_{\Sigma_\star(t)}(\mu_\star(x_t,y,t))=\E[x_0|x_t,y]$.
\end{proof}

\subsection{Proof of Proposition~\ref{prop:posterior-velocity}}
\label{app:derivation:velocity-proof}

For brevity in this proof we write $D^\star(x_t,y,t) \coloneqq D_{\Sigma_\star(t)}(\mu_\star(x_t,y,t))$, which by Theorem~\ref{thm:main} equals $\E[x_0 | x_t, y]$.
The interpolant satisfies $x_t=\alpha_t x_0+\beta_t\epsilon$, hence conditional on $(x_t,y)$,
\begin{equation}
    \E[\epsilon|x_t,y]=\frac{x_t-\alpha_t\E[x_0|x_t,y]}{\beta_t}
    =\frac{x_t-\alpha_tD^\star(x_t,y,t)}{\beta_t}.
\end{equation}
The posterior velocity is the conditional expectation of the path derivative:
\begin{align}
    v_t^y(x_t)
    &=\E[\dot\alpha_t x_0+\dot\beta_t\epsilon| x_t,y]\\
    &=\dot\alpha_t D^\star(x_t,y,t)
    +\dot\beta_t\frac{x_t-\alpha_t D^\star(x_t,y,t)}{\beta_t}\\
    &=\frac{\dot\beta_t}{\beta_t}x_t+
    \left(\dot\alpha_t-\frac{\alpha_t\dot\beta_t}{\beta_t}\right)D^\star(x_t,y,t).
\end{align}
Thus estimating the exact posterior denoiser is equivalent to estimating the exact posterior velocity for the interpolant.

\subsection{Proof of Proposition~\ref{prop:isotropic-pivot}}
\label{app:derivation:isotropic-pivot}

Let
\begin{equation}
    \Lambda_t
    =
    \frac{\alpha_t^2}{\beta_t^2}I_d
    +
    \frac{1}{\sigma_y^2}A^\top A,
    \qquad
    \Sigma_\star(t)=\Lambda_t^{-1}.
\end{equation}
Using $x_t=\alpha_t x_0+\beta_t\epsilon$ and
$y=Ax_0+\sigma_y\eta$, the pivot is
\begin{align}
    \mu_\star
    &=
    \Sigma_\star
    \left(
        \frac{\alpha_t}{\beta_t^2}x_t
        +
        \frac{1}{\sigma_y^2}A^\top y
    \right)\\
    &=
    \Sigma_\star
    \left[
        \left(
            \frac{\alpha_t^2}{\beta_t^2}I_d
            +
            \frac{1}{\sigma_y^2}A^\top A
        \right)x_0
        +
        \frac{\alpha_t}{\beta_t}\epsilon
        +
        \frac{1}{\sigma_y}A^\top\eta
    \right]\\
    &=
    x_0
    +
    \Sigma_\star
    \left(
        \frac{\alpha_t}{\beta_t}\epsilon
        +
        \frac{1}{\sigma_y}A^\top\eta
    \right).
\end{align}
The noise term is Gaussian with mean zero and covariance
\begin{align}
    \Sigma_\star
    \left(
        \frac{\alpha_t^2}{\beta_t^2}I_d
        +
        \frac{1}{\sigma_y^2}A^\top A
    \right)
    \Sigma_\star
    =
    \Sigma_\star\Lambda_t\Sigma_\star
    =
    \Sigma_\star.
\end{align}
Therefore $\mu_\star|x_0,t\sim\N(x_0,\Sigma_\star(t))$, as claimed.

\subsection{Proof of Observation~\ref{obs:one-step-limit}}
\label{app:derivation:one-step}

For EDM, $\alpha_t=1$ and $\beta_t=\sigma_t$. Let
\begin{equation}
    A = U_r S V_r^\top
\end{equation}
be the compact SVD of $A$, where
$S=\mathrm{diag}(s_1,\ldots,s_r)$ contains the positive singular values.
Let $V_0$ be an orthonormal basis for $\mathcal N(A)$. Then
\begin{equation}
    P_{\mathcal N(A)} = V_0V_0^\top,
    \qquad
    A^\dagger = V_rS^{-1}U_r^\top .
\end{equation}
With $\lambda=\sigma_y^2/\sigma_t^2$, the pivot can be written as
\begin{align}
    \mu_\star
    &=
    V_r(S^2+\lambda I)^{-1}
    \left(SU_r^\top y+\lambda V_r^\top x_t\right)
    +
    V_0V_0^\top x_t .
    \label{eq:app:svd-pivot}
\end{align}
Taking $\lambda\to0$ gives
\begin{equation}
    \mu_\star
    \longrightarrow
    V_rS^{-1}U_r^\top y + V_0V_0^\top x_t
    =
    A^\dagger y + P_{\mathcal N(A)}x_t .
\end{equation}

It remains to show that the corresponding anisotropic denoiser converges to the
posterior mean $\E[x_0|y]$. In the same SVD coordinates, the covariance is
\begin{align}
    \Sigma_\star(t)
    &=
    \left(
        \frac{1}{\sigma_t^2}I
        +
        \frac{1}{\sigma_y^2}A^\top A
    \right)^{-1}  \\
    &=
    \sigma_y^2
    V_r(S^2+\lambda I)^{-1}V_r^\top
    +
    \sigma_t^2 V_0V_0^\top .
    \label{eq:app:sigma-svd}
\end{align}
Thus, as $\sigma_t\to\infty$, the row-space covariance converges to
\begin{equation}
    \sigma_y^2 V_rS^{-2}V_r^\top
    =
    \sigma_y^2(A^\top A)^\dagger,
\end{equation}
while the nullspace variance diverges. Therefore the limiting denoising query
contains finite information only in the row space of $A$ and no information in
the nullspace.

More explicitly, the limiting finite-noise observation is
\begin{equation}
    z
    =
    A^\dagger y
    =
    A^\dagger A x_0 + A^\dagger\eta
    =
    P_{\mathcal R(A^\top)}x_0 + \zeta,
    \qquad
    \zeta\sim
    \N\!\left(0,\sigma_y^2(A^\top A)^\dagger\right).
\end{equation}
Since $A^\dagger$ is a deterministic function of $y$, conditioning on $y$
implies conditioning on $z$. Conversely, for Gaussian measurement noise with
known covariance, $z=A^\dagger y$ is a sufficient statistic for $y$ with respect
to $x_0$: the remaining component of $y$ orthogonal to $\mathcal R(A)$ carries no
information about $x_0$. Hence
\begin{equation}
    \E[x_0|z]=\E[x_0|y].
\end{equation}
The limiting anisotropic denoiser is exactly the Bayes estimator associated with
this row-space observation and infinite nullspace uncertainty. Therefore
\begin{equation}
    \lim_{\sigma_t\to\infty}
    D_{\Sigma_\star(t)}(\mu_\star(x_t,y,t))
    =
    \E[x_0|y].
\end{equation}
This proves Observation~\ref{obs:one-step-limit}.

\subsection{Equivalent-Time and GLASS Limit}
\label{app:derivation:equivalent}

When $A^\top A=\gamma^2 I_d$, the covariance in \eqref{eq:sigma-mu-star} is isotropic:
\begin{equation}
    \Sigma_\star(t)=\sigma_\star^2(t)I_d,
    \qquad
    \sigma_\star^2(t)=\frac{\beta_t^2\sigma_y^2}{\alpha_t^2\sigma_y^2+\beta_t^2\gamma^2}.
\end{equation}
In this special case the posterior denoising query can be represented by an equivalent scalar noise level. If the base denoiser is trained for isotropic corruptions with effective noise ratio $\beta_s^2/\alpha_s^2$, we choose $s=t^\star$ such that
\begin{equation}
    \frac{\beta_{t^\star}^2}{\alpha_{t^\star}^2}=\sigma_\star^2(t).
\end{equation}
Then $D_{\Sigma_\star(t)}(\mu_\star)$ can be approximated by the pretrained isotropic denoiser at time $t^\star$. This is the setting in which an equivalent-time, training-free reduction such as GLASS \cite{holderrieth2025glassflowstransitionsampling} is available. For a general inverse problem, however, $A^\top A$ has different eigenvalues and often a nontrivial nullspace; $\Sigma_\star$ is then anisotropic and no single scalar time can represent the posterior denoising kernel. EPS fine-tuning is introduced precisely to learn this missing anisotropic denoising geometry.

\subsection{Connection to ridge regression and Gaussian processes.} 
\label{app:derivation:gp}

Equation~\eqref{eq:sigma-mu-star} is exactly the linear-Gaussian Bayesian update familiar from ridge regression and Gaussian process regression. Here $\sigma_y^{-2} A^\top A$ plays the role of the data precision, $(\alpha_t^2/\beta_t^2) I_d$ plays the role of the prior precision (set by the current diffusion noise level), and $\mu_\star$ is their precision-weighted mean. As $\sigma_t$ grows the prior precision vanishes and $\mu_\star$ converges to the data-only ridge solution, while as $\sigma_t$ shrinks $\mu_\star$ collapses onto $x_t$. The covariance $\Sigma_\star(t)$ is the corresponding posterior covariance, which shrinks along measured directions (large eigenvalues of $A^\top A$) and remains diffuse along weakly observed directions, mirroring the heteroscedastic uncertainty of GP posteriors \cite{williams2006gaussian}.

\clearpage
\section{Implementation Details}
\label{app:implementation}

\paragraph{Architecture.} We use the EDM-ADM
checkpoint~\citep{karras2022elucidating} for ImageNet-$64{\times}64$
(\texttt{edm-imagenet-64x64-cond-adm.pkl}, $\sim$$296$M parameters,
class-conditional with 1000-way one-hot embedding), and an EDM-DDPM++
checkpoint we trained from scratch for FFHQ-$64{\times}64$. EPS extends
the first conv from $3$ to $6$ input channels: the mean pivot
$\mu_\star$ ($3$ channels) concatenated with a task-specific
observation tensor  ($3$ channels). The added input channels
are zero-initialised so the network reproduces the unconditional
pretrained mapping at step zero of fine-tuning. The observation
tensor is the masked observation $y$ for inpainting, the
nearest-neighbour upsampling of $y$ for super-resolution, and the
blurred observation $y$ for deblurring; the total input width is
therefore $6$ channels for every task.

\paragraph{Pivot solver.} For binary masks $\Sigma_\star$ is diagonal
and the pivot solve is element-wise. For $4{\times}$ super-resolution
the structured solve uses average-pool / nearest-upsample primitives
in $O(d)$. For circular blur kernels we diagonalize $A^{\top}A$ by
the 2D FFT; the per-step solve is then a complex element-wise divide
plus an inverse FFT, $O(d \log d)$. Empirically the pivot solve
contributes ${<}1\,\text{ms}$ per step compared to a U-Net forward of
${\sim}19\,\text{ms}$ at batch~$1$.

\paragraph{Optimization.} We use the EDM optimizer stack unchanged:
Adam ($\beta_1{=}0.9$, $\beta_2{=}0.999$, $\varepsilon{=}10^{-8}$);
log-normal noise sampling with $P_{\text{mean}}{=}-1.2$,
$P_{\text{std}}{=}1.2$, $\sigma_{\text{data}}{=}0.5$;
schedule extrema $\sigma_{\min}{=}0.002$, $\sigma_{\max}{=}80$,
$\rho{=}7$. Learning rate is $10^{-4}$ with a $2{\times}10^{6}$-image
linear warm-up. EMA uses a half-life of $500$ kimg with a $5\%$
ramp-up ratio. We weight the loss by the standard EDM weighting
$\lambda(\sigma)={(\sigma^2{+}\sigma_{\text{data}}^2)/(\sigma\,\sigma_{\text{data}})^2}$.
Per-task fine-tuning runs for $10$ epochs ($\sim$$25$k iterations) on
ImageNet-$64$ at batch~$128$ across $4$ NVIDIA B200 GPUs (gradient
accumulation${=}1$); FFHQ-$64$ trains for the same iteration budget
at batch~$192$. End-to-end fine-tuning takes ${\sim}24$\,h on
ImageNet and ${\sim}10$\,h on FFHQ per task.

\paragraph{Operator randomization during training.} We re-sample the
operator at every minibatch step. Random inpainting samples a per-pixel
Bernoulli mask with mask-density $\sim\mathcal{U}(50\%, 70\%)$
(default training density $70\%$). Box inpainting samples a uniformly
random rectangle whose side lengths are drawn from
$\mathcal{U}([H/4, H/2])\times \mathcal{U}([W/4, W/2])$ with margins
$H/16$ at every edge. Motion-deblur kernels are generated with random
length ${\in}\{7, 9, 11, 13, 15\}$ and random angle ${\in}[0,180^{\circ}]$;
Gaussian-deblur kernels use a fixed bandwidth
$\sigma_{\text{blur}}{=}0.75$ at length $11$. Super-resolution applies
a fixed $4{\times}$ average-pool. Observation noise is fixed at
$\sigma_y{=}0.05$ throughout both training and evaluation.

\paragraph{Sampling at inference.} We use the deterministic EDM
Euler ODE sampler for the $20$- and $100$-NFE variants
(\texttt{second\_order=False}). The $1$-NFE variant evaluates the
denoiser once at the highest noise level $\sigma_{\max}$ via the
high-noise pivot of Observation~\ref{obs:one-step-limit}; the resulting
single forward pass returns the conditional MMSE estimator
$\mathbb{E}[x_0|y]$ directly. Class labels at inference time use the
ground-truth ImageNet-1k class for ImageNet-$64$ and are not used for
FFHQ-$64$.

\paragraph{Reproducibility.} All random seeds are fixed; the same
$100$ evaluation images and $10$ posterior seeds per image are used
for every method. The structured-solve, training-loop, and sampler
implementations will be released as open-source upon acceptance, along
with all per-task fine-tuned checkpoints.

\clearpage
\section{Broader Impact}
\label{app:broader-impact}

EPS provides a principled, calibrated approach to posterior sampling
for linear inverse problems. Because it preserves the input/output
structure of standard denoising pretraining, an existing pretrained
prior can be repurposed into an uncertainty-aware posterior sampler
with a lightweight fine-tune rather than a from-scratch retraining.
This makes the method straightforward to adapt across scientific
imaging applications where reliable reconstructions and quantified
posterior uncertainty are valuable, and the closed-form pivot and
covariance offer a transparent handle for analyzing the sampler in any
downstream pipeline. Beyond these benefits, our work does not
introduce societal impacts that go meaningfully beyond those of the
existing generative diffusion priors it builds on; the standard
considerations around dual-use of high-fidelity image generation and
the demographic or domain biases of the underlying training data
continue to apply.

\clearpage
\section{Additional Experiments}
\label{app:additional-experiments}
We collect ablations and extended tables that support the main claims. 

\subsection{Input Configuration Ablation}
\label{app:additional-experiments:inputs}
The central input ablation compares raw-state conditioning to shifted-pivot conditioning while keeping the backbone, compute budget, and EDM warm start fixed. We evaluate four input streams to the denoiser: (a) $[x_t,y,t]$ (Palette-style), the standard conditional baseline that feeds the noised latent alongside the observation and exposes no closed-form posterior structure; (b) $[\mu_\star,t]$, which replaces $x_t$ with the closed-form posterior mean $\mu_\star(x_t,y,\sigma_t)$ obtained by Gaussian-merging $p(x_t\mid x_0)$ and $p(y\mid x_0)$ inside the integral, dropping $y$ from the input; (c) $[\mu_\star,\Sigma_\star,t]$, which additionally passes the per-component posterior covariance $\Sigma_\star$ as a side channel, giving the network explicit access to the local anisotropic uncertainty; and (d) $[\mu_\star,y,t]$ (EPS, ours), the full EPS input where the posterior mean is concatenated with the raw observation, allowing the network to use $y$ both directly and through the analytical pivot.

Table~\ref{tab:input-config} reports this ablation on FFHQ-64 (DDPM++/EDM backbone) and ImageNet-64 (EDM-ADM backbone) at NFE$=$100 with the EDM Euler sampler. Both backbones are warm-started from the same pretrained checkpoint and fine-tuned under matched protocols. The progression Palette$\,\to\,\mu_\star\to\,\mu_\star{+}\Sigma_\star\to$EPS is monotone on every distortion and distributional metric in the average and on most tasks individually, on both datasets. Replacing $x_t$ by $\mu_\star$ explains most of the gain, and the auxiliary observation channel provides an additional anchor.
\begin{table*}[t]
    \centering
    \scriptsize
    \setlength{\tabcolsep}{4pt}
    \caption{\textbf{The shifted pivot $\mu_\star$ is the right input.} Input-configuration ablation on FFHQ-64 (top) and ImageNet-64 (bottom) at NFE$=$100 with the EDM Euler sampler. The Palette$\,\to\,\mu_\star\to\,\mu_\star{+}\Sigma_\star\to$EPS progression is monotone on every metric in the average and on most tasks individually, on both datasets. Best in \textbf{bold}, second-best \underline{underlined}; the EPS row is highlighted in light pink.}
    \label{tab:input-config}
    \resizebox{\textwidth}{!}{%
        \begin{tabular}{llcccccccc}
            \toprule
            Task & Input & PSNR $\uparrow$ & SSIM $\uparrow$ & LPIPS $\downarrow$ & FID $\downarrow$ & MMD-pix $\downarrow$ & MMD-Inc $\downarrow$ & CRPS-pix $\downarrow$ & CRPS-Inc $\downarrow$ \\
            \midrule
            \multicolumn{10}{c}{\textit{FFHQ $64{\times}64$}} \\
            \midrule
            \multirow{4}{*}{Average}
            & $[x_t, y, t]$ (Palette) & 26.03 & 0.8590 & 0.0626 & 31.50 & -6.6e-03 & -4.7e-03 & 3.43 & 3.86 \\
            & $[\mu_\star, t]$ & 26.39 & 0.8583 & 0.0632 & 31.34 & \underline{-6.7e-03} & -4.6e-03 & 3.30 & 3.81 \\
            & $[\mu_\star, \Sigma_\star, t]$ & \underline{26.62} & \underline{0.8636} & \underline{0.0603} & \underline{30.20} & -6.7e-03 & \underline{-4.8e-03} & \underline{3.27} & \underline{3.76} \\
            \epsrow \cellcolor{white} & $[\mu_\star, y, t]$ (EPS) & \textbf{26.69} & \textbf{0.8661} & \textbf{0.0590} & \textbf{29.94} & \textbf{-6.7e-03} & \textbf{-4.9e-03} & \textbf{3.24} & \textbf{3.73} \\
            \midrule
            \multirow{4}{*}{Random inpaint}
            & $[x_t, y, t]$ (Palette) & 25.76 & 0.8809 & 0.0593 & 33.30 & \underline{-6.7e-03} & -4.8e-03 & 3.31 & 3.87 \\
            & $[\mu_\star, t]$ & 24.87 & 0.8564 & 0.0698 & 37.92 & -6.7e-03 & -3.9e-03 & 3.39 & 4.08 \\
            & $[\mu_\star, \Sigma_\star, t]$ & \underline{26.00} & \underline{0.8845} & \underline{0.0546} & \underline{32.10} & -6.7e-03 & \underline{-5.0e-03} & \underline{3.22} & \underline{3.76} \\
            \epsrow \cellcolor{white} & $[\mu_\star, y, t]$ (EPS) & \textbf{26.16} & \textbf{0.8879} & \textbf{0.0533} & \textbf{31.87} & \textbf{-6.7e-03} & \textbf{-5.0e-03} & \textbf{3.16} & \textbf{3.75} \\
            \midrule
            \multirow{4}{*}{Box inpaint}
            & $[x_t, y, t]$ (Palette) & \underline{24.18} & 0.8426 & 0.0577 & 25.09 & -6.6e-03 & -5.3e-03 & \underline{4.02} & 3.47 \\
            & $[\mu_\star, t]$ & 24.14 & \underline{0.8431} & 0.0575 & 24.94 & \textbf{-6.6e-03} & \underline{-5.4e-03} & 4.02 & \underline{3.45} \\
            & $[\mu_\star, \Sigma_\star, t]$ & 24.17 & 0.8430 & \underline{0.0574} & \underline{24.88} & -6.6e-03 & \textbf{-5.4e-03} & 4.03 & 3.45 \\
            \epsrow \cellcolor{white} & $[\mu_\star, y, t]$ (EPS) & \textbf{24.23} & \textbf{0.8448} & \textbf{0.0567} & \textbf{24.74} & \underline{-6.6e-03} & -5.4e-03 & \textbf{4.01} & \textbf{3.45} \\
            \midrule
            \multirow{4}{*}{Super-res ($4{\times}$)}
            & $[x_t, y, t]$ (Palette) & \underline{21.95} & \underline{0.7220} & \underline{0.1273} & \textbf{49.28} & \underline{-6.3e-03} & \underline{-3.0e-03} & \textbf{5.22} & \textbf{5.29} \\
            & $[\mu_\star, t]$ & 21.89 & 0.7179 & 0.1300 & 49.85 & -6.3e-03 & -2.8e-03 & 5.29 & 5.36 \\
            & $[\mu_\star, \Sigma_\star, t]$ & 21.86 & 0.7162 & 0.1308 & 50.14 & -6.2e-03 & -2.7e-03 & 5.30 & 5.39 \\
            \epsrow \cellcolor{white} & $[\mu_\star, y, t]$ (EPS) & \textbf{21.96} & \textbf{0.7232} & \textbf{0.1262} & \underline{49.29} & \textbf{-6.3e-03} & \textbf{-3.0e-03} & \underline{5.23} & \underline{5.30} \\
            \midrule
            \multirow{4}{*}{Gaussian deblur}
            & $[x_t, y, t]$ (Palette) & 30.47 & 0.9397 & 0.0286 & 21.63 & -6.9e-03 & -5.6e-03 & 1.91 & 3.06 \\
            & $[\mu_\star, t]$ & \underline{30.82} & \underline{0.9408} & \textbf{0.0273} & 20.72 & \textbf{-6.9e-03} & -5.7e-03 & \underline{1.84} & 2.99 \\
            & $[\mu_\star, \Sigma_\star, t]$ & 30.82 & 0.9408 & 0.0273 & \underline{20.70} & -6.9e-03 & \textbf{-5.7e-03} & 1.84 & \underline{2.99} \\
            \epsrow \cellcolor{white} & $[\mu_\star, y, t]$ (EPS) & \textbf{30.82} & \textbf{0.9408} & \underline{0.0273} & \textbf{20.68} & \underline{-6.9e-03} & \underline{-5.7e-03} & \textbf{1.84} & \textbf{2.99} \\
            \midrule
            \multirow{4}{*}{Motion deblur}
            & $[x_t, y, t]$ (Palette) & 27.79 & 0.9099 & 0.0404 & 28.23 & -6.8e-03 & -5.0e-03 & 2.67 & 3.59 \\
            & $[\mu_\star, t]$ & 30.23 & 0.9334 & 0.0314 & 23.24 & \underline{-6.9e-03} & -5.4e-03 & 1.94 & \underline{3.18} \\
            & $[\mu_\star, \Sigma_\star, t]$ & \underline{30.25} & \underline{0.9337} & \underline{0.0312} & \underline{23.16} & -6.9e-03 & \underline{-5.4e-03} & \underline{1.94} & 3.18 \\
            \epsrow \cellcolor{white} & $[\mu_\star, y, t]$ (EPS) & \textbf{30.27} & \textbf{0.9339} & \textbf{0.0311} & \textbf{23.10} & \textbf{-6.9e-03} & \textbf{-5.4e-03} & \textbf{1.94} & \textbf{3.18} \\
            \midrule
            \multicolumn{10}{c}{\textit{ImageNet $64{\times}64$}} \\
            \midrule
            \multirow{4}{*}{Average}
            & $[x_t, y, t]$ (Palette) & 24.32 & 0.7673 & 0.1124 & 82.57 & -6.4e-03 & -4.4e-03 & 4.35 & 5.54 \\
            & $[\mu_\star, t]$ & 24.36 & 0.7627 & 0.1151 & 83.64 & -6.4e-03 & -4.3e-03 & 4.31 & 5.55 \\
            & $[\mu_\star, \Sigma_\star, t]$ & \underline{24.52} & \underline{0.7699} & \underline{0.1115} & \underline{82.06} & \underline{-6.4e-03} & \underline{-4.4e-03} & \underline{4.29} & \underline{5.50} \\
            \epsrow \cellcolor{white} & $[\mu_\star, y, t]$ (EPS) & \textbf{24.53} & \textbf{0.7712} & \textbf{0.1103} & \textbf{81.46} & \textbf{-6.4e-03} & \textbf{-4.4e-03} & \textbf{4.27} & \textbf{5.48} \\
            \midrule
            \multirow{4}{*}{Random inpaint}
            & $[x_t, y, t]$ (Palette) & 24.09 & 0.7869 & 0.1011 & 81.88 & -6.5e-03 & -4.4e-03 & 4.16 & 5.52 \\
            & $[\mu_\star, t]$ & 23.53 & 0.7588 & 0.1173 & 88.24 & -6.5e-03 & -4.2e-03 & 4.18 & 5.66 \\
            & $[\mu_\star, \Sigma_\star, t]$ & \textbf{24.35} & \underline{0.7944} & \underline{0.0986} & \underline{80.19} & \underline{-6.5e-03} & \underline{-4.4e-03} & \underline{4.05} & \underline{5.43} \\
            \epsrow \cellcolor{white} & $[\mu_\star, y, t]$ (EPS) & \underline{24.34} & \textbf{0.7948} & \textbf{0.0979} & \textbf{79.60} & \textbf{-6.5e-03} & \textbf{-4.5e-03} & \textbf{4.04} & \textbf{5.41} \\
            \midrule
            \multirow{4}{*}{Box inpaint}
            & $[x_t, y, t]$ (Palette) & 21.12 & 0.7541 & 0.1218 & 92.73 & \underline{-6.1e-03} & -4.1e-03 & 5.92 & 5.93 \\
            & $[\mu_\star, t]$ & \underline{21.19} & \underline{0.7552} & \underline{0.1209} & 91.74 & -6.1e-03 & \underline{-4.2e-03} & \underline{5.88} & 5.86 \\
            & $[\mu_\star, \Sigma_\star, t]$ & 21.18 & 0.7544 & 0.1212 & \textbf{91.03} & -6.1e-03 & \textbf{-4.2e-03} & 5.92 & \textbf{5.82} \\
            \epsrow \cellcolor{white} & $[\mu_\star, y, t]$ (EPS) & \textbf{21.24} & \textbf{0.7569} & \textbf{0.1196} & \underline{91.07} & \textbf{-6.1e-03} & -4.2e-03 & \textbf{5.87} & \underline{5.84} \\
            \midrule
            \multirow{4}{*}{Super-res ($4{\times}$)}
            & $[x_t, y, t]$ (Palette) & \underline{20.24} & \underline{0.5364} & \underline{0.2220} & \textbf{128.76} & \textbf{-5.9e-03} & \underline{-2.8e-03} & \textbf{6.50} & \textbf{7.33} \\
            & $[\mu_\star, t]$ & 20.20 & 0.5308 & 0.2246 & 130.53 & -5.8e-03 & -2.6e-03 & 6.55 & 7.39 \\
            & $[\mu_\star, \Sigma_\star, t]$ & 20.20 & 0.5319 & 0.2250 & 131.17 & -5.8e-03 & -2.7e-03 & 6.54 & 7.40 \\
            \epsrow \cellcolor{white} & $[\mu_\star, y, t]$ (EPS) & \textbf{20.25} & \textbf{0.5369} & \textbf{0.2207} & \underline{128.80} & \underline{-5.9e-03} & \textbf{-2.8e-03} & \underline{6.52} & \underline{7.35} \\
            \midrule
            \multirow{4}{*}{Gaussian deblur}
            & $[x_t, y, t]$ (Palette) & 29.15 & 0.9010 & 0.0491 & \underline{46.62} & \textbf{-6.8e-03} & \textbf{-5.6e-03} & 2.26 & \textbf{4.11} \\
            & $[\mu_\star, t]$ & 29.18 & \underline{0.9014} & \textbf{0.0486} & 46.76 & -6.8e-03 & -5.5e-03 & 2.25 & 4.12 \\
            & $[\mu_\star, \Sigma_\star, t]$ & \underline{29.18} & 0.9014 & 0.0487 & 46.73 & -6.8e-03 & -5.5e-03 & \underline{2.25} & 4.13 \\
            \epsrow \cellcolor{white} & $[\mu_\star, y, t]$ (EPS) & \textbf{29.18} & \textbf{0.9015} & \underline{0.0486} & \textbf{46.55} & \underline{-6.8e-03} & \underline{-5.6e-03} & \textbf{2.25} & \underline{4.11} \\
            \midrule
            \multirow{4}{*}{Motion deblur}
            & $[x_t, y, t]$ (Palette) & 27.02 & 0.8582 & 0.0680 & 62.86 & -6.7e-03 & -5.0e-03 & 2.93 & 4.81 \\
            & $[\mu_\star, t]$ & \textbf{27.68} & \textbf{0.8674} & \textbf{0.0641} & \textbf{60.94} & \textbf{-6.8e-03} & \textbf{-5.1e-03} & \textbf{2.68} & 4.70 \\
            & $[\mu_\star, \Sigma_\star, t]$ & \underline{27.67} & \underline{0.8673} & \underline{0.0642} & \underline{61.19} & \underline{-6.8e-03} & -5.1e-03 & \underline{2.68} & \textbf{4.70} \\
            \epsrow \cellcolor{white} & $[\mu_\star, y, t]$ (EPS) & 27.62 & 0.8661 & 0.0647 & 61.29 & -6.8e-03 & \underline{-5.1e-03} & 2.69 & \underline{4.70} \\
            \bottomrule
        \end{tabular}%
    }
\end{table*}

\clearpage
\subsection{Zero-Shot Pivoting}
\label{app:additional-experiments:zeroshot}
Before fine-tuning, we can feed $\mu_\star$ directly to the pretrained denoiser (with the closest available EDM noise level, or the exact equivalent time in isotropic cases). This is not exact for general $A$ because the pretrained model has not learned anisotropic denoising, but it tests whether the pivot already carries useful posterior information.

Table~\ref{tab:zeroshot} compares zero-shot pivoting to fine-tuned EPS and to Palette under the same 100-step Euler sampler on FFHQ-64 and ImageNet-64. Zero-shot pivoting feeds $\mu_\star$ directly to the pretrained EDM denoiser at the current $\sigma_t$ (no fine-tuning) and runs the standard 100-step Euler loop; fine-tuned EPS uses the same sampler but with the denoiser adapted to take $[\mu_\star,y]$ on each task. Zero-shot pivoting underperforms both Palette and fine-tuned EPS on every task, confirming that the pretrained denoiser does not natively handle the anisotropic geometry of $\Sigma_\star(t)$, and that the EPS fine-tuning step is what unlocks the benefit of the pivot.

\begin{table*}[ht!]
    \centering
    \scriptsize
    \setlength{\tabcolsep}{4pt}
    \caption{\textbf{The pivot needs the fine-tuning step.} Zero-shot pivoting vs.\ fine-tuned EPS at NFE$=$100 (Euler sampler) on FFHQ-64 (top) and ImageNet-64 (bottom). Feeding $\mu_\star$ to the pretrained denoiser without adaptation underperforms Palette and fine-tuned EPS on every task and every metric, on both datasets. Best in \textbf{bold}, second-best \underline{underlined}; the EPS row is highlighted in light pink.}
    \label{tab:zeroshot}
    \resizebox{\textwidth}{!}{%
        \begin{tabular}{llcccccccc}
            \toprule
            Task & Method & PSNR $\uparrow$ & SSIM $\uparrow$ & LPIPS $\downarrow$ & FID $\downarrow$ & MMD-pix $\downarrow$ & MMD-Inc $\downarrow$ & CRPS-pix $\downarrow$ & CRPS-Inc $\downarrow$ \\
            \midrule
            \multicolumn{10}{c}{\textit{FFHQ $64{\times}64$}} \\
            \midrule
            \multirow{3}{*}{Average}
            & Palette & \underline{26.03} & \underline{0.8590} & \underline{0.0626} & \underline{31.50} & \underline{-6.65e-03} & \underline{-4.74e-03} & \underline{3.43} & \underline{3.86} \\
            & Zero-shot EPS pivot & 22.97 & 0.7386 & 0.1503 & 70.88 & -3.14e-03 & 3.56e-02 & 5.61 & 6.68 \\
            \epsrow \cellcolor{white} \cellcolor{white} & EPS fine-tuned & \textbf{26.69} & \textbf{0.8661} & \textbf{0.0590} & \textbf{29.94} & \textbf{-6.69e-03} & \textbf{-4.89e-03} & \textbf{3.24} & \textbf{3.73} \\
            \midrule
            \multirow{3}{*}{Random inpaint}
            & Palette & \underline{25.76} & \underline{0.8809} & \underline{0.0593} & \underline{33.30} & \underline{-6.71e-03} & \underline{-4.76e-03} & \underline{3.31} & \underline{3.87} \\
            & Zero-shot EPS pivot & 16.36 & 0.5302 & 0.2985 & 121.98 & 6.58e-03 & 1.02e-01 & 9.08 & 8.66 \\
            \epsrow \cellcolor{white} & EPS fine-tuned & \textbf{26.16} & \textbf{0.8879} & \textbf{0.0533} & \textbf{31.87} & \textbf{-6.74e-03} & \textbf{-5.00e-03} & \textbf{3.16} & \textbf{3.75} \\
            \midrule
            \multirow{3}{*}{Box inpaint}
            & Palette & \underline{24.18} & \underline{0.8426} & \underline{0.0577} & \underline{25.09} & \underline{-6.58e-03} & \underline{-5.30e-03} & \underline{4.02} & \underline{3.47} \\
            & Zero-shot EPS pivot & 18.75 & 0.7065 & 0.1699 & 84.24 & -5.19e-03 & 4.32e-02 & 6.71 & 6.69 \\
            \epsrow \cellcolor{white} & EPS fine-tuned & \textbf{24.23} & \textbf{0.8448} & \textbf{0.0567} & \textbf{24.74} & \textbf{-6.59e-03} & \textbf{-5.35e-03} & \textbf{4.01} & \textbf{3.45} \\
            \midrule
            \multirow{3}{*}{Super-res ($4{\times}$)}
            & Palette & \underline{21.95} & \underline{0.7220} & \underline{0.1273} & \textbf{49.28} & \underline{-6.29e-03} & \underline{-2.98e-03} & \textbf{5.22} & \textbf{5.29} \\
            & Zero-shot EPS pivot & 20.50 & 0.6234 & 0.1713 & 57.45 & -3.68e-03 & 4.23e-03 & 6.62 & 6.20 \\
            \epsrow \cellcolor{white} & EPS fine-tuned & \textbf{21.96} & \textbf{0.7232} & \textbf{0.1262} & \underline{49.29} & \textbf{-6.30e-03} & \textbf{-3.00e-03} & \underline{5.23} & \underline{5.30} \\
            \midrule
            \multirow{3}{*}{Gaussian deblur}
            & Palette & \underline{30.47} & \underline{0.9397} & \underline{0.0286} & \underline{21.63} & \underline{-6.90e-03} & \underline{-5.65e-03} & \underline{1.91} & \underline{3.06} \\
            & Zero-shot EPS pivot & 29.79 & 0.9197 & 0.0540 & 43.77 & -6.73e-03 & 1.31e-02 & 2.76 & 5.79 \\
            \epsrow \cellcolor{white} & EPS fine-tuned & \textbf{30.82} & \textbf{0.9408} & \textbf{0.0273} & \textbf{20.68} & \textbf{-6.91e-03} & \textbf{-5.73e-03} & \textbf{1.84} & \textbf{2.99} \\
            \midrule
            \multirow{3}{*}{Motion deblur}
            & Palette & 27.79 & 0.9099 & \underline{0.0404} & \underline{28.23} & \underline{-6.77e-03} & \underline{-4.98e-03} & \underline{2.67} & \underline{3.59} \\
            & Zero-shot EPS pivot & \underline{29.42} & \underline{0.9135} & 0.0579 & 46.95 & -6.69e-03 & 1.54e-02 & 2.89 & 6.05 \\
            \epsrow \cellcolor{white} & EPS fine-tuned & \textbf{30.27} & \textbf{0.9339} & \textbf{0.0311} & \textbf{23.10} & \textbf{-6.90e-03} & \textbf{-5.39e-03} & \textbf{1.94} & \textbf{3.18} \\
            \midrule
            \multicolumn{10}{c}{\textit{ImageNet $64{\times}64$}} \\
            \midrule
            \multirow{3}{*}{Average}
            & Palette & \underline{24.32} & \underline{0.7673} & \underline{0.1124} & \underline{82.57} & \underline{-6.40e-03} & \underline{-4.38e-03} & \underline{4.35} & \underline{5.54} \\
            & Zero-shot EPS pivot & 22.02 & 0.6398 & 0.2461 & 146.88 & 2.48e-03 & 1.06e-02 & 6.96 & 8.89 \\
            \epsrow \cellcolor{white} & EPS fine-tuned & \textbf{24.53} & \textbf{0.7712} & \textbf{0.1103} & \textbf{81.46} & \textbf{-6.42e-03} & \textbf{-4.45e-03} & \textbf{4.27} & \textbf{5.48} \\
            \midrule
            \multirow{3}{*}{Random inpaint}
            & Palette & \underline{24.09} & \underline{0.7869} & \underline{0.1011} & \underline{81.88} & \underline{-6.50e-03} & \underline{-4.41e-03} & \underline{4.16} & \underline{5.52} \\
            & Zero-shot EPS pivot & 16.92 & 0.4168 & 0.4350 & 223.68 & 2.61e-02 & 4.15e-02 & 10.14 & 11.23 \\
            \epsrow \cellcolor{white} & EPS fine-tuned & \textbf{24.34} & \textbf{0.7948} & \textbf{0.0979} & \textbf{79.60} & \textbf{-6.52e-03} & \textbf{-4.51e-03} & \textbf{4.04} & \textbf{5.41} \\
            \midrule
            \multirow{3}{*}{Box inpaint}
            & Palette & \underline{21.12} & \underline{0.7541} & \underline{0.1218} & \underline{92.73} & \underline{-6.10e-03} & \underline{-4.12e-03} & \underline{5.92} & \underline{5.93} \\
            & Zero-shot EPS pivot & 18.70 & 0.6442 & 0.2333 & 139.38 & -1.38e-03 & 2.45e-03 & 8.56 & 8.10 \\
            \epsrow \cellcolor{white} & EPS fine-tuned & \textbf{21.24} & \textbf{0.7569} & \textbf{0.1196} & \textbf{91.07} & \textbf{-6.11e-03} & \textbf{-4.23e-03} & \textbf{5.87} & \textbf{5.84} \\
            \midrule
            \multirow{3}{*}{Super-res ($4{\times}$)}
            & Palette & \underline{20.24} & \underline{0.5364} & \underline{0.2220} & \textbf{128.76} & \textbf{-5.86e-03} & \underline{-2.79e-03} & \textbf{6.50} & \textbf{7.33} \\
            & Zero-shot EPS pivot & 19.85 & 0.4663 & 0.3035 & 160.80 & -9.25e-04 & 6.29e-03 & 7.91 & 8.72 \\
            \epsrow \cellcolor{white} & EPS fine-tuned & \textbf{20.25} & \textbf{0.5369} & \textbf{0.2207} & \underline{128.80} & \underline{-5.86e-03} & \textbf{-2.84e-03} & \underline{6.52} & \underline{7.35} \\
            \midrule
            \multirow{3}{*}{Gaussian deblur}
            & Palette & \underline{29.15} & \underline{0.9010} & \underline{0.0491} & \underline{46.62} & \textbf{-6.83e-03} & \textbf{-5.56e-03} & \underline{2.26} & \textbf{4.11} \\
            & Zero-shot EPS pivot & 27.83 & 0.8507 & 0.1211 & 97.87 & -5.85e-03 & 7.94e-04 & 3.87 & 7.97 \\
            \epsrow \cellcolor{white} & EPS fine-tuned & \textbf{29.18} & \textbf{0.9015} & \textbf{0.0486} & \textbf{46.55} & \underline{-6.83e-03} & \underline{-5.55e-03} & \textbf{2.25} & \underline{4.11} \\
            \midrule
            \multirow{3}{*}{Motion deblur}
            & Palette & \underline{27.02} & \underline{0.8582} & \underline{0.0680} & \underline{62.86} & \underline{-6.73e-03} & \underline{-5.02e-03} & \underline{2.93} & \underline{4.81} \\
            & Zero-shot EPS pivot & 26.82 & 0.8212 & 0.1374 & 112.64 & -5.59e-03 & 2.19e-03 & 4.31 & 8.42 \\
            \epsrow \cellcolor{white} & EPS fine-tuned & \textbf{27.62} & \textbf{0.8661} & \textbf{0.0647} & \textbf{61.29} & \textbf{-6.77e-03} & \textbf{-5.11e-03} & \textbf{2.69} & \textbf{4.70} \\
            \bottomrule
        \end{tabular}%
    }
\end{table*}

\clearpage
\subsection{Palette vs EPS}
\label{app:additional-experiments:warmstart}

Figure~\ref{fig:finetune} reports fine-tuning convergence curves for EPS and Palette warm-started from the same pretrained EDM checkpoint. Rows index five (dataset, task) pairs (ImageNet-64 random inpainting, ImageNet-64 motion deblurring, FFHQ-64 random inpainting, FFHQ-64 motion deblurring, FFHQ-64 Gaussian deblurring); columns track training loss, PSNR, SSIM, LPIPS, and FID against the number of fine-tuning iterations under the matched 100-step Euler sampler. Three patterns hold across all rows. First, EPS starts from a much better initialization on every metric: at iteration zero, the shifted pivot $\mu_\star$ already encodes enough of the measurement geometry that PSNR and SSIM are within a few units of their converged values, while Palette has to climb from the unconditional prior. Second, EPS reaches its asymptotic LPIPS and FID in a small fraction of the iterations Palette needs and stays at least as low for the rest of training. Third, the gap is largest on the deblurring tasks, where the pivot provides the strongest measurement signal at initialization, and on the perceptual and distributional metrics (LPIPS, FID), where the conditioning structure matters most. This is consistent with the structural-locality argument: EPS preserves the input/output type and Gaussian-corrupted-target geometry of the pretrained denoising task, and only adapts to the operator-induced anisotropic covariance $\Sigma_\star(t)$.

\begin{figure*}[ht!]
\centering
\includegraphics[width=\textwidth]{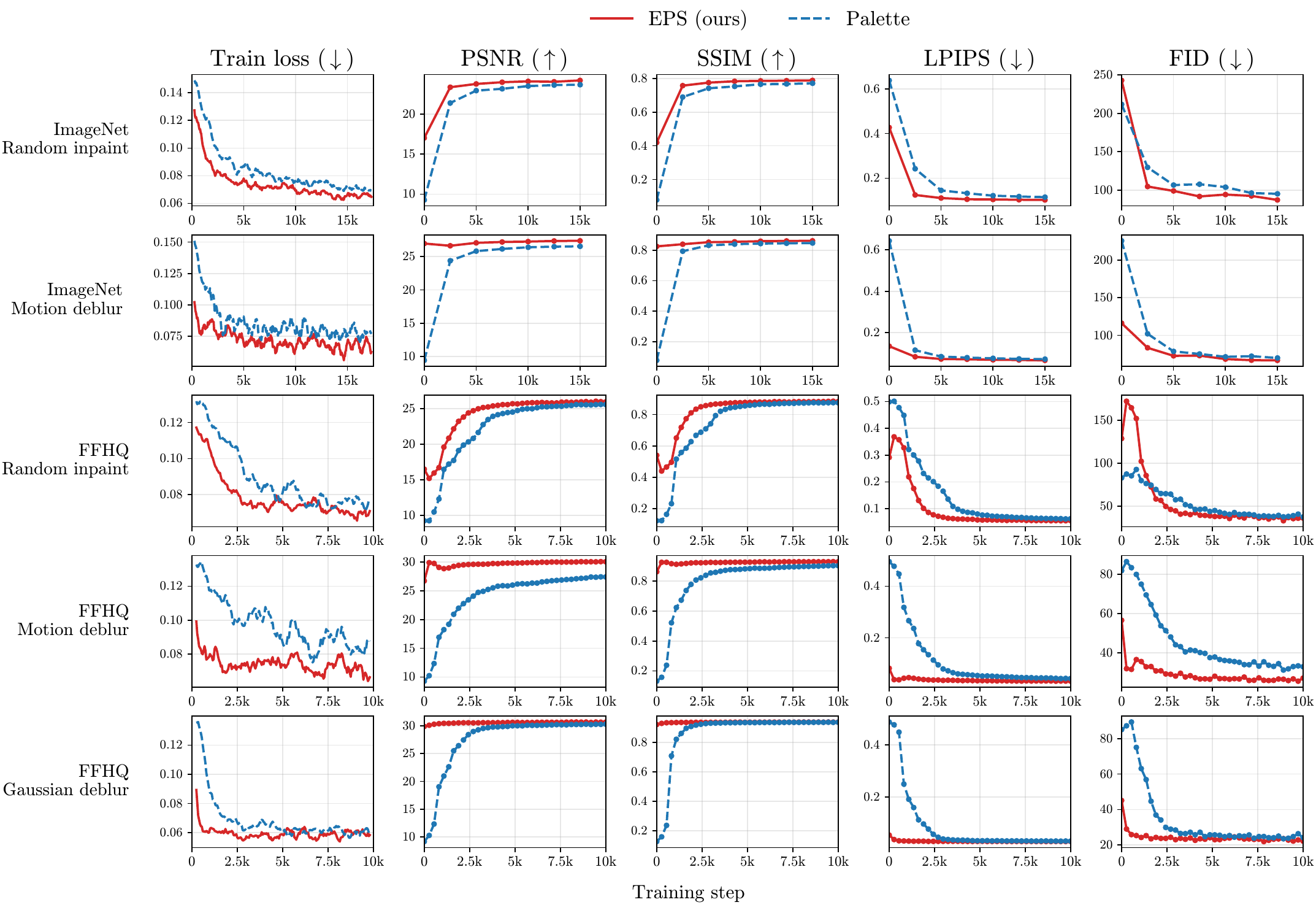}
\caption{\textbf{EPS converges faster than Palette from the same warm start.} Fine-tuning curves for EPS and Palette, both initialized from the same pretrained EDM checkpoint, on five (dataset, task) pairs (rows) and five metrics (columns: training loss, PSNR, SSIM, LPIPS, FID) under the matched 100-step Euler sampler. EPS starts from a markedly better initialization on every metric and reaches its asymptote in a small fraction of the iterations Palette needs.}
\label{fig:finetune}
\end{figure*}

\clearpage
\subsection{Sampling Efficiency}
\label{app:additional-experiments:sampling}
We sweep NFE at inference from 5 to 100 across all five tasks, with all methods using a 1-NFE-per-step Euler sampler under matched conditions.

Figures~\ref{fig:sampling-steps:ffhq} and~\ref{fig:sampling-steps:imagenet} report PSNR, FID, and Inception-feature CRPS as a function of sampler iterations on FFHQ-64 and ImageNet-64 respectively. Rows index the five tasks (random inpaint, box inpaint, $4{\times}$ super-resolution, Gaussian deblur, motion deblur); columns track the three reported metrics. EPS reaches its asymptotic FID and CRPS within roughly 15-20 steps on every task and stays flat thereafter, while sampling-based baselines either fail to reach the same level (DPS, MPGD) or are slower to converge (DDNM, $\Pi$GDM). The flat right tail of the EPS curves is the practical justification for the 20-NFE setting reported in the main results (Tables~\ref{tab:imagenet} and~\ref{tab:ffhq}), and the gap to the strongest sampling-based baseline ($\Pi$GDM) widens on the deblurring tasks under the more diverse ImageNet distribution.

\begin{figure}[p]
\centering
\includegraphics[width=\textwidth]{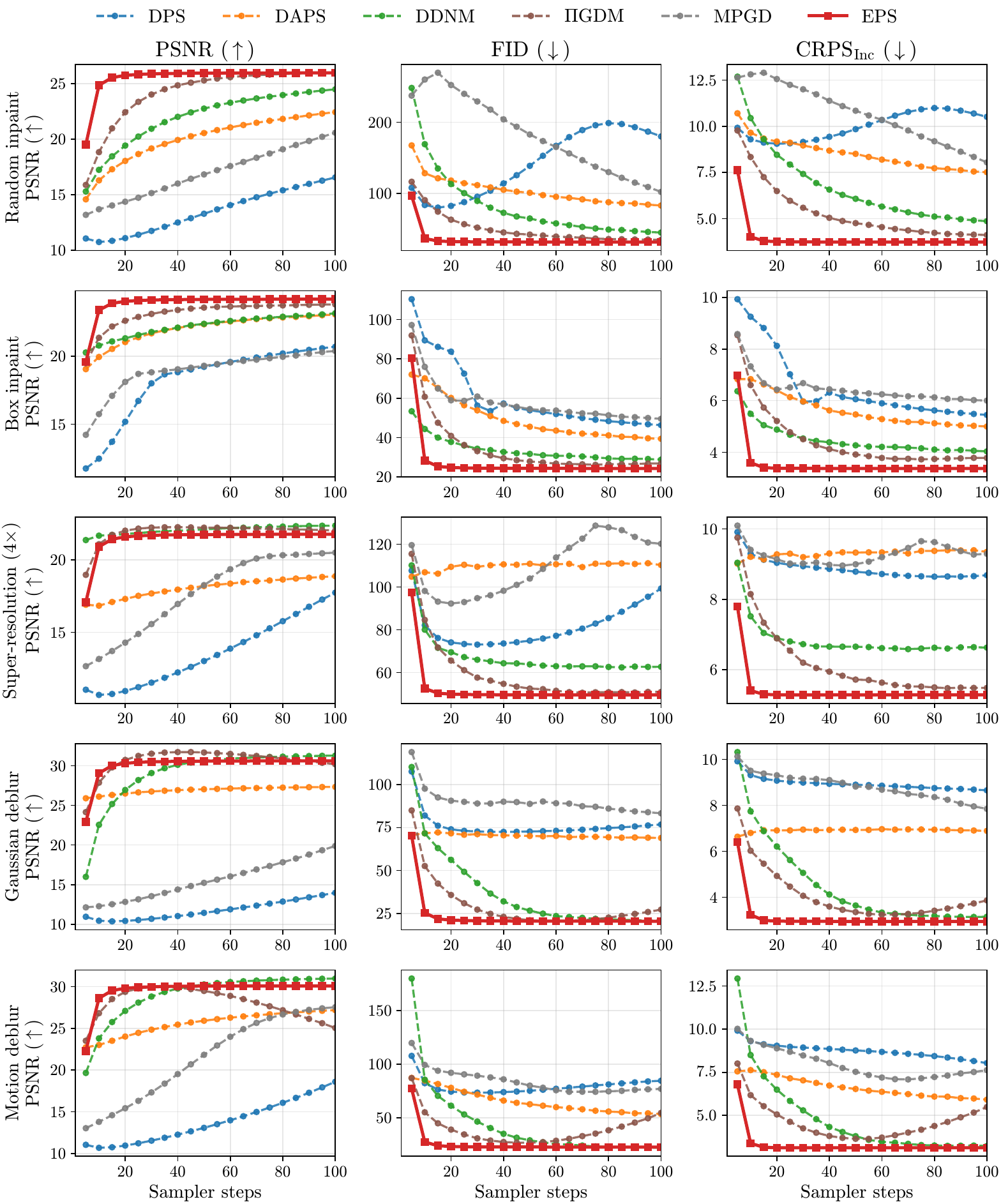}
\caption{\textbf{EPS plateaus by 15-20 steps on FFHQ-64.} Sampling-step sensitivity on FFHQ-64 across the five tasks (rows). Columns report PSNR ($\uparrow$), FID ($\downarrow$), and Inception-feature CRPS ($\downarrow$) versus sampler iterations under a 1-NFE-per-step Euler sampler. EPS reaches its asymptote within roughly 15-20 steps on every task and remains best or tied-best on FID and CRPS thereafter.}
\label{fig:sampling-steps:ffhq}
\end{figure}

\begin{figure}[p]
\centering
\includegraphics[width=\textwidth]{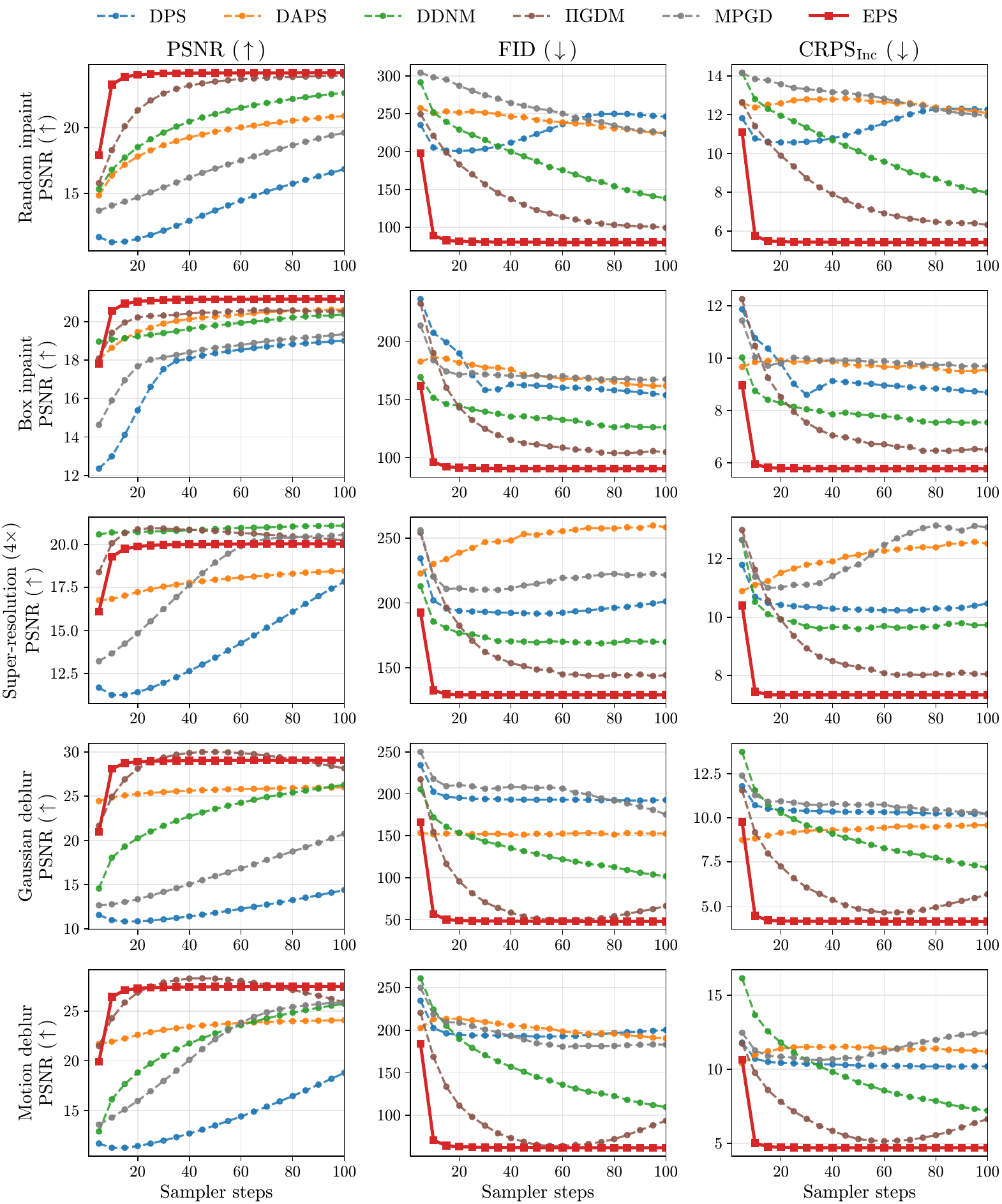}
\caption{\textbf{The plateau transfers to ImageNet-64.} Sampling-step sensitivity on ImageNet-64; same layout as Fig.~\ref{fig:sampling-steps:ffhq}. EPS plateaus at the same step count on a more diverse class-conditional distribution, and the gap between EPS and the strongest sampling-based baseline ($\Pi$GDM) widens on the deblurring tasks.}
\label{fig:sampling-steps:imagenet}
\end{figure}

\clearpage
\subsection{Amortized Variant Across All Five Tasks}
\label{app:additional-experiments:amortized}
As a deployment-friendly alternative to per-task fine-tuning, we train a single EPS checkpoint across all five tasks using uniformly sampled operators per training step and no task indicator at the input.

Table~\ref{tab:amortized} compares per-task EPS, amortized EPS, and Palette under matched compute on FFHQ-64 (top) and ImageNet-64 (bottom). On FFHQ-64 we report two amortized snapshots, at 55k and 160k training steps; the 160k checkpoint surpasses per-task EPS on every task and every metric, indicating that a single network can absorb all five operators without loss. On ImageNet-64, the amortized model is within $0.1$-$0.2$ dB PSNR of per-task EPS and matches Palette or better on the distributional metrics, while requiring 5$\times$ less storage and a single set of weights at deployment.
\begin{table*}[ht!]
    \centering
    \scriptsize
    \setlength{\tabcolsep}{4pt}
    \caption{\textbf{Amortization works.} Amortized EPS vs.\ per-task EPS on FFHQ-64 (top) and ImageNet-64 (bottom) at NFE$=$100 with the EDM Euler sampler. The FFHQ amortized model is reported at two training-step snapshots (55k, 160k); the 160k snapshot surpasses per-task EPS across every task and metric. The ImageNet amortized model (single 296M-param ADM checkpoint, $\sim$60 epochs) matches per-task EPS within $0.1$--$0.2$ dB PSNR and matches Palette or better on the distributional metrics. Best in \textbf{bold}, second-best \underline{underlined}; amortized rows are highlighted in light pink.}
    \label{tab:amortized}
    \resizebox{\textwidth}{!}{%
        \begin{tabular}{llcccccccc}
            \toprule
            Task & Method & PSNR $\uparrow$ & SSIM $\uparrow$ & LPIPS $\downarrow$ & FID $\downarrow$ & MMD-pix $\downarrow$ & MMD-Inc $\downarrow$ & CRPS-pix $\downarrow$ & CRPS-Inc $\downarrow$ \\
            \midrule
            \multicolumn{10}{c}{\textit{FFHQ $64{\times}64$}} \\
            \midrule
            \multirow{4}{*}{Average}
            & Palette & 26.03 & 0.8590 & 0.0626 & 31.50 & -6.65e-03 & -4.74e-03 & 3.43 & 3.86 \\
            & Per-task EPS & 26.69 & 0.8661 & 0.0590 & 29.94 & -6.69e-03 & -4.89e-03 & 3.24 & 3.73 \\
            & \epscell{Amortized EPS (55k steps)} & \epscell{\underline{26.95}} & \epscell{\underline{0.8745}} & \epscell{\underline{0.0552}} & \epscell{\underline{27.86}} & \epscell{\underline{-6.71e-03}} & \epscell{\underline{-5.20e-03}} & \epscell{\underline{3.10}} & \epscell{\underline{3.51}} \\
            & \epscell{Amortized EPS (160k steps)} & \epscell{\textbf{27.31}} & \epscell{\textbf{0.8830}} & \epscell{\textbf{0.0513}} & \epscell{\textbf{26.16}} & \epscell{\textbf{-6.74e-03}} & \epscell{\textbf{-5.25e-03}} & \epscell{\textbf{2.98}} & \epscell{\textbf{3.38}} \\
            \midrule
            \multirow{4}{*}{Random inpaint}
            & Palette & 25.76 & 0.8809 & 0.0593 & 33.30 & -6.71e-03 & -4.76e-03 & 3.31 & 3.87 \\
            & Per-task EPS & 26.16 & 0.8879 & 0.0533 & 31.87 & \underline{-6.74e-03} & -5.00e-03 & 3.16 & 3.75 \\
            & \epscell{Amortized EPS (55k steps)} & \epscell{\underline{26.27}} & \epscell{\underline{0.8920}} & \epscell{\underline{0.0519}} & \epscell{\underline{29.76}} & \epscell{\underline{-6.74e-03}} & \epscell{\underline{-5.20e-03}} & \epscell{\underline{3.10}} & \epscell{\underline{3.56}} \\
            & \epscell{Amortized EPS (160k steps)} & \epscell{\textbf{26.73}} & \epscell{\textbf{0.9001}} & \epscell{\textbf{0.0478}} & \epscell{\textbf{27.57}} & \epscell{\textbf{-6.77e-03}} & \epscell{\textbf{-5.34e-03}} & \epscell{\textbf{2.96}} & \epscell{\textbf{3.40}} \\
            \midrule
            \multirow{4}{*}{Box inpaint}
            & Palette & 24.18 & 0.8426 & 0.0577 & 25.09 & -6.58e-03 & -5.30e-03 & 4.02 & 3.47 \\
            & Per-task EPS & 24.23 & 0.8448 & 0.0567 & 24.74 & -6.59e-03 & -5.35e-03 & 4.01 & 3.45 \\
            & \epscell{Amortized EPS (55k steps)} & \epscell{\underline{24.75}} & \epscell{\underline{0.8580}} & \epscell{\underline{0.0517}} & \epscell{\underline{22.25}} & \epscell{\underline{-6.64e-03}} & \epscell{\underline{-5.60e-03}} & \epscell{\underline{3.74}} & \epscell{\underline{3.18}} \\
            & \epscell{Amortized EPS (160k steps)} & \epscell{\textbf{25.14}} & \epscell{\textbf{0.8663}} & \epscell{\textbf{0.0486}} & \epscell{\textbf{21.07}} & \epscell{\textbf{-6.67e-03}} & \epscell{\textbf{-5.68e-03}} & \epscell{\textbf{3.61}} & \epscell{\textbf{3.07}} \\
            \midrule
            \multirow{4}{*}{Super-res ($4{\times}$)}
            & Palette & 21.95 & 0.7220 & 0.1273 & 49.28 & -6.29e-03 & -2.98e-03 & 5.22 & 5.29 \\
            & Per-task EPS & 21.96 & 0.7232 & 0.1262 & 49.29 & -6.30e-03 & -3.00e-03 & 5.23 & 5.30 \\
            & \epscell{Amortized EPS (55k steps)} & \epscell{\underline{22.33}} & \epscell{\underline{0.7433}} & \epscell{\underline{0.1158}} & \epscell{\underline{45.53}} & \epscell{\underline{-6.36e-03}} & \epscell{\textbf{-3.80e-03}} & \epscell{\underline{4.96}} & \epscell{\underline{4.89}} \\
            & \epscell{Amortized EPS (160k steps)} & \epscell{\textbf{22.74}} & \epscell{\textbf{0.7633}} & \epscell{\textbf{0.1072}} & \epscell{\textbf{43.42}} & \epscell{\textbf{-6.42e-03}} & \epscell{\underline{-3.66e-03}} & \epscell{\textbf{4.74}} & \epscell{\textbf{4.74}} \\
            \midrule
            \multirow{4}{*}{Gaussian deblur}
            & Palette & 30.47 & 0.9397 & 0.0286 & 21.63 & -6.90e-03 & -5.65e-03 & 1.91 & 3.06 \\
            & Per-task EPS & 30.82 & 0.9408 & 0.0273 & 20.68 & \textbf{-6.91e-03} & -5.73e-03 & 1.84 & 2.99 \\
            & \epscell{Amortized EPS (55k steps)} & \epscell{\underline{31.01}} & \epscell{\underline{0.9433}} & \epscell{\underline{0.0260}} & \epscell{\underline{19.54}} & \epscell{\textbf{-6.92e-03}} & \epscell{\underline{-5.82e-03}} & \epscell{\underline{1.79}} & \epscell{\underline{2.87}} \\
            & \epscell{Amortized EPS (160k steps)} & \epscell{\textbf{31.19}} & \epscell{\textbf{0.9451}} & \epscell{\textbf{0.0250}} & \epscell{\textbf{18.57}} & \epscell{\textbf{-6.92e-03}} & \epscell{\textbf{-5.87e-03}} & \epscell{\textbf{1.76}} & \epscell{\textbf{2.80}} \\
            \midrule
            \multirow{4}{*}{Motion deblur}
            & Palette & 27.79 & 0.9099 & 0.0404 & 28.23 & -6.77e-03 & -4.98e-03 & 2.67 & 3.59 \\
            & Per-task EPS & 30.27 & 0.9339 & 0.0311 & 23.10 & -6.90e-03 & -5.39e-03 & 1.94 & 3.18 \\
            & \epscell{Amortized EPS (55k steps)} & \epscell{\underline{30.37}} & \epscell{\underline{0.9358}} & \epscell{\underline{0.0307}} & \epscell{\underline{22.21}} & \epscell{\textbf{-6.91e-03}} & \epscell{\underline{-5.59e-03}} & \epscell{\underline{1.90}} & \epscell{\underline{3.03}} \\
            & \epscell{Amortized EPS (160k steps)} & \epscell{\textbf{30.74}} & \epscell{\textbf{0.9402}} & \epscell{\textbf{0.0281}} & \epscell{\textbf{20.17}} & \epscell{\textbf{-6.91e-03}} & \epscell{\textbf{-5.71e-03}} & \epscell{\textbf{1.83}} & \epscell{\textbf{2.88}} \\
            \midrule
            \multicolumn{10}{c}{\textit{ImageNet $64{\times}64$}} \\
            \midrule
            \multirow{3}{*}{Average}
            & Palette & 24.32 & 0.7673 & 0.1124 & 82.57 & -6.40e-03 & -4.38e-03 & 4.35 & 5.54 \\
            & Per-task EPS & \textbf{24.53} & \textbf{0.7712} & \textbf{0.1103} & \textbf{81.46} & \textbf{-6.42e-03} & \textbf{-4.45e-03} & \textbf{4.27} & \textbf{5.48} \\
            & \epscell{Amortized EPS} & \epscell{\underline{24.41}} & \epscell{\underline{0.7677}} & \epscell{\underline{0.1116}} & \epscell{\underline{81.82}} & \epscell{\underline{-6.41e-03}} & \epscell{\underline{-4.39e-03}} & \epscell{\underline{4.29}} & \epscell{\underline{5.50}} \\
            \midrule
            \multirow{3}{*}{Random inpaint}
            & Palette & 24.09 & 0.7869 & 0.1011 & 81.88 & -6.50e-03 & -4.41e-03 & 4.16 & 5.52 \\
            & Per-task EPS & \textbf{24.34} & \textbf{0.7948} & \textbf{0.0979} & \textbf{79.60} & \textbf{-6.52e-03} & \textbf{-4.51e-03} & \textbf{4.04} & \textbf{5.41} \\
            & \epscell{Amortized EPS} & \epscell{\underline{24.14}} & \epscell{\underline{0.7890}} & \epscell{\underline{0.1007}} & \epscell{\underline{81.27}} & \epscell{\textbf{-6.52e-03}} & \epscell{\underline{-4.42e-03}} & \epscell{\underline{4.08}} & \epscell{\underline{5.48}} \\
            \midrule
            \multirow{3}{*}{Box inpaint}
            & Palette & \underline{21.12} & 0.7541 & 0.1218 & 92.73 & \underline{-6.10e-03} & -4.12e-03 & 5.92 & 5.93 \\
            & Per-task EPS & \textbf{21.24} & \textbf{0.7569} & \textbf{0.1196} & \underline{91.07} & \textbf{-6.11e-03} & \underline{-4.23e-03} & \textbf{5.87} & \underline{5.84} \\
            & \epscell{Amortized EPS} & \epscell{\underline{21.12}} & \epscell{\underline{0.7549}} & \epscell{\underline{0.1199}} & \epscell{\textbf{89.59}} & \epscell{\underline{-6.10e-03}} & \epscell{\textbf{-4.34e-03}} & \epscell{\underline{5.88}} & \epscell{\textbf{5.81}} \\
            \midrule
            \multirow{3}{*}{Super-res ($4{\times}$)}
            & Palette & \underline{20.24} & \underline{0.5364} & \underline{0.2220} & \textbf{128.76} & \textbf{-5.86e-03} & \underline{-2.79e-03} & \textbf{6.50} & \textbf{7.33} \\
            & Per-task EPS & \textbf{20.25} & \textbf{0.5369} & \textbf{0.2207} & \underline{128.80} & \textbf{-5.86e-03} & \textbf{-2.84e-03} & \underline{6.52} & \underline{7.35} \\
            & \epscell{Amortized EPS} & \epscell{20.15} & \epscell{0.5308} & \epscell{0.2238} & \epscell{130.55} & \epscell{-5.85e-03} & \epscell{-2.54e-03} & \epscell{6.56} & \epscell{7.38} \\
            \midrule
            \multirow{3}{*}{Gaussian deblur}
            & Palette & 29.15 & 0.9010 & 0.0491 & 46.62 & \textbf{-6.83e-03} & \underline{-5.56e-03} & 2.26 & \underline{4.11} \\
            & Per-task EPS & \underline{29.18} & \underline{0.9015} & \underline{0.0486} & \underline{46.55} & \textbf{-6.83e-03} & -5.55e-03 & \underline{2.25} & \underline{4.11} \\
            & \epscell{Amortized EPS} & \epscell{\textbf{29.19}} & \epscell{\textbf{0.9017}} & \epscell{\textbf{0.0478}} & \epscell{\textbf{45.33}} & \epscell{\textbf{-6.83e-03}} & \epscell{\textbf{-5.61e-03}} & \epscell{\textbf{2.24}} & \epscell{\textbf{4.06}} \\
            \midrule
            \multirow{3}{*}{Motion deblur}
            & Palette & 27.02 & 0.8582 & 0.0680 & 62.86 & -6.73e-03 & -5.02e-03 & 2.93 & 4.81 \\
            & Per-task EPS & \textbf{27.62} & \textbf{0.8661} & \textbf{0.0647} & \textbf{61.29} & \textbf{-6.77e-03} & \textbf{-5.11e-03} & \textbf{2.69} & \textbf{4.70} \\
            & \epscell{Amortized EPS} & \epscell{\underline{27.44}} & \epscell{\underline{0.8619}} & \epscell{\underline{0.0657}} & \epscell{\underline{62.35}} & \epscell{\underline{-6.76e-03}} & \epscell{\underline{-5.07e-03}} & \epscell{\underline{2.72}} & \epscell{\underline{4.74}} \\
            \bottomrule
        \end{tabular}%
    }
\end{table*}

\clearpage
\subsection{One-Step Posterior Mean Check}
\label{app:additional-experiments:onestep}
Section~\ref{sec:method:one-step} predicts that a single high-noise evaluation returns a posterior-mean estimator. We compare EPS at 1 NFE (a single direct Tweedie call $D_\theta(\mu_\star,\sigma_{\max})$ at $\sigma_{\max}{=}80$, no sampler loop) to the empirical mean of $J{=}10$ multi-step posterior samples drawn with the standard 100-step Euler sampler (1000 NFE per image), and to a single posterior sample from the same multi-step sampler (the EPS NFE$=$100 row used in our main results).

Table~\ref{tab:onestep} reports this comparison on FFHQ-64 (top) and ImageNet-64 (bottom) across all five tasks. The 1-NFE Tweedie call recovers most of the PSNR/SSIM gain that the multi-step empirical mean attains while using $1000\times$ fewer denoiser evaluations, confirming the high-noise posterior-mean prediction. As expected, both rows that target the conditional mean (the 1-NFE row and the empirical-mean row) score better on distortion (PSNR, SSIM) than the single-sample row, while the single-sample row scores better on perceptual and distributional metrics (LPIPS, FID, CRPS, MMD). Note that MMD/CRPS for the empirical-mean row degenerate to deterministic distances since the per-image ensemble has size $J{=}1$ after averaging.
\begin{table*}[ht!]
    \centering
    \scriptsize
    \setlength{\tabcolsep}{4pt}
    \caption{\textbf{One Tweedie call recovers most of the multi-step gain.} One-step EPS vs.\ empirical posterior mean from multi-step EPS samples on FFHQ-64 (top) and ImageNet-64 (bottom). The 1-NFE row is a single direct Tweedie call at $\sigma_{\max}{=}80$ (no sampler loop). The posterior-mean row averages 10 independent 100-step Euler samples (1000 NFE per image). The single-sample row reports per-seed metrics from the same 100-step sampler (matching the main-text NFE$=$100 row). The 1-NFE row matches or trails the empirical-mean row by a small margin on distortion metrics while using $1000\times$ fewer denoiser calls. Best in \textbf{bold}, second-best \underline{underlined}.}
    \label{tab:onestep}
    \resizebox{\textwidth}{!}{%
        \begin{tabular}{llcccccccc}
            \toprule
            Task & Method & PSNR $\uparrow$ & SSIM $\uparrow$ & LPIPS $\downarrow$ & FID $\downarrow$ & MMD-pix $\downarrow$ & MMD-Inc $\downarrow$ & CRPS-pix $\downarrow$ & CRPS-Inc $\downarrow$ \\
            \midrule
            \multicolumn{10}{c}{\textit{FFHQ $64{\times}64$}} \\
            \midrule
            \multirow{3}{*}{Average}
            & EPS, 1 NFE & \underline{27.34} & \underline{0.8887} & 0.0699 & 46.30 & -6.31e-03 & 1.13e-02 & 4.67 & 6.91 \\
            & EPS, posterior mean from samples & \textbf{29.10} & \textbf{0.9095} & \textbf{0.0480} & \underline{34.17} & \textbf{-1.22e-02} & \underline{-2.77e-03} & \underline{4.35} & \underline{6.46} \\
            & EPS, single posterior sample & 26.69 & 0.8661 & \underline{0.0590} & \textbf{29.94} & \underline{-6.69e-03} & \textbf{-4.89e-03} & \textbf{3.24} & \textbf{3.73} \\
            \midrule
            \multirow{3}{*}{Random inpaint}
            & EPS, 1 NFE & \underline{27.90} & \underline{0.9148} & 0.0549 & 39.71 & -6.63e-03 & 2.18e-03 & \underline{3.94} & \underline{5.75} \\
            & EPS, posterior mean from samples & \textbf{28.66} & \textbf{0.9293} & \textbf{0.0399} & \underline{34.17} & \textbf{-1.23e-02} & \textbf{-5.15e-03} & 4.26 & 6.48 \\
            & EPS, single posterior sample & 26.16 & 0.8879 & \underline{0.0533} & \textbf{31.87} & \underline{-6.74e-03} & \underline{-5.00e-03} & \textbf{3.16} & \textbf{3.75} \\
            \midrule
            \multirow{3}{*}{Box inpaint}
            & EPS, 1 NFE & \underline{26.01} & \underline{0.8706} & 0.0644 & 32.90 & -6.36e-03 & -8.49e-04 & \underline{5.05} & \underline{5.62} \\
            & EPS, posterior mean from samples & \textbf{26.72} & \textbf{0.8906} & \textbf{0.0464} & \underline{26.75} & \textbf{-1.20e-02} & \textbf{-8.08e-03} & 5.38 & 5.93 \\
            & EPS, single posterior sample & 24.23 & 0.8448 & \underline{0.0567} & \textbf{24.74} & \underline{-6.59e-03} & \underline{-5.35e-03} & \textbf{4.01} & \textbf{3.45} \\
            \midrule
            \multirow{3}{*}{Super-res ($4{\times}$)}
            & EPS, 1 NFE & \underline{24.21} & \underline{0.7999} & 0.1265 & 77.97 & -5.63e-03 & 4.09e-02 & \underline{6.57} & \underline{9.27} \\
            & EPS, posterior mean from samples & \textbf{24.21} & \textbf{0.8042} & \textbf{0.1087} & \underline{64.71} & \textbf{-1.14e-02} & \underline{1.35e-02} & 7.03 & 9.36 \\
            & EPS, single posterior sample & 21.96 & 0.7232 & \underline{0.1262} & \textbf{49.29} & \underline{-6.30e-03} & \textbf{-3.00e-03} & \textbf{5.23} & \textbf{5.30} \\
            \midrule
            \multirow{3}{*}{Gaussian deblur}
            & EPS, 1 NFE & \underline{30.86} & \underline{0.9503} & 0.0369 & 31.39 & -6.62e-03 & 2.35e-03 & 3.21 & 6.15 \\
            & EPS, posterior mean from samples & \textbf{33.18} & \textbf{0.9637} & \textbf{0.0211} & \underline{21.15} & \textbf{-1.26e-02} & \textbf{-7.75e-03} & \underline{2.48} & \underline{5.05} \\
            & EPS, single posterior sample & 30.82 & 0.9408 & \underline{0.0273} & \textbf{20.68} & \underline{-6.91e-03} & \underline{-5.73e-03} & \textbf{1.84} & \textbf{2.99} \\
            \midrule
            \multirow{3}{*}{Motion deblur}
            & EPS, 1 NFE & 27.73 & 0.9078 & 0.0669 & 49.55 & -6.31e-03 & 1.21e-02 & 4.60 & 7.76 \\
            & EPS, posterior mean from samples & \textbf{32.71} & \textbf{0.9597} & \textbf{0.0239} & \underline{24.06} & \textbf{-1.25e-02} & \textbf{-6.39e-03} & \underline{2.61} & \underline{5.46} \\
            & EPS, single posterior sample & \underline{30.27} & \underline{0.9339} & \underline{0.0311} & \textbf{23.10} & \underline{-6.90e-03} & \underline{-5.39e-03} & \textbf{1.94} & \textbf{3.18} \\
            \midrule
            \multicolumn{10}{c}{\textit{ImageNet $64{\times}64$}} \\
            \midrule
            \multirow{3}{*}{Average}
            & EPS, 1 NFE & \underline{25.67} & \underline{0.8165} & 0.1318 & 110.95 & -5.58e-03 & 4.49e-03 & 5.91 & 9.96 \\
            & EPS, posterior mean from samples & \textbf{26.97} & \textbf{0.8330} & \textbf{0.0999} & \underline{90.93} & \textbf{-1.16e-02} & \textbf{-5.77e-03} & \underline{5.74} & \underline{9.76} \\
            & EPS, single posterior sample & 24.53 & 0.7712 & \underline{0.1103} & \textbf{81.46} & \underline{-6.42e-03} & \underline{-4.45e-03} & \textbf{4.27} & \textbf{5.48} \\
            \midrule
            \multirow{3}{*}{Random inpaint}
            & EPS, 1 NFE & \underline{26.60} & \underline{0.8580} & \underline{0.0933} & 88.59 & -6.15e-03 & -4.84e-04 & \underline{4.98} & \underline{8.32} \\
            & EPS, posterior mean from samples & \textbf{26.81} & \textbf{0.8651} & \textbf{0.0790} & \underline{81.97} & \textbf{-1.19e-02} & \textbf{-7.41e-03} & 5.44 & 9.35 \\
            & EPS, single posterior sample & 24.34 & 0.7948 & 0.0979 & \textbf{79.60} & \underline{-6.52e-03} & \underline{-4.51e-03} & \textbf{4.04} & \textbf{5.41} \\
            \midrule
            \multirow{3}{*}{Box inpaint}
            & EPS, 1 NFE & \underline{23.60} & \underline{0.7908} & 0.1514 & 129.11 & -5.60e-03 & 7.34e-03 & \underline{7.16} & \underline{10.38} \\
            & EPS, posterior mean from samples & \textbf{23.65} & \textbf{0.7985} & \textbf{0.1149} & \underline{111.46} & \textbf{-1.10e-02} & \underline{-3.94e-03} & 7.86 & 11.12 \\
            & EPS, single posterior sample & 21.24 & 0.7569 & \underline{0.1196} & \textbf{91.07} & \underline{-6.11e-03} & \textbf{-4.23e-03} & \textbf{5.87} & \textbf{5.84} \\
            \midrule
            \multirow{3}{*}{Super-res ($4{\times}$)}
            & EPS, 1 NFE & \textbf{22.78} & \textbf{0.6530} & 0.2455 & 182.92 & -4.97e-03 & 1.98e-02 & \underline{8.06} & \underline{13.47} \\
            & EPS, posterior mean from samples & \underline{22.57} & \underline{0.6447} & \textbf{0.2078} & \underline{158.62} & \textbf{-1.03e-02} & \underline{1.40e-03} & 8.75 & 13.68 \\
            & EPS, single posterior sample & 20.25 & 0.5369 & \underline{0.2207} & \textbf{128.80} & \underline{-5.86e-03} & \textbf{-2.84e-03} & \textbf{6.52} & \textbf{7.35} \\
            \midrule
            \multirow{3}{*}{Gaussian deblur}
            & EPS, 1 NFE & 28.82 & \underline{0.9194} & 0.0606 & 56.53 & -5.14e-03 & -3.46e-03 & 4.09 & 7.73 \\
            & EPS, posterior mean from samples & \textbf{31.66} & \textbf{0.9399} & \textbf{0.0404} & \textbf{41.79} & \textbf{-1.24e-02} & \textbf{-1.02e-02} & \underline{3.03} & \underline{6.63} \\
            & EPS, single posterior sample & \underline{29.18} & 0.9015 & \underline{0.0486} & \underline{46.55} & \underline{-6.83e-03} & \underline{-5.55e-03} & \textbf{2.25} & \textbf{4.11} \\
            \midrule
            \multirow{3}{*}{Motion deblur}
            & EPS, 1 NFE & 26.56 & 0.8613 & 0.1079 & 97.59 & -6.03e-03 & -7.30e-04 & 5.25 & 9.91 \\
            & EPS, posterior mean from samples & \textbf{30.17} & \textbf{0.9167} & \textbf{0.0573} & \textbf{60.82} & \textbf{-1.23e-02} & \textbf{-8.67e-03} & \underline{3.63} & \underline{8.03} \\
            & EPS, single posterior sample & \underline{27.62} & \underline{0.8661} & \underline{0.0647} & \underline{61.29} & \underline{-6.77e-03} & \underline{-5.11e-03} & \textbf{2.69} & \textbf{4.70} \\
            \bottomrule
        \end{tabular}%
    }
\end{table*}

\clearpage
\subsection{Additional $64{\times}64$ Results}
\label{app:additional-experiments:64-extra}
For completeness, Table~\ref{tab:ffhq} reproduces the FFHQ-64 main-table comparison from the body of the paper, broken out by task with all metrics. EPS at NFE$=$20 is the strongest configuration on perceptual and distributional metrics across all five tasks, while the NFE$=$1 Tweedie variant trades distributional fidelity for distortion (PSNR, SSIM), as predicted by the high-noise posterior-mean limit of Section~\ref{sec:method:one-step}.
\begin{table*}[ht!]
    \centering
    \scriptsize
    \setlength{\tabcolsep}{4pt}
    \caption{\textbf{Detailed FFHQ-64 results.} Quantitative comparison across the five inverse problems on FFHQ $64{\times}64$. All methods use a 1-NFE-per-step Euler sampler; reported NFE equals the number of sampler iterations. Best in \textbf{bold}, second-best \underline{underlined}; EPS rows highlighted in light pink. \dag\ The NFE$=$1 row evaluates the deterministic high-noise posterior-mean limit (one direct Tweedie call $D_\theta(\mu_\star,\sigma_{\max})$); MMSE-optimal in pixel space but does not produce posterior samples, hence its strong PSNR/SSIM but weaker distributional metrics.}
    \label{tab:ffhq}
    \resizebox{\textwidth}{!}{%
        \begin{tabular}{llccccccccc}
            \toprule
            Task & Method & NFE & PSNR $\uparrow$ & SSIM $\uparrow$ & LPIPS $\downarrow$ & FID $\downarrow$ & MMD-pix $\downarrow$ & MMD-Inc $\downarrow$ & CRPS-pix $\downarrow$ & CRPS-Inc $\downarrow$ \\
            \midrule
            \multirow{9}{*}{Random inpaint}
            & DPS & 250 & 23.09 & 0.8003 & 0.1080 & 70.27 & -6.04e-03 & 3.00e-02 & 4.56 & 6.33 \\
            & DAPS & 100 & 22.45 & 0.7576 & 0.1519 & 83.31 & -4.74e-03 & 4.95e-02 & 5.47 & 7.52 \\
            & DDNM & 100 & 24.50 & 0.8566 & 0.0766 & 45.40 & -6.44e-03 & 4.88e-04 & 4.03 & 4.87 \\
            & $\Pi$GDM & 100 & 26.00 & 0.8898 & 0.0556 & 35.17 & -6.69e-03 & -3.83e-03 & 3.44 & 4.12 \\
            & MPGD & 100 & 20.60 & 0.7097 & 0.1755 & 102.37 & -1.67e-03 & 7.62e-02 & 6.29 & 8.05 \\
            & Palette & 100 & 25.76 & 0.8809 & 0.0593 & 33.30 & -6.71e-03 & \underline{-4.76e-03} & 3.31 & \underline{3.87} \\
            & \epscell{EPS (ours)} & \epscell{100} & \epscell{26.16} & \epscell{0.8879} & \epscell{\underline{0.0533}} & \epscell{\underline{31.87}} & \epscell{\textbf{-6.74e-03}} & \epscell{\textbf{-5.00e-03}} & \epscell{\underline{3.16}} & \epscell{\textbf{3.75}} \\
            & \epscell{EPS (ours)} & \epscell{20} & \epscell{\underline{26.75}} & \epscell{\underline{0.9006}} & \epscell{\textbf{0.0489}} & \epscell{\textbf{31.56}} & \epscell{\textbf{-6.74e-03}} & \epscell{-3.94e-03} & \epscell{\textbf{3.15}} & \epscell{3.89} \\
            & \epscell{EPS (ours)\dag} & \epscell{1} & \epscell{\textbf{27.90}} & \epscell{\textbf{0.9148}} & \epscell{0.0549} & \epscell{39.71} & \epscell{-6.63e-03} & \epscell{2.18e-03} & \epscell{3.94} & \epscell{5.75} \\
            \midrule
            \multirow{9}{*}{Box inpaint}
            & DPS & 250 & 22.33 & 0.7692 & 0.0946 & 41.58 & -6.30e-03 & 5.29e-03 & 5.09 & 5.08 \\
            & DAPS & 100 & 23.06 & 0.7968 & 0.0900 & 39.35 & -6.01e-03 & 6.47e-03 & 5.09 & 5.00 \\
            & DDNM & 100 & 23.15 & 0.8231 & 0.0737 & 28.75 & -6.20e-03 & -4.28e-03 & 4.82 & 4.03 \\
            & $\Pi$GDM & 100 & 23.80 & 0.8359 & 0.0623 & 26.81 & -6.51e-03 & -4.33e-03 & 4.36 & 3.78 \\
            & MPGD & 100 & 20.39 & 0.7211 & 0.1229 & 49.51 & -4.55e-03 & 9.16e-03 & 6.98 & 6.00 \\
            & Palette & 100 & 24.18 & 0.8426 & 0.0577 & 25.09 & -6.58e-03 & \underline{-5.30e-03} & 4.02 & \underline{3.47} \\
            & \epscell{EPS (ours)} & \epscell{100} & \epscell{24.23} & \epscell{0.8448} & \epscell{\underline{0.0567}} & \epscell{\underline{24.74}} & \epscell{\textbf{-6.59e-03}} & \epscell{\textbf{-5.35e-03}} & \epscell{\underline{4.01}} & \epscell{\textbf{3.45}} \\
            & \epscell{EPS (ours)} & \epscell{20} & \epscell{\underline{24.70}} & \epscell{\underline{0.8560}} & \epscell{\textbf{0.0536}} & \epscell{\textbf{24.45}} & \epscell{\textbf{-6.59e-03}} & \epscell{-4.77e-03} & \epscell{\textbf{3.99}} & \epscell{3.55} \\
            & \epscell{EPS (ours)\dag} & \epscell{1} & \epscell{\textbf{26.01}} & \epscell{\textbf{0.8706}} & \epscell{0.0644} & \epscell{32.90} & \epscell{-6.36e-03} & \epscell{-8.49e-04} & \epscell{5.05} & \epscell{5.62} \\
            \midrule
            \multirow{9}{*}{Super-res ($4{\times}$)}
            & DPS & 250 & 20.43 & 0.5968 & 0.2162 & 90.79 & -5.71e-03 & 4.85e-02 & 6.62 & 7.69 \\
            & DAPS & 100 & 18.87 & 0.5112 & 0.2658 & 110.31 & -1.81e-03 & 1.15e-01 & 8.28 & 9.35 \\
            & DDNM & 100 & 22.35 & 0.7201 & 0.1592 & 62.67 & -6.06e-03 & 1.90e-02 & 5.79 & 6.63 \\
            & $\Pi$GDM & 100 & 22.02 & 0.7242 & 0.1281 & 50.92 & -6.28e-03 & -1.23e-03 & 5.39 & 5.48 \\
            & MPGD & 100 & 20.49 & 0.5809 & 0.2496 & 120.27 & -4.69e-03 & 1.05e-01 & 7.59 & 9.29 \\
            & Palette & 100 & 21.95 & 0.7220 & 0.1273 & \textbf{49.28} & \underline{-6.29e-03} & \underline{-2.98e-03} & \textbf{5.22} & \textbf{5.29} \\
            & \epscell{EPS (ours)} & \epscell{100} & \epscell{21.96} & \epscell{0.7232} & \epscell{\underline{0.1262}} & \epscell{\underline{49.29}} & \epscell{\textbf{-6.30e-03}} & \epscell{\textbf{-3.00e-03}} & \epscell{\underline{5.23}} & \epscell{\underline{5.30}} \\
            & \epscell{EPS (ours)} & \epscell{20} & \epscell{\underline{22.56}} & \epscell{\underline{0.7480}} & \epscell{\textbf{0.1188}} & \epscell{50.11} & \epscell{-6.24e-03} & \epscell{1.01e-03} & \epscell{5.24} & \epscell{5.50} \\
            & \epscell{EPS (ours)\dag} & \epscell{1} & \epscell{\textbf{24.21}} & \epscell{\textbf{0.7999}} & \epscell{0.1265} & \epscell{77.97} & \epscell{-5.63e-03} & \epscell{4.09e-02} & \epscell{6.57} & \epscell{9.27} \\
            \midrule
            \multirow{9}{*}{Gaussian deblur}
            & DPS & 250 & 25.89 & 0.8530 & 0.0901 & 68.78 & -5.78e-03 & 2.85e-02 & 3.75 & 6.40 \\
            & DAPS & 100 & 27.32 & 0.8494 & 0.0771 & 68.96 & -6.80e-03 & 5.00e-02 & 2.75 & 6.89 \\
            & DDNM & 100 & \underline{31.27} & 0.9454 & \underline{0.0273} & 20.96 & \textbf{-6.91e-03} & -5.45e-03 & 1.87 & 3.15 \\
            & $\Pi$GDM & 100 & 30.18 & 0.9332 & 0.0339 & 27.30 & -6.85e-03 & -8.48e-04 & 2.26 & 3.86 \\
            & MPGD & 100 & 19.88 & 0.6716 & 0.1770 & 83.31 & 2.72e-02 & 5.93e-02 & 7.59 & 7.84 \\
            & Palette & 100 & 30.47 & 0.9397 & 0.0286 & 21.63 & -6.90e-03 & \underline{-5.65e-03} & 1.91 & \underline{3.06} \\
            & \epscell{EPS (ours)} & \epscell{100} & \epscell{30.82} & \epscell{0.9408} & \epscell{\underline{0.0273}} & \epscell{\underline{20.68}} & \epscell{\textbf{-6.91e-03}} & \epscell{\textbf{-5.73e-03}} & \epscell{\textbf{1.84}} & \epscell{\textbf{2.99}} \\
            & \epscell{EPS (ours)} & \epscell{20} & \epscell{\textbf{31.47}} & \epscell{\underline{0.9488}} & \epscell{\textbf{0.0249}} & \epscell{\textbf{20.36}} & \epscell{\textbf{-6.91e-03}} & \epscell{-4.97e-03} & \epscell{\textbf{1.84}} & \epscell{3.13} \\
            & \epscell{EPS (ours)\dag} & \epscell{1} & \epscell{30.86} & \epscell{\textbf{0.9503}} & \epscell{0.0369} & \epscell{31.39} & \epscell{-6.62e-03} & \epscell{2.35e-03} & \epscell{3.21} & \epscell{6.15} \\
            \midrule
            \multirow{9}{*}{Motion deblur}
            & DPS & 250 & 27.85 & 0.8694 & 0.0752 & 70.24 & -6.76e-03 & 4.68e-02 & 3.29 & 7.20 \\
            & DAPS & 100 & 27.18 & 0.8703 & 0.0732 & 53.53 & -6.70e-03 & 1.85e-02 & 3.12 & 5.92 \\
            & DDNM & 100 & \textbf{30.96} & \underline{0.9411} & \underline{0.0297} & \textbf{22.06} & \textbf{-6.91e-03} & \underline{-5.32e-03} & \textbf{1.93} & \underline{3.23} \\
            & $\Pi$GDM & 100 & 25.04 & 0.8529 & 0.0940 & 54.68 & -6.23e-03 & 1.24e-02 & 4.10 & 5.49 \\
            & MPGD & 100 & 27.53 & 0.8764 & 0.0746 & 77.02 & -6.63e-03 & 5.17e-02 & 3.62 & 7.63 \\
            & Palette & 100 & 27.79 & 0.9099 & 0.0404 & 28.23 & -6.77e-03 & -4.98e-03 & 2.67 & 3.59 \\
            & \epscell{EPS (ours)} & \epscell{100} & \epscell{30.27} & \epscell{0.9339} & \epscell{0.0311} & \epscell{23.10} & \epscell{-6.90e-03} & \epscell{\textbf{-5.39e-03}} & \epscell{\underline{1.94}} & \epscell{\textbf{3.18}} \\
            & \epscell{EPS (ours)} & \epscell{20} & \epscell{\underline{30.88}} & \epscell{\textbf{0.9422}} & \epscell{\textbf{0.0285}} & \epscell{\underline{22.79}} & \epscell{\textbf{-6.91e-03}} & \epscell{-4.44e-03} & \epscell{\underline{1.94}} & \epscell{3.32} \\
            & \epscell{EPS (ours)\dag} & \epscell{1} & \epscell{27.73} & \epscell{0.9078} & \epscell{0.0669} & \epscell{49.55} & \epscell{-6.31e-03} & \epscell{1.21e-02} & \epscell{4.60} & \epscell{7.76} \\
            \bottomrule
        \end{tabular}%
    }
\end{table*}

Figures~\ref{fig:qual:ffhq} and~\ref{fig:qual:imagenet} show qualitative reconstructions on FFHQ-64 and ImageNet-64 across the five inverse problems, comparing EPS against DPS, DAPS, DDNM, $\Pi$GDM, MPGD, and Palette. Two example observations per task; numbers in the bottom-right corner of each panel are per-image PSNR. On FFHQ-64, EPS recovers facial structure (eyes, mouth, hairline) under aggressive random inpainting and box inpainting, and produces sharper texture and edge geometry on super-resolution and deblurring than the sampling-based and Palette baselines. On ImageNet-64 the same pattern holds on a more diverse class-conditional distribution, with EPS preserving operator-consistent structure where DPS, DAPS, and MPGD oversmooth or hallucinate texture inconsistent with the measurement.

\begin{figure*}[ht!]
\centering
\includegraphics[width=\textwidth]{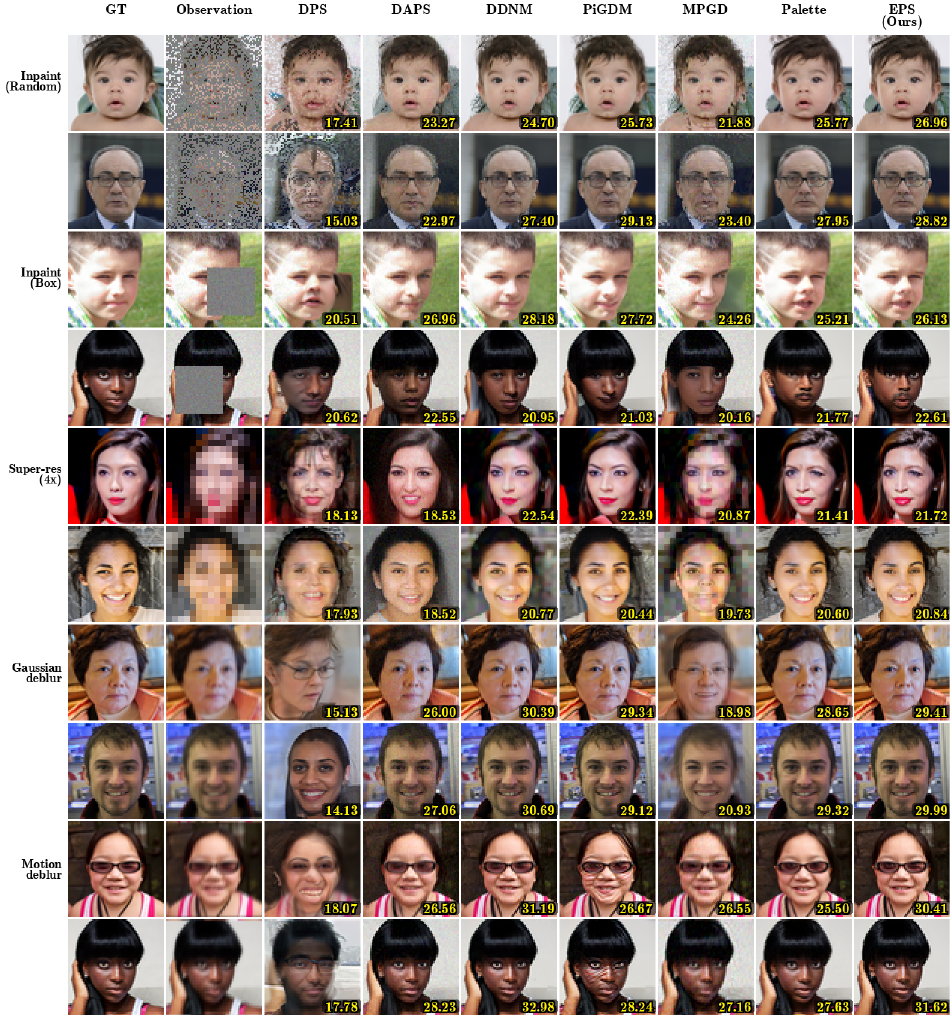}
\caption{\textbf{Qualitative reconstructions on FFHQ-64.} Two example observations per task across the five inverse problems; numbers in the bottom-right corner of each panel are per-image PSNR. EPS recovers facial structure under aggressive random inpainting and box inpainting, and produces sharper texture and edge geometry on super-resolution and deblurring than the sampling-based and Palette baselines.}
\label{fig:qual:ffhq}
\end{figure*}

\begin{figure*}[ht!]
\centering
\includegraphics[width=\textwidth]{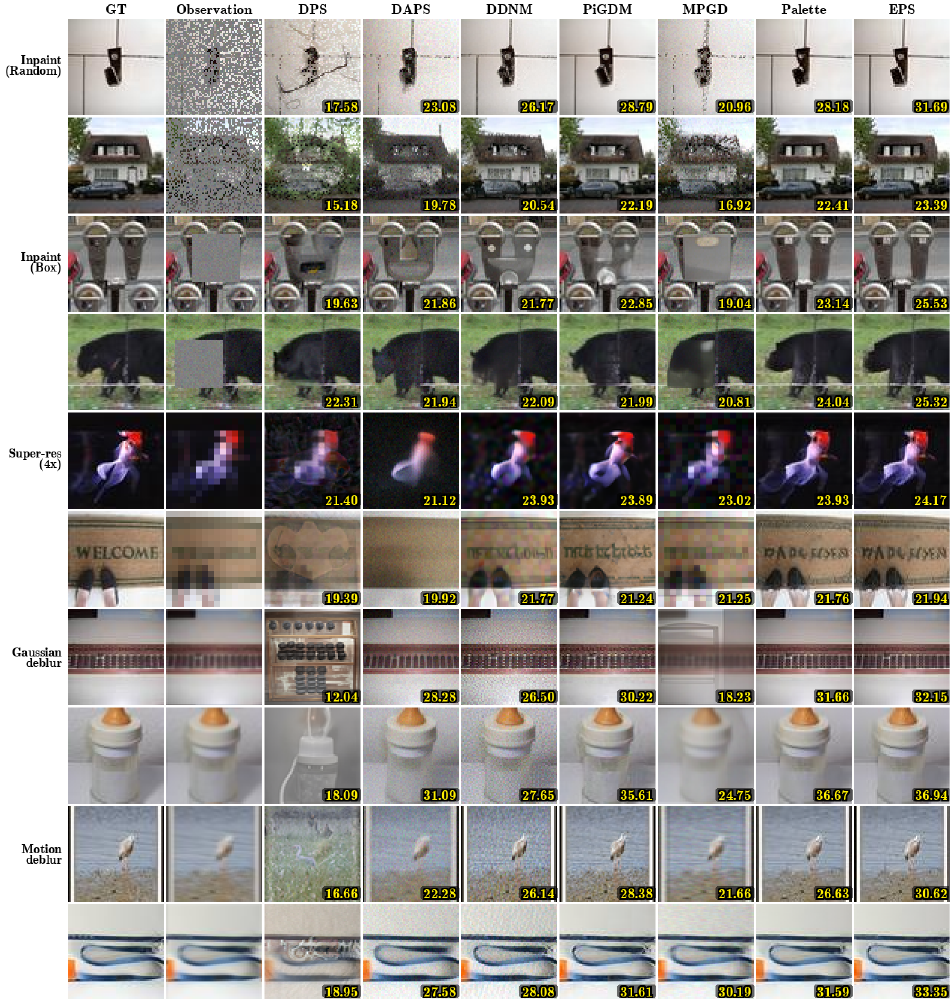}
\caption{\textbf{Qualitative reconstructions on ImageNet-64.} Same layout as Fig.~\ref{fig:qual:ffhq}. EPS preserves operator-consistent structure where DPS, DAPS, and MPGD oversmooth or hallucinate texture inconsistent with the measurement.}
\label{fig:qual:imagenet}
\end{figure*}

\clearpage
\subsection{Extreme Tasks $64{\times}64$}
\label{app:additional-experiments:extreme}
We test EPS on two extreme regimes that fall outside the main-text protocol: random inpainting with 95\% of pixels missing (only 5\% observed) and $16{\times}$ super-resolution (a $4{\times}4$ low-resolution observation upsampled to $64{\times}64$). Both push the operator nullspace to occupy almost the entire signal space, so the prior must do most of the reconstruction work and the measurement-matching score is correspondingly noisier.

Table~\ref{tab:extreme} reports this comparison on ImageNet-64 against DPS, DAPS, DDNM, $\Pi$GDM, and MPGD; Palette is omitted because no Palette checkpoint was trained for these regimes. EPS at NFE$=$20 is the strongest method on perceptual and distributional metrics across both tasks (LPIPS, FID, MMD, CRPS), and is competitive on PSNR/SSIM with the strongest sampling-based baselines despite their having access to a full sampler trajectory. On 95\% inpainting, EPS-20 reduces FID by roughly 25\% over the best sampling-based baseline ($\Pi$GDM at 195) and roughly 30\% over DPS. On $16{\times}$ super-resolution, the perceptual gap is smaller because the operator preserves only a single anchor pixel per $4{\times}4$ block, leaving little measurement signal in the pivot to exploit.

\begin{table*}[ht!]
    \centering
    \scriptsize
    \setlength{\tabcolsep}{4pt}
    \caption{\textbf{EPS holds up under extreme operator nullspace.} Quantitative comparison on ImageNet-64 in two extreme inverse-problem regimes: 95\% random inpainting (only 5\% of pixels observed) and $16{\times}$ super-resolution (a $4{\times}4$ low-resolution observation upsampled to $64{\times}64$). All baselines use a 1-NFE-per-step Euler/DDIM sampler; EPS uses the EDM Euler sampler. Reported NFE equals the number of sampler iterations. Best in \textbf{bold}, second-best \underline{underlined}; EPS rows highlighted in light pink. Palette is omitted because no Palette checkpoint was trained for these regimes.}
    \label{tab:extreme}
    \resizebox{\textwidth}{!}{%
        \begin{tabular}{llccccccccc}
            \toprule
            Task & Method & NFE & PSNR $\uparrow$ & SSIM $\uparrow$ & LPIPS $\downarrow$ & FID $\downarrow$ & MMD-pix $\downarrow$ & MMD-Inc $\downarrow$ & CRPS-pix $\downarrow$ & CRPS-Inc $\downarrow$ \\
            \midrule
            \multirow{7}{*}{Inpaint (95\% masked)}
            & DPS & 250 & 14.14 & 0.2101 & 0.5066 & 210.05 & 1.63e-02 & 2.48e-02 & 12.78 & 10.68 \\
            & DAPS & 100 & 14.20 & 0.1753 & 0.5962 & 329.19 & 1.08e-01 & 2.14e-01 & 15.96 & 14.77 \\
            & DDNM & 100 & 14.47 & 0.2316 & 0.5480 & 251.66 & 6.91e-02 & 7.34e-02 & 14.37 & 11.91 \\
            & $\Pi$GDM & 100 & 16.09 & 0.3015 & 0.4481 & 195.18 & 9.08e-03 & 1.50e-02 & 10.96 & 10.10 \\
            & MPGD & 100 & 13.87 & 0.2073 & 0.5632 & 267.41 & 7.01e-02 & 9.36e-02 & 15.15 & 12.34 \\
            & \epscell{EPS (ours)} & \epscell{100} & \epscell{\underline{17.85}} & \epscell{\underline{0.4468}} & \epscell{\underline{0.2813}} & \epscell{\underline{145.17}} & \epscell{\textbf{-5.27e-03}} & \epscell{\textbf{-2.11e-03}} & \epscell{\underline{8.38}} & \epscell{\textbf{7.93}} \\
            & \epscell{EPS (ours)} & \epscell{20} & \epscell{\textbf{18.41}} & \epscell{\textbf{0.4755}} & \epscell{\textbf{0.2712}} & \epscell{\textbf{144.18}} & \epscell{\underline{-5.08e-03}} & \epscell{\underline{-1.68e-03}} & \epscell{\textbf{8.34}} & \epscell{\underline{8.00}} \\
            \midrule
            \multirow{7}{*}{Super-res ($16{\times}$)}
            & DPS & 250 & 13.30 & 0.1171 & 0.5550 & 191.98 & 2.63e-03 & 7.74e-03 & 13.88 & 10.23 \\
            & DAPS & 100 & 12.70 & 0.1437 & 0.6439 & 267.84 & 7.64e-02 & 9.77e-02 & 17.51 & 12.47 \\
            & DDNM & 100 & \textbf{15.84} & \textbf{0.2295} & 0.5787 & 224.62 & 7.58e-03 & 5.73e-02 & 13.30 & 11.62 \\
            & $\Pi$GDM & 100 & 14.43 & 0.1703 & 0.5073 & 181.06 & \underline{-2.85e-03} & 8.31e-03 & 12.71 & \underline{9.72} \\
            & MPGD & 100 & 13.84 & \underline{0.1875} & 0.6176 & 233.46 & 8.89e-02 & 5.67e-02 & 15.85 & 11.61 \\
            & \epscell{EPS (ours)} & \epscell{100} & \epscell{14.17} & \epscell{0.1585} & \epscell{\underline{0.5007}} & \epscell{\textbf{175.76}} & \epscell{\textbf{-3.10e-03}} & \epscell{\textbf{3.34e-04}} & \epscell{\underline{12.71}} & \epscell{\textbf{9.60}} \\
            & \epscell{EPS (ours)} & \epscell{20} & \epscell{\underline{14.86}} & \epscell{0.1822} & \epscell{\textbf{0.4948}} & \epscell{\underline{180.14}} & \epscell{-6.49e-04} & \epscell{\underline{3.63e-03}} & \epscell{\textbf{12.67}} & \epscell{9.81} \\
            \bottomrule
        \end{tabular}%
    }
\end{table*}

Figure~\ref{fig:extreme-qualitative} shows EPS reconstructions across six independent latent seeds on three 95\%-inpainting and three $16{\times}$ super-resolution observations. Samples agree on the broad spatial layout dictated by the few observed pixels but diverge sharply in fine structure and unobserved content (foreground identity, background texture, occluded geometry), which is the qualitative signature of a calibrated posterior in a regime where the operator nullspace dominates.

\begin{figure*}[ht!]
\centering
\includegraphics[width=\textwidth]{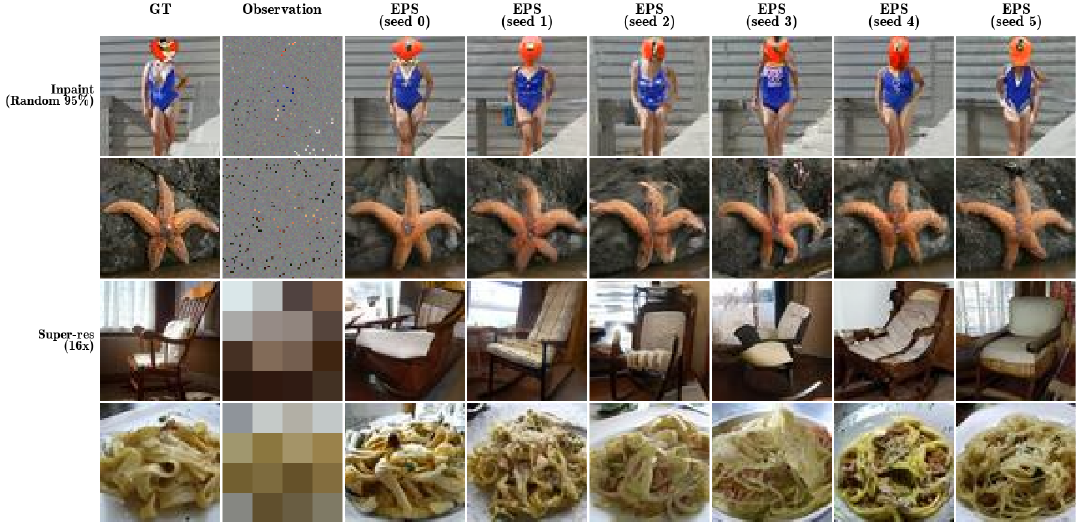}
\caption{\textbf{Posterior diversity under extreme operator nullspace.} EPS reconstructions on ImageNet-64 across six independent latent seeds. Top three rows: 95\% random inpainting (only 5\% of pixels observed). Bottom three rows: $16{\times}$ super-resolution (a $4{\times}4$ low-resolution observation upsampled to $64{\times}64$). Samples agree on the broad spatial layout consistent with the measurement but vary substantially in unobserved directions, illustrating that EPS produces genuinely distinct posterior samples rather than near-duplicates of a single conditional mean.}
\label{fig:extreme-qualitative}
\end{figure*}

\clearpage
\subsection{OOD Mask-Density Experiments}
\label{app:additional-experiments:ood}
We test how EPS and Palette generalize when the test-time mask density differs from training. Both checkpoints are trained on random inpainting at 70\% masking and frozen; at evaluation we re-sample masks at five densities, ranging from 50\% (easier than training) through 70\% (in-distribution) to 90\% (much harder than training). All results use NFE$=$100 with the EDM Euler sampler over 100 images $\times$ 10 seeds, with $\sigma_y{=}0.05$.

Table~\ref{tab:ood-mask} reports this comparison on ImageNet-64 (top) and FFHQ-64 (bottom). On ImageNet, EPS wins on every metric at every density except 80\% (near-tie). On FFHQ, EPS dominates in- and near-distribution (50\%-70\%) but degrades faster than Palette on PSNR/SSIM at heavily-OOD densities (80\%-90\%); EPS still wins on the distributional metrics (MMD-pix, CRPS) at every density. The PSNR/SSIM crossover at high mask fraction is consistent with the closed-form pivot $\mu_\star$ extrapolating poorly when the operator shifts far from training at test time, since the precision weighting of the pivot is calibrated to a 70\% mask under $\sigma_y{=}0.05$ and is mismatched when the actual mask density changes.
\begin{table*}[ht!]
    \centering
    \scriptsize
    \setlength{\tabcolsep}{4pt}
    \caption{\textbf{OOD generalization across mask density.} Both checkpoints are trained on random inpainting at 70\% masking and frozen; evaluation re-samples masks at five densities (50\%, 60\%, 70\% in-distribution, 80\%, 90\%) on the same eval $x_0$ at $\sigma_y{=}0.05$. ImageNet-64 (top): EPS wins on every metric at every density except 80\% (near-tie). FFHQ-64 (bottom): EPS dominates in- and near-distribution but degrades faster than Palette on PSNR/SSIM at heavily-OOD densities, while still winning on MMD and CRPS at every density. All numbers at NFE$=$100 (EDM Euler). Best per row in \textbf{bold}; EPS rows highlighted in light pink.}
    \label{tab:ood-mask}
    \resizebox{\textwidth}{!}{%
        \begin{tabular}{llcccccccc}
            \toprule
            Mask \% & Method & PSNR $\uparrow$ & SSIM $\uparrow$ & LPIPS $\downarrow$ & FID $\downarrow$ & MMD-pix $\downarrow$ & MMD-Inc $\downarrow$ & CRPS-pix $\downarrow$ & CRPS-Inc $\downarrow$ \\
            \midrule
            \multicolumn{10}{c}{\textit{ImageNet $64{\times}64$}} \\
            \midrule
            \multirow{2}{*}{50\%}
            & Palette & 27.10 & 0.8820 & 0.0608 & 55.12 & -6.71e-03 & -5.26e-03 & 2.87 & 4.41 \\
            & \epscell{EPS (ours)} & \epscell{\textbf{27.88}} & \epscell{\textbf{0.8969}} & \epscell{\textbf{0.0545}} & \epscell{\textbf{51.79}} & \epscell{\textbf{-6.76e-03}} & \epscell{\textbf{-5.32e-03}} & \epscell{\textbf{2.69}} & \epscell{\textbf{4.31}} \\
            \midrule
            \multirow{2}{*}{60\%}
            & Palette & 25.62 & 0.8415 & 0.0778 & 65.85 & -6.63e-03 & -5.01e-03 & 3.44 & 4.83 \\
            & \epscell{EPS (ours)} & \epscell{\textbf{26.07}} & \epscell{\textbf{0.8526}} & \epscell{\textbf{0.0729}} & \epscell{\textbf{62.98}} & \epscell{\textbf{-6.66e-03}} & \epscell{\textbf{-5.10e-03}} & \epscell{\textbf{3.31}} & \epscell{\textbf{4.73}} \\
            \midrule
            \multirow{2}{*}{70\% (in-dist.)}
            & Palette & 24.09 & 0.7869 & 0.1011 & 81.88 & -6.50e-03 & -4.41e-03 & 4.16 & 5.52 \\
            & \epscell{EPS (ours)} & \epscell{\textbf{24.34}} & \epscell{\textbf{0.7948}} & \epscell{\textbf{0.0979}} & \epscell{\textbf{79.60}} & \epscell{\textbf{-6.52e-03}} & \epscell{\textbf{-4.51e-03}} & \epscell{\textbf{4.04}} & \epscell{\textbf{5.41}} \\
            \midrule
            \multirow{2}{*}{80\%}
            & Palette & 22.29 & \textbf{0.7084} & \textbf{0.1399} & \textbf{98.49} & -6.27e-03 & \textbf{-3.69e-03} & 5.14 & 6.06 \\
            & \epscell{EPS (ours)} & \epscell{\textbf{22.30}} & \epscell{0.7080} & \epscell{0.1403} & \epscell{98.91} & \epscell{\textbf{-6.28e-03}} & \epscell{-3.61e-03} & \epscell{\textbf{5.08}} & \epscell{\textbf{6.04}} \\
            \midrule
            \multirow{2}{*}{90\%}
            & Palette & 18.60 & 0.5135 & 0.2635 & 152.54 & -3.54e-03 & 3.38e-03 & 8.04 & 8.33 \\
            & \epscell{EPS (ours)} & \epscell{\textbf{19.06}} & \epscell{\textbf{0.5258}} & \epscell{\textbf{0.2456}} & \epscell{\textbf{144.18}} & \epscell{\textbf{-5.01e-03}} & \epscell{\textbf{6.27e-04}} & \epscell{\textbf{7.19}} & \epscell{\textbf{7.78}} \\
            \midrule
            \multicolumn{10}{c}{\textit{FFHQ $64{\times}64$}} \\
            \midrule
            \multirow{2}{*}{50\%}
            & Palette & 27.95 & 0.9171 & 0.0414 & 25.07 & -6.47e-03 & \textbf{-5.32e-03} & 2.50 & \textbf{3.32} \\
            & \epscell{EPS (ours)} & \epscell{\textbf{29.64}} & \epscell{\textbf{0.9409}} & \epscell{\textbf{0.0307}} & \epscell{\textbf{22.92}} & \epscell{\textbf{-6.86e-03}} & \epscell{-3.85e-03} & \epscell{\textbf{2.19}} & \epscell{3.34} \\
            \midrule
            \multirow{2}{*}{60\%}
            & Palette & 27.11 & 0.9061 & 0.0470 & 28.16 & -6.70e-03 & \textbf{-5.11e-03} & 2.77 & 3.55 \\
            & \epscell{EPS (ours)} & \epscell{\textbf{28.05}} & \epscell{\textbf{0.9206}} & \epscell{\textbf{0.0393}} & \epscell{\textbf{26.19}} & \epscell{\textbf{-6.82e-03}} & \epscell{-4.48e-03} & \epscell{\textbf{2.58}} & \epscell{\textbf{3.49}} \\
            \midrule
            \multirow{2}{*}{70\% (in-dist.)}
            & Palette & 25.76 & 0.8809 & 0.0593 & 33.30 & -6.71e-03 & -4.76e-03 & 3.31 & 3.87 \\
            & \epscell{EPS (ours)} & \epscell{\textbf{26.16}} & \epscell{\textbf{0.8879}} & \epscell{\textbf{0.0533}} & \epscell{\textbf{31.87}} & \epscell{\textbf{-6.74e-03}} & \epscell{\textbf{-5.00e-03}} & \epscell{\textbf{3.16}} & \epscell{\textbf{3.75}} \\
            \midrule
            \multirow{2}{*}{80\%}
            & Palette & \textbf{23.78} & \textbf{0.8321} & \textbf{0.0826} & \textbf{40.23} & -6.42e-03 & -3.44e-03 & 4.28 & 4.51 \\
            & \epscell{EPS (ours)} & \epscell{23.57} & \epscell{0.8250} & \epscell{0.0838} & \epscell{40.83} & \epscell{\textbf{-6.53e-03}} & \epscell{\textbf{-3.58e-03}} & \epscell{\textbf{4.19}} & \epscell{\textbf{4.45}} \\
            \midrule
            \multirow{2}{*}{90\%}
            & Palette & \textbf{18.34} & \textbf{0.6023} & 0.2286 & \textbf{83.35} & 4.20e-03 & \textbf{2.17e-02} & 8.68 & 7.09 \\
            & \epscell{EPS (ours)} & \epscell{17.82} & \epscell{0.5683} & \epscell{\textbf{0.2284}} & \epscell{90.78} & \epscell{\textbf{-1.54e-03}} & \epscell{3.45e-02} & \epscell{\textbf{7.71}} & \epscell{\textbf{7.08}} \\
            \bottomrule
        \end{tabular}%
    }
\end{table*}

Figure~\ref{fig:ood-qualitative} shows qualitative reconstructions from EPS and Palette across the five mask densities on both datasets. The visual gap between the two methods is largest in the in-distribution regime and narrows at the extremes: at 50\% masking both methods recover most of the image structure, while at 90\% both methods struggle and the reconstructions diverge sharply from the ground truth.

\begin{figure*}[ht!]
    \centering
    \includegraphics[width=\textwidth]{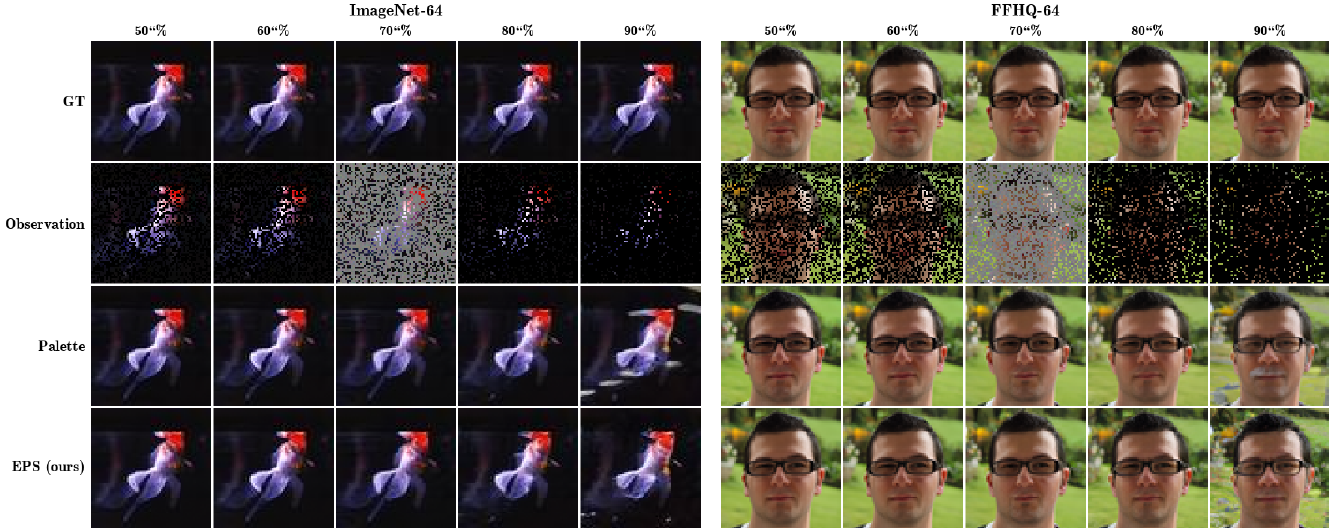}
    \caption{\textbf{Qualitative OOD generalization across mask density.} EPS and Palette reconstructions on ImageNet-64 and FFHQ-64 across five mask densities (50\%, 60\%, 70\% in-distribution, 80\%, 90\%). Both checkpoints were trained on 70\% masking and frozen at evaluation. EPS preserves operator-consistent structure best in- and near-distribution; degradation at heavily-OOD densities (80\%, 90\%) is consistent with the pivot $\mu_\star$ being calibrated to a 70\% mask.}
    \label{fig:ood-qualitative}
\end{figure*}

\clearpage
\subsection{Additional $256{\times}256$ Results}
\label{app:additional-experiments:256}
Table~\ref{tab:imagenet256} reports EPS at NFE$\in\{1, 20, 100\}$ against sampling-based and training-based baselines on ImageNet $256{\times}256$ across five inverse problems. EDM-DDPM++ architecture is used, and trained from scratch on each task (no pretrained backbone). Baseline numbers are taken directly from the DAPS paper~\cite{zhang2025improving}; we did not re-run them under our protocol, and distributional metrics (CRPS, MMD) are not reported because they are not available in the source paper. EPS leads on both inpainting tasks and on Gaussian deblurring, particularly on perceptual metrics (LPIPS, FID), and is competitive with DAPS on motion deblurring (taking second place on SSIM, LPIPS, and FID at NFE$=$20). DAPS retains an edge on $4{\times}$ super-resolution. As at $64{\times}64$, the 1-NFE EPS row is consistently the strongest EPS variant on PSNR/SSIM, mirroring the high-noise posterior-mean check of Section~\ref{sec:method:one-step}.
\begin{table*}[ht!]
    \centering
    \scriptsize
    \setlength{\tabcolsep}{3pt}
    \caption{\textbf{ImageNet $256{\times}256$ results.} EPS vs.\ sampling-based and training-based baselines on five linear inverse problems at $256{\times}256$ resolution. Baseline numbers are taken from the DAPS paper~\citep{zhang2025improving}. Best in \textbf{bold}, second-best \underline{underlined}; EPS rows highlighted in light pink. \dag\ The NFE$=$1 row applies a single direct Tweedie evaluation $D_\theta(\mu_\star,\sigma_{\max})$, returning the conditional posterior mean rather than a posterior sample.}
    \label{tab:imagenet256}
    \begin{minipage}[t]{0.5\textwidth}
    \centering
    \resizebox{\linewidth}{!}{%
        \begin{tabular}{llcccc}
            \toprule
            Task & Method & PSNR $\uparrow$ & SSIM $\uparrow$ & LPIPS $\downarrow$ & FID $\downarrow$ \\
            \midrule
            \multirow{10}{*}{Super-res ($4{\times}$)}
            & DAPS & \textbf{25.89} & \underline{0.694} & \underline{0.276} & \textbf{83.57} \\
            & DPS & 21.13 & 0.489 & 0.361 & 106.32 \\
            & DDRM & 22.62 & 0.521 & 0.324 & 103.85 \\
            & DDNM & 23.96 & 0.604 & 0.475 & 98.62 \\
            & DCDP & -- & -- & -- & -- \\
            & FPS-SMC & \underline{24.82} & \textbf{0.703} & 0.313 & \underline{97.51} \\
            & DiffPIR & 23.18 & -- & 0.371 & 106.32 \\
            & \epscell{EPS, NFE=100} & \epscell{19.60} & \epscell{0.520} & \epscell{0.294} & \epscell{138.55} \\
            & \epscell{EPS, NFE=20} & \epscell{20.32} & \epscell{0.561} & \epscell{0.277} & \epscell{130.66} \\
            & \epscell{EPS, NFE=1\dag} & \epscell{22.65} & \epscell{0.671} & \epscell{\textbf{0.254}} & \epscell{116.82} \\
            \midrule
            \multirow{10}{*}{Gaussian deblur}
            & DAPS & \underline{26.15} & 0.684 & 0.253 & 75.68 \\
            & DPS & 20.31 & 0.598 & 0.397 & 116.42 \\
            & DDRM & 21.26 & 0.564 & 0.443 & 146.89 \\
            & DDNM & \textbf{28.06} & \underline{0.703} & 0.278 & 81.43 \\
            & DCDP & -- & -- & -- & -- \\
            & FPS-SMC & 23.91 & 0.601 & 0.387 & 91.72 \\
            & DiffPIR & 22.80 & -- & 0.355 & 93.36 \\
            & \epscell{EPS, NFE=100} & \epscell{24.34} & \epscell{0.640} & \epscell{\underline{0.224}} & \epscell{\underline{72.38}} \\
            & \epscell{EPS, NFE=20} & \epscell{24.92} & \epscell{0.666} & \epscell{\textbf{0.222}} & \epscell{\textbf{71.31}} \\
            & \epscell{EPS, NFE=1\dag} & \epscell{25.90} & \epscell{\textbf{0.704}} & \epscell{0.249} & \epscell{100.10} \\
            \midrule
            \multirow{8}{*}{Motion deblur}
            & DAPS & \textbf{27.86} & \textbf{0.766} & \textbf{0.196} & \textbf{61.83} \\
            & DPS & 18.96 & 0.629 & 0.423 & 137.81 \\
            & DCDP & -- & -- & -- & -- \\
            & FPS-SMC & \underline{24.52} & 0.647 & 0.326 & 87.43 \\
            & DiffPIR & 24.01 & -- & 0.366 & 94.63 \\
            & \epscell{EPS, NFE=100} & \epscell{23.65} & \epscell{0.616} & \epscell{0.254} & \epscell{79.79} \\
            & \epscell{EPS, NFE=20} & \epscell{24.39} & \epscell{\underline{0.655}} & \epscell{\underline{0.240}} & \epscell{\underline{75.52}} \\
            & \epscell{EPS, NFE=1\dag} & \epscell{24.12} & \epscell{0.647} & \epscell{0.318} & \epscell{137.59} \\
            \bottomrule
        \end{tabular}%
    }
    \end{minipage}
    \hfill
    \begin{minipage}[t]{0.49\textwidth}
    \centering
    \resizebox{\linewidth}{!}{%
        \begin{tabular}{llcccc}
            \toprule
            Task & Method & PSNR $\uparrow$ & SSIM $\uparrow$ & LPIPS $\downarrow$ & FID $\downarrow$ \\
            \midrule
            \multirow{9}{*}{Inpaint (box)}
            & DAPS & 21.43 & 0.725 & 0.214 & 109.85 \\
            & DPS & 18.94 & 0.722 & 0.257 & 126.52 \\
            & DDRM & 18.63 & 0.733 & 0.254 & 116.37 \\
            & DDNM & 21.64 & 0.748 & 0.319 & 103.97 \\
            & DCDP & -- & -- & -- & -- \\
            & FPS-SMC & \textbf{22.16} & 0.726 & 0.208 & 111.58 \\
            & \epscell{EPS, NFE=100} & \epscell{19.57} & \epscell{0.785} & \epscell{\underline{0.154}} & \epscell{\underline{100.30}} \\
            & \epscell{EPS, NFE=20} & \epscell{20.12} & \epscell{\underline{0.795}} & \epscell{\textbf{0.149}} & \epscell{\textbf{94.22}} \\
            & \epscell{EPS, NFE=1\dag} & \epscell{\underline{22.13}} & \epscell{\textbf{0.809}} & \epscell{0.177} & \epscell{113.50} \\
            \midrule
            \multirow{8}{*}{Inpaint (random)}
            & DAPS & 28.44 & 0.775 & 0.135 & 54.25 \\
            & DPS & 23.52 & 0.745 & 0.297 & 87.53 \\
            & DDNM & \textbf{31.16} & \underline{0.841} & 0.191 & 63.84 \\
            & DCDP & -- & -- & -- & -- \\
            & FPS-SMC & 24.52 & 0.701 & 0.316 & 79.12 \\
            & \epscell{EPS, NFE=100} & \epscell{27.26} & \epscell{0.817} & \epscell{0.113} & \epscell{\underline{21.39}} \\
            & \epscell{EPS, NFE=20} & \epscell{27.90} & \epscell{0.836} & \epscell{\underline{0.104}} & \epscell{\textbf{20.17}} \\
            & \epscell{EPS, NFE=1\dag} & \epscell{\underline{29.73}} & \epscell{\textbf{0.876}} & \epscell{\textbf{0.103}} & \epscell{24.03} \\
            \bottomrule
        \end{tabular}%
    }
    \end{minipage}
\end{table*}

Figure~\ref{fig:qual:imagenet256} shows qualitative reconstructions of EPS on ImageNet $256{\times}256$ across the five inverse problems, with one example observation per task. EPS recovers sharp object structure under aggressive random inpainting and box inpainting, and produces texture and edge geometry consistent with the measurement on super-resolution and deblurring at this higher resolution.

\begin{figure*}[ht!]
\centering
\includegraphics[width=0.9\textwidth]{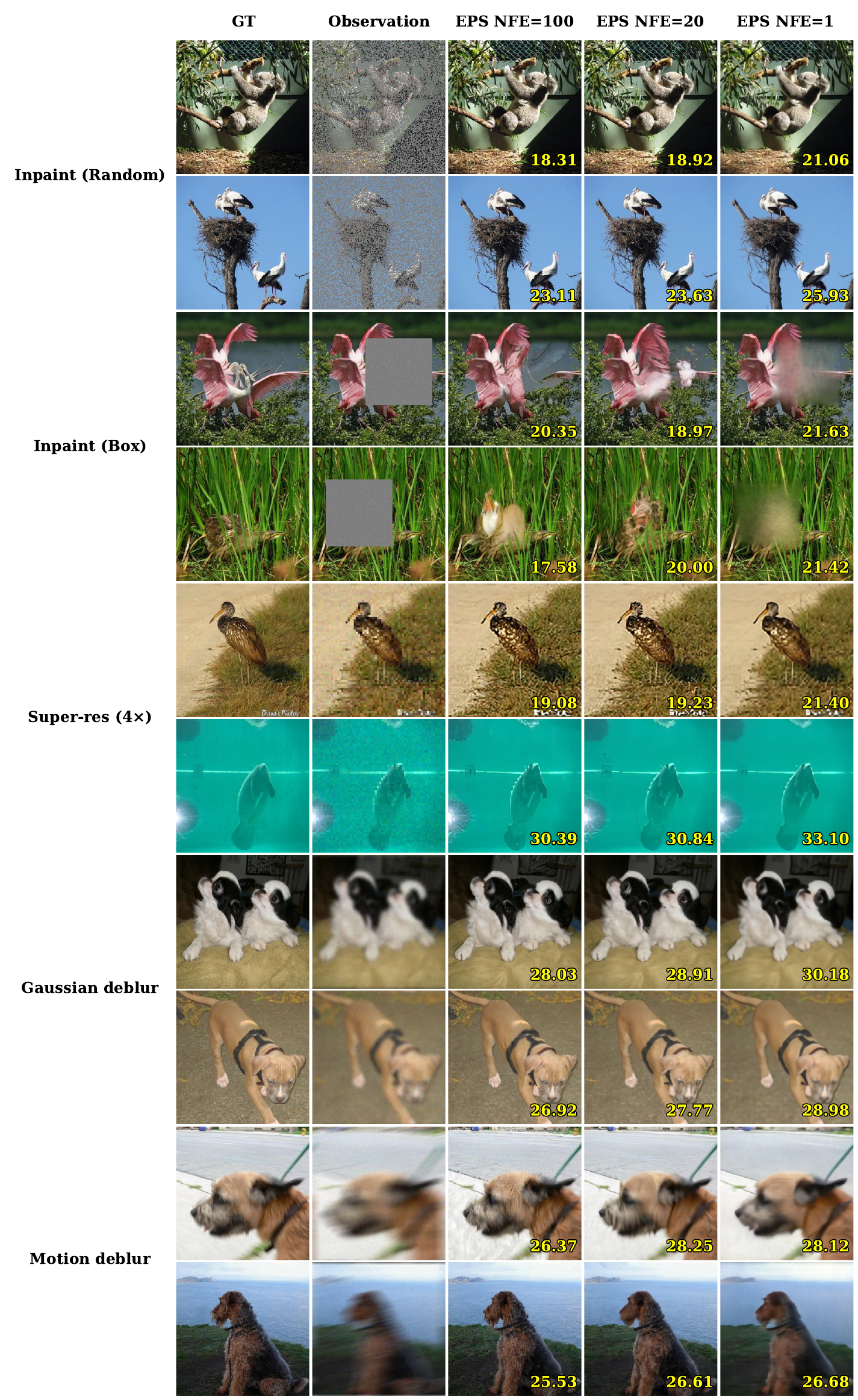}
\caption{\textbf{Qualitative reconstructions on ImageNet-256.} EPS reconstructions across the five inverse problems on ImageNet $256{\times}256$. Numbers in the bottom-right corner of each panel are per-image PSNR. EPS preserves operator-consistent structure under aggressive inpainting and produces sharp texture and edges on super-resolution and deblurring at this higher resolution.}
\label{fig:qual:imagenet256}
\end{figure*}

\clearpage
\subsection{Palette vs.\ EPS at Matched NFE}
\label{app:additional-experiments:palette-1step}
The one-step posterior-mean check in Section~\ref{app:additional-experiments:onestep} showed that a single Tweedie call recovers most of the multi-step PSNR/SSIM gain. Here we ask the matched question for Palette: does the same one-step shortcut close the gap to EPS? Concretely, we evaluate Palette $[x_t,y,t]$ and EPS $[\mu_\star,y,t]$ at NFE$=$1 (a single direct denoiser call at $\sigma_{\max}{=}80$, no sampler loop) and at NFE$=$100 (EDM Euler sampler), on both datasets and all five tasks.

Table~\ref{tab:palette-1step} reports this comparison on FFHQ-64 (top) and ImageNet-64 (bottom). Two patterns dominate. First, at NFE$=$1 the two methods coincide, as predicted by Observation~\ref{obs:one-step-limit}: at $\sigma_{\max}$ the noisy state carries vanishing information about $x_0$, so both the Palette input $(x_t,y)$ and the EPS input $(\mu_\star,y)$ are informationally equivalent to $y$ alone, and both networks estimate $\mathbb{E}[x_0|y]$; the small differences here are training noise rather than methodological difference. Second, at NFE$=$100 the EPS pivot becomes informative and EPS dominates Palette across distortion and distributional metrics on every task. Both 1-NFE rows trade perceptual quality for distortion: PSNR/SSIM jump sharply because the conditional mean is the MMSE-optimal estimator under squared error, but FID, LPIPS, CRPS, and MMD all degrade because no actual sample is produced.
\begin{table*}[ht!]
    \centering
    \scriptsize
    \setlength{\tabcolsep}{4pt}
    \caption{\textbf{Palette vs.\ EPS at NFE$=$1 and NFE$=$100.} One-step (single Tweedie call at $\sigma_{\max}{=}80$) vs.\ 100-step (EDM Euler) generation for Palette $[x_t,y,t]$ and EPS $[\mu_\star,y,t]$ on FFHQ-64 (top) and ImageNet-64 (bottom). At NFE$=$1, the two methods receive nearly identical inputs since $\mu_\star \approx x_t$ at $\sigma_{\max}$; at NFE$=$100 the pivot becomes informative and EPS dominates across distortion and distributional metrics. Best in \textbf{bold}, second-best \underline{underlined}; EPS rows highlighted in light pink.}
    \label{tab:palette-1step}
    \resizebox{\textwidth}{!}{%
        \begin{tabular}{llcccccccc}
            \toprule
            Task & Method & PSNR $\uparrow$ & SSIM $\uparrow$ & LPIPS $\downarrow$ & FID $\downarrow$ & MMD-pix $\downarrow$ & MMD-Inc $\downarrow$ & CRPS-pix $\downarrow$ & CRPS-Inc $\downarrow$ \\
            \midrule
            \multicolumn{10}{c}{\textit{FFHQ $64{\times}64$}} \\
            \midrule
            \multirow{4}{*}{Average}
            & Palette (NFE=100) & 26.03 & 0.8590 & \underline{0.0626} & \underline{31.50} & \underline{-6.6e-03} & \underline{-4.7e-03} & \underline{3.43} & \underline{3.86} \\
            & Palette (NFE=1) & \textbf{27.77} & \textbf{0.8920} & 0.0638 & 41.93 & -6.4e-03 & 8.0e-03 & 4.36 & 6.21 \\
            & \epscell{EPS (NFE=1)} & \epscell{\underline{27.34}} & \epscell{\underline{0.8887}} & \epscell{0.0699} & \epscell{46.30} & \epscell{-6.3e-03} & \epscell{1.1e-02} & \epscell{4.67} & \epscell{6.91} \\
            & \epscell{EPS (NFE=100)} & \epscell{26.69} & \epscell{0.8661} & \epscell{\textbf{0.0590}} & \epscell{\textbf{29.94}} & \epscell{\textbf{-6.7e-03}} & \epscell{\textbf{-4.9e-03}} & \epscell{\textbf{3.24}} & \epscell{\textbf{3.73}} \\
            \midrule
            \multirow{4}{*}{Random inpaint}
            & Palette (NFE=100) & 25.76 & 0.8809 & 0.0593 & \underline{33.30} & \underline{-6.7e-03} & \underline{-4.8e-03} & \underline{3.31} & \underline{3.87} \\
            & Palette (NFE=1) & \underline{27.58} & \underline{0.9095} & 0.0551 & 39.20 & -6.5e-03 & 1.1e-03 & 4.24 & 5.96 \\
            & \epscell{EPS (NFE=1)} & \epscell{\textbf{27.90}} & \epscell{\textbf{0.9148}} & \epscell{\underline{0.0549}} & \epscell{39.71} & \epscell{-6.6e-03} & \epscell{2.2e-03} & \epscell{3.94} & \epscell{5.75} \\
            & \epscell{EPS (NFE=100)} & \epscell{26.16} & \epscell{0.8879} & \epscell{\textbf{0.0533}} & \epscell{\textbf{31.87}} & \epscell{\textbf{-6.7e-03}} & \epscell{\textbf{-5.0e-03}} & \epscell{\textbf{3.16}} & \epscell{\textbf{3.75}} \\
            \midrule
            \multirow{4}{*}{Box inpaint}
            & Palette (NFE=100) & 24.18 & 0.8426 & \underline{0.0577} & \underline{25.09} & \underline{-6.6e-03} & \underline{-5.3e-03} & \underline{4.02} & \underline{3.47} \\
            & Palette (NFE=1) & \textbf{26.14} & \textbf{0.8727} & 0.0639 & 32.60 & -6.3e-03 & -1.2e-03 & 5.01 & 5.51 \\
            & \epscell{EPS (NFE=1)} & \epscell{\underline{26.01}} & \epscell{\underline{0.8706}} & \epscell{0.0644} & \epscell{32.90} & \epscell{-6.4e-03} & \epscell{-8.5e-04} & \epscell{5.05} & \epscell{5.62} \\
            & \epscell{EPS (NFE=100)} & \epscell{24.23} & \epscell{0.8448} & \epscell{\textbf{0.0567}} & \epscell{\textbf{24.74}} & \epscell{\textbf{-6.6e-03}} & \epscell{\textbf{-5.4e-03}} & \epscell{\textbf{4.01}} & \epscell{\textbf{3.45}} \\
            \midrule
            \multirow{4}{*}{Super-res ($4{\times}$)}
            & Palette (NFE=100) & 21.95 & 0.7220 & 0.1273 & \textbf{49.28} & \underline{-6.3e-03} & \underline{-3.0e-03} & \textbf{5.22} & \textbf{5.29} \\
            & Palette (NFE=1) & \textbf{24.23} & \textbf{0.8009} & \textbf{0.1255} & 77.61 & -5.6e-03 & 4.1e-02 & 6.56 & 9.31 \\
            & \epscell{EPS (NFE=1)} & \epscell{\underline{24.21}} & \epscell{\underline{0.7999}} & \epscell{0.1265} & \epscell{77.97} & \epscell{-5.6e-03} & \epscell{4.1e-02} & \epscell{6.57} & \epscell{9.27} \\
            & \epscell{EPS (NFE=100)} & \epscell{21.96} & \epscell{0.7232} & \epscell{\underline{0.1262}} & \epscell{\underline{49.29}} & \epscell{\textbf{-6.3e-03}} & \epscell{\textbf{-3.0e-03}} & \epscell{\underline{5.23}} & \epscell{\underline{5.30}} \\
            \midrule
            \multirow{4}{*}{Gaussian deblur}
            & Palette (NFE=100) & 30.47 & 0.9397 & 0.0286 & \underline{21.63} & \underline{-6.9e-03} & \underline{-5.6e-03} & \underline{1.91} & \underline{3.06} \\
            & Palette (NFE=1) & \textbf{32.04} & \textbf{0.9529} & \underline{0.0285} & 24.49 & -6.8e-03 & -1.9e-03 & 2.40 & 4.61 \\
            & \epscell{EPS (NFE=1)} & \epscell{\underline{30.86}} & \epscell{\underline{0.9503}} & \epscell{0.0369} & \epscell{31.39} & \epscell{-6.6e-03} & \epscell{2.4e-03} & \epscell{3.21} & \epscell{6.15} \\
            & \epscell{EPS (NFE=100)} & \epscell{30.82} & \epscell{0.9408} & \epscell{\textbf{0.0273}} & \epscell{\textbf{20.68}} & \epscell{\textbf{-6.9e-03}} & \epscell{\textbf{-5.7e-03}} & \epscell{\textbf{1.84}} & \epscell{\textbf{2.99}} \\
            \midrule
            \multirow{4}{*}{Motion deblur}
            & Palette (NFE=100) & 27.79 & 0.9099 & \underline{0.0404} & \underline{28.23} & \underline{-6.8e-03} & \underline{-5.0e-03} & \underline{2.67} & \underline{3.59} \\
            & Palette (NFE=1) & \underline{28.87} & \underline{0.9240} & 0.0461 & 35.76 & -6.6e-03 & 1.4e-03 & 3.59 & 5.68 \\
            & \epscell{EPS (NFE=1)} & \epscell{27.73} & \epscell{0.9078} & \epscell{0.0669} & \epscell{49.55} & \epscell{-6.3e-03} & \epscell{1.2e-02} & \epscell{4.60} & \epscell{7.76} \\
            & \epscell{EPS (NFE=100)} & \epscell{\textbf{30.27}} & \epscell{\textbf{0.9339}} & \epscell{\textbf{0.0311}} & \epscell{\textbf{23.10}} & \epscell{\textbf{-6.9e-03}} & \epscell{\textbf{-5.4e-03}} & \epscell{\textbf{1.94}} & \epscell{\textbf{3.18}} \\
            \midrule
            \multicolumn{10}{c}{\textit{ImageNet $64{\times}64$}} \\
            \midrule
            \multirow{4}{*}{Average}
            & Palette (NFE=100) & 24.32 & 0.7673 & \underline{0.1124} & \underline{82.57} & \underline{-6.4e-03} & \underline{-4.4e-03} & \underline{4.35} & \underline{5.54} \\
            & Palette (NFE=1) & \textbf{26.73} & \textbf{0.8289} & 0.1215 & 104.38 & -6.1e-03 & 4.3e-03 & 5.35 & 9.37 \\
            & \epscell{EPS (NFE=1)} & \epscell{\underline{25.67}} & \epscell{\underline{0.8165}} & \epscell{0.1318} & \epscell{110.95} & \epscell{-5.6e-03} & \epscell{4.5e-03} & \epscell{5.91} & \epscell{9.96} \\
            & \epscell{EPS (NFE=100)} & \epscell{24.53} & \epscell{0.7712} & \epscell{\textbf{0.1103}} & \epscell{\textbf{81.46}} & \epscell{\textbf{-6.4e-03}} & \epscell{\textbf{-4.4e-03}} & \epscell{\textbf{4.27}} & \epscell{\textbf{5.48}} \\
            \midrule
            \multirow{4}{*}{Random inpaint}
            & Palette (NFE=100) & 24.09 & 0.7869 & 0.1011 & \underline{81.88} & \underline{-6.5e-03} & \underline{-4.4e-03} & \underline{4.16} & \underline{5.52} \\
            & Palette (NFE=1) & \underline{26.50} & \underline{0.8550} & \textbf{0.0926} & 91.79 & -6.4e-03 & 4.3e-04 & 5.16 & 8.90 \\
            & \epscell{EPS (NFE=1)} & \epscell{\textbf{26.60}} & \epscell{\textbf{0.8580}} & \epscell{\underline{0.0933}} & \epscell{88.59} & \epscell{-6.1e-03} & \epscell{-4.8e-04} & \epscell{4.98} & \epscell{8.32} \\
            & \epscell{EPS (NFE=100)} & \epscell{24.34} & \epscell{0.7948} & \epscell{0.0979} & \epscell{\textbf{79.60}} & \epscell{\textbf{-6.5e-03}} & \epscell{\textbf{-4.5e-03}} & \epscell{\textbf{4.04}} & \epscell{\textbf{5.41}} \\
            \midrule
            \multirow{4}{*}{Box inpaint}
            & Palette (NFE=100) & 21.12 & 0.7541 & \underline{0.1218} & \underline{92.73} & \underline{-6.1e-03} & \underline{-4.1e-03} & \underline{5.92} & \underline{5.93} \\
            & Palette (NFE=1) & \textbf{23.60} & \textbf{0.7926} & 0.1539 & 131.54 & -5.5e-03 & 8.7e-03 & 7.19 & 10.67 \\
            & \epscell{EPS (NFE=1)} & \epscell{\underline{23.60}} & \epscell{\underline{0.7908}} & \epscell{0.1514} & \epscell{129.11} & \epscell{-5.6e-03} & \epscell{7.3e-03} & \epscell{7.16} & \epscell{10.38} \\
            & \epscell{EPS (NFE=100)} & \epscell{21.24} & \epscell{0.7569} & \epscell{\textbf{0.1196}} & \epscell{\textbf{91.07}} & \epscell{\textbf{-6.1e-03}} & \epscell{\textbf{-4.2e-03}} & \epscell{\textbf{5.87}} & \epscell{\textbf{5.84}} \\
            \midrule
            \multirow{4}{*}{Super-res ($4{\times}$)}
            & Palette (NFE=100) & 20.24 & 0.5364 & \underline{0.2220} & \textbf{128.76} & \textbf{-5.9e-03} & \underline{-2.8e-03} & \textbf{6.50} & \textbf{7.33} \\
            & Palette (NFE=1) & \textbf{22.79} & \textbf{0.6538} & 0.2452 & 185.15 & -4.9e-03 & 1.9e-02 & 8.05 & 13.63 \\
            & \epscell{EPS (NFE=1)} & \epscell{\underline{22.78}} & \epscell{\underline{0.6530}} & \epscell{0.2455} & \epscell{182.92} & \epscell{-5.0e-03} & \epscell{2.0e-02} & \epscell{8.06} & \epscell{13.47} \\
            & \epscell{EPS (NFE=100)} & \epscell{20.25} & \epscell{0.5369} & \epscell{\textbf{0.2207}} & \epscell{\underline{128.80}} & \epscell{\underline{-5.9e-03}} & \epscell{\textbf{-2.8e-03}} & \epscell{\underline{6.52}} & \epscell{\underline{7.35}} \\
            \midrule
            \multirow{4}{*}{Gaussian deblur}
            & Palette (NFE=100) & 29.15 & 0.9010 & 0.0491 & 46.62 & \underline{-6.8e-03} & \textbf{-5.6e-03} & \underline{2.26} & \textbf{4.11} \\
            & Palette (NFE=1) & \textbf{31.73} & \textbf{0.9407} & \textbf{0.0441} & \textbf{43.40} & \textbf{-6.8e-03} & -4.2e-03 & 2.74 & 6.15 \\
            & \epscell{EPS (NFE=1)} & \epscell{28.82} & \epscell{\underline{0.9194}} & \epscell{0.0606} & \epscell{56.53} & \epscell{-5.1e-03} & \epscell{-3.5e-03} & \epscell{4.09} & \epscell{7.73} \\
            & \epscell{EPS (NFE=100)} & \epscell{\underline{29.18}} & \epscell{0.9015} & \epscell{\underline{0.0486}} & \epscell{\underline{46.55}} & \epscell{-6.8e-03} & \epscell{\underline{-5.6e-03}} & \epscell{\textbf{2.25}} & \epscell{\underline{4.11}} \\
            \midrule
            \multirow{4}{*}{Motion deblur}
            & Palette (NFE=100) & 27.02 & 0.8582 & \underline{0.0680} & \underline{62.86} & \underline{-6.7e-03} & \underline{-5.0e-03} & \underline{2.93} & \underline{4.81} \\
            & Palette (NFE=1) & \textbf{29.04} & \textbf{0.9023} & 0.0715 & 70.01 & -6.7e-03 & -2.3e-03 & 3.63 & 7.51 \\
            & \epscell{EPS (NFE=1)} & \epscell{26.56} & \epscell{0.8613} & \epscell{0.1079} & \epscell{97.59} & \epscell{-6.0e-03} & \epscell{-7.3e-04} & \epscell{5.25} & \epscell{9.91} \\
            & \epscell{EPS (NFE=100)} & \epscell{\underline{27.62}} & \epscell{\underline{0.8661}} & \epscell{\textbf{0.0647}} & \epscell{\textbf{61.29}} & \epscell{\textbf{-6.8e-03}} & \epscell{\textbf{-5.1e-03}} & \epscell{\textbf{2.69}} & \epscell{\textbf{4.70}} \\
            \bottomrule
        \end{tabular}%
    }
\end{table*}

\clearpage
\subsection{Runtime Analysis}
\label{app:additional-experiments:runtime}
We measure single-image wall-clock sampling latency on ImageNet $64{\times}64$ at NFE$=$100 with batch size 1. All methods use the same EDM-ADM~\cite{karras2022elucidating} denoiser checkpoint (\texttt{edm-imagenet-64x64-cond-adm.pkl}, $\sim$296M parameters) on a single NVIDIA B200 GPU, with class labels set to the evaluation-set ground-truth ImageNet-1k classes. Each cell in Table~\ref{tab:runtime} is the mean of five independent sampling runs after two warm-up runs that amortise CUDA kernel JIT and cuDNN auto-tuning. All methods use the Euler ODE schedule (\texttt{second\_order=False}) so NFE equals the number of sampler steps; for DAPS we set $\text{annealing\_steps}{=}20$ and $\text{ode\_steps}{=}5$ so total NFE matches.

The dominant cost in all methods is the U-Net forward, and for DPS and $\Pi$GDM the U-Net backward as well. EDM (uncond.) runs the bare pretrained denoiser with no measurement-aware updates and is task-independent. DPS and $\Pi$GDM require a backward pass per step (likelihood gradient / Jacobian-vector product), $\sim 2.2\times$ the EDM unconditional cost. DAPS runs a nested ODE rollout plus Langevin correction at each annealing step, $\sim 2.0\times$ unconditional cost. DDNM and MPGD ($\sim 1.1\times$) add only a closed-form nullspace or manifold projection on top of a single denoiser forward. EPS is the fastest sampler in the comparison: at batch 1 it runs $\sim 0.8\times$ the wall-clock of EDM unconditional. EPS's freshly constructed preconditioning wrapper avoids deserialisation overhead present in the pretrained EDMPrecond pickle, while the structured solve for $\mu_\star$ (FFT for deblurring; element-wise for inpainting and super-resolution) is sub-millisecond. EPS runtime is essentially task-independent: all five tasks land within $\pm 0.04$~s of the average. Palette shares EPS's architecture and per-step forward cost, so its runtime is well approximated by the EPS row. At larger batch sizes the per-image gap closes (at batch 8, EPS is only $\sim 3.5\%$ slower than EDM unconditional), so EPS is the right pick for low-latency single-image inference and is essentially free vs.\ EDM unconditional in throughput terms.

\begin{table}[ht!]
\centering
\scriptsize
\setlength{\tabcolsep}{4pt}
\caption{\textbf{EPS matches the bare denoiser in wall-clock cost.} Per-image sampling latency (s, lower is better) on ImageNet-64 at NFE$=$100, batch size 1, on a single B200 GPU. Among methods that solve the inverse problem, EPS is fastest on every task and on average, at essentially the same cost as the bare EDM unconditional sampler ($1.006\times$). Sampling-based baselines that require a backward pass (DPS, $\Pi$GDM) are $\sim 2.3\times$ slower; nested-rollout methods (DAPS) are $\sim 2.2\times$ slower. \dag\ EDM (uncond.) is shown as a reference: it runs the bare pretrained denoiser without any measurement-aware update, so it does not actually solve the inverse problem.}
\label{tab:runtime}
\resizebox{\textwidth}{!}{%
\begin{tabular}{lccccccc}
\toprule
Method & Inpaint (Random) & Inpaint (Box) & Super-res ($4{\times}$) & Gaussian deblur & Motion deblur & Avg & $\times$ EDM \\
\midrule
\textit{EDM (uncond.)\dag} & \textit{1.87} & \textit{1.87} & \textit{1.87} & \textit{1.87} & \textit{1.87} & \textit{1.87} & \textit{$1.00\times$} \\
\midrule
DPS & 4.61 & 4.17 & 4.05 & 4.14 & 4.42 & 4.28 & $2.29\times$ \\
DAPS & 3.98 & 4.02 & 3.99 & 4.19 & 4.20 & 4.07 & $2.18\times$ \\
DDNM & 2.52 & 2.50 & 2.49 & 2.51 & 2.51 & 2.51 & $1.34\times$ \\
$\Pi$GDM & 4.55 & 4.55 & 4.53 & 4.54 & 4.55 & 4.54 & $2.43\times$ \\
MPGD & 2.51 & 2.51 & 2.52 & 2.53 & 2.54 & 2.52 & $1.35\times$ \\
EPS (ours) & \textbf{1.87} & \textbf{1.87} & \textbf{1.88} & \textbf{1.89} & \textbf{1.89} & \textbf{1.88} & $\mathbf{1.006\times}$ \\
\bottomrule
\end{tabular}%
}
\end{table}

\clearpage
\subsection{Posterior diversity.} Figures~\ref{fig:diversity:ffhq} and~\ref{fig:diversity:imagenet} show four EPS reconstructions of the same observation drawn with independent latent seeds, alongside the ground truth, on box inpainting and $4{\times}$ super-resolution. Samples agree on observed structure while differing in the unobserved directions (skin texture, hair detail, background, occluded foreground content) - the qualitative signature of a calibrated posterior under a non-trivial operator nullspace.

\begin{figure*}[ht!]
\centering
\includegraphics[width=\textwidth]{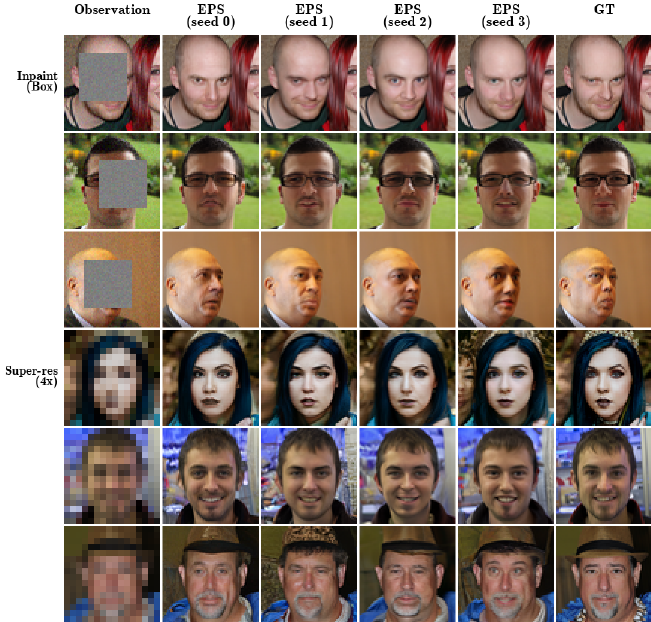}
\caption{\textbf{Posterior diversity from EPS on FFHQ-64.} Four reconstructions per observation drawn with independent latent seeds, alongside the ground truth, on box inpainting and $4{\times}$ super-resolution. Samples agree on observed structure while differing in unobserved directions (skin texture, hair detail, background) — the qualitative signature of a calibrated posterior under a non-trivial operator nullspace.}
\label{fig:diversity:ffhq}
\end{figure*}

\begin{figure*}[ht!]
\centering
\includegraphics[width=\textwidth]{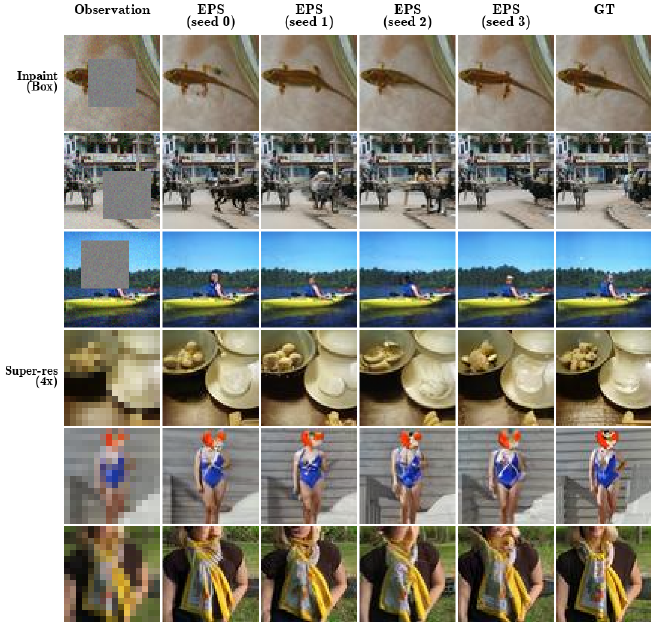}
\caption{\textbf{Posterior diversity from EPS on ImageNet-64.} Same layout as Fig.~\ref{fig:diversity:ffhq}. Diversity is concentrated in the operator nullspace: occluded foreground content varies under box inpainting, while sharp high-frequency detail varies under $4{\times}$ super-resolution.}
\label{fig:diversity:imagenet}
\end{figure*}

\clearpage
\subsection{Sampling Budget}
\label{app:additional-experiments:nfe-qualitative}
Figure~\ref{fig:nfe-qualitative} compares EPS reconstructions at NFE$=$1, 20, and 100 on the same observation across all five tasks, on both ImageNet-64 (left) and FFHQ-64 (right). The NFE$=$1 column is the deterministic high-noise posterior-mean limit (Section~\ref{sec:method:one-step}), which is MMSE-optimal in pixel space and yields the highest per-image PSNR. The NFE$=$20 and NFE$=$100 columns target posterior samples and trade pointwise fidelity for distributional sharpness, in line with the perception-distortion pattern visible in Table~\ref{tab:onestep}.

\begin{figure*}[ht!]
\centering
\includegraphics[width=\textwidth]{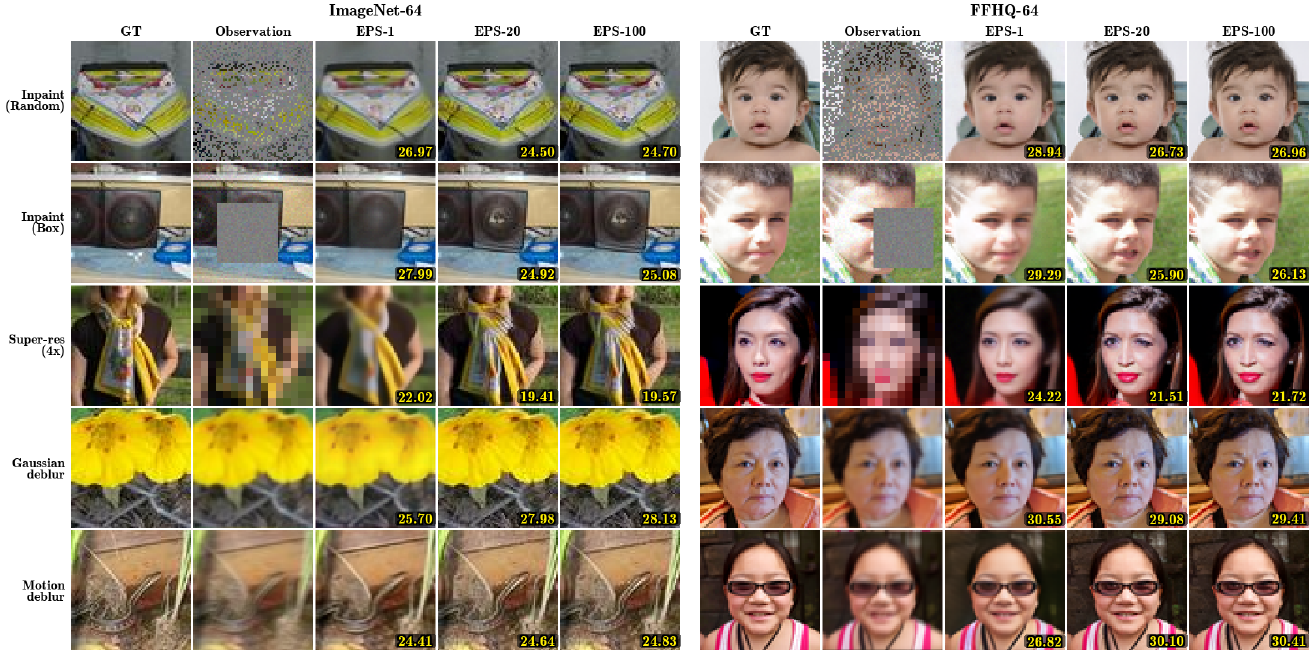}
\caption{\textbf{EPS reconstructions at varying sampling budgets.} ImageNet-64 (left) and FFHQ-64 (right) across the five inverse problems. For each observation, columns show NFE$=$1, 20, and 100. The NFE$=$1 column is the deterministic high-noise posterior-mean limit (Section~\ref{sec:method:one-step}), MMSE-optimal in pixel space; NFE$=$20 and NFE$=$100 target posterior samples and trade pointwise fidelity for distributional sharpness.}
\label{fig:nfe-qualitative}
\end{figure*}

\clearpage
\section{Baseline Configurations}
\label{app:baselines}
Table~\ref{tab:baseline-configs} reports per-task hyperparameters for the sampling-based and training-based baselines (DPS, DAPS, DDNM, $\Pi$GDM, MPGD, Palette), including NFE, step size or guidance scale, noise schedule, projection or correction rule, and additional notes. Each baseline is tuned on a disjoint validation split, and the values reported are those used in the main and appendix tables.

\begin{table}[ht!]
\centering
\scriptsize
\caption{\textbf{Baseline hyperparameter configurations.} Per-task
hyperparameters for the sampling-based (DPS, DAPS, DDNM, $\Pi$GDM,
MPGD) and training-based (Palette) baselines. NFE is the number of
denoiser forward passes per image (DAPS NFE${=}\text{annealing}{\times}
\text{ode}$ Euler steps). All samplers use the EDM VE noise
schedule with $\sigma_{\max}{=}80$, $\sigma_{\min}{=}0.002$, $\rho{=}7$
unless noted. Observation noise is $\sigma_y{=}0.05$ for every cell.
Identical settings are used on ImageNet-64 and FFHQ-64; only the
underlying denoiser changes.}
\label{tab:baseline-configs}
\resizebox{\textwidth}{!}{
\begin{tabular}{lllcccc}
\toprule
Task & Method & NFE & Step size / guidance & Noise schedule & Projection / correction & Notes \\
\midrule
Random inpaint & DPS              & 250  & $\zeta'{=}5.0$  & EDM VE                                   & likelihood grad, $\zeta'/\Vert r\Vert$ & Euler, $S_{\text{churn}}{=}0$ \\
Random inpaint & DAPS             & 500  & mcmc $5{\times}10^{-4}$ & EDM VE                          & nested ODE + Langevin                  & annealing${=}100$, ode${=}5$, mcmc${=}100$, Euler \\
Random inpaint & DDNM             & 100  & ---             & EDM VE                                   & null-space $\lambda_t{=}\sigma_t^2/(\sigma_t^2{+}\sigma_y^2)$ & DDIM-VE, $\eta{=}0.85$ \\
Random inpaint & $\Pi$GDM         & 100  & $\zeta{=}1.0$   & EDM VE                                   & VJP guidance                           & DDIM-VE, $\eta{=}1.0$ \\
Random inpaint & MPGD             & 100  & $\zeta{=}20.0$  & EDM VE                                   & manifold projection                    & DDIM-VE, $\eta{=}0.85$ \\
Random inpaint & Palette          & 100  & ---             & EDM VE                                   & ---                                    & training-based; per-task EPS backbone \\
\midrule
Box inpaint    & DPS              & 250  & $\zeta'{=}10.0$ & EDM VE                                   & likelihood grad, $\zeta'/\Vert r\Vert$ & Euler, $S_{\text{churn}}{=}0$ \\
Box inpaint    & DAPS             & 500  & mcmc $5{\times}10^{-4}$ & EDM VE                          & nested ODE + Langevin                  & annealing${=}100$, ode${=}5$, mcmc${=}100$, Euler \\
Box inpaint    & DDNM             & 100  & ---             & EDM VE                                   & null-space $\lambda_t{=}\sigma_t^2/(\sigma_t^2{+}\sigma_y^2)$ & DDIM-VE, $\eta{=}0.85$ \\
Box inpaint    & $\Pi$GDM         & 100  & $\zeta{=}1.0$   & EDM VE                                   & VJP guidance                           & DDIM-VE, $\eta{=}1.0$ \\
Box inpaint    & MPGD             & 100  & $\zeta{=}15.0$  & EDM VE                                   & manifold projection                    & DDIM-VE, $\eta{=}0.85$ \\
Box inpaint    & Palette          & 100  & ---             & EDM VE                                   & ---                                    & training-based; per-task EPS backbone \\
\midrule
Super-res ($4{\times}$) & DPS     & 250  & $\zeta'{=}10.0$ & EDM VE                                   & likelihood grad, $\zeta'/\Vert r\Vert$ & Euler, $S_{\text{churn}}{=}0$ \\
Super-res ($4{\times}$) & DAPS    & 500  & mcmc $9{\times}10^{-4}$ & EDM VE, $\sigma_{\max}{=}30$, $\sigma_{\min}{=}0.1$ & nested ODE + Langevin       & annealing${=}100$, ode${=}5$, mcmc${=}100$, Euler \\
Super-res ($4{\times}$) & DDNM    & 100  & ---             & EDM VE                                   & null-space $\lambda_t{=}\sigma_t^2/(\sigma_t^2{+}\sigma_y^2)$ & DDIM-VE, $\eta{=}0.85$, $A^{\dagger}{=}$nearest upsample \\
Super-res ($4{\times}$) & $\Pi$GDM & 100 & $\zeta{=}1.0$   & EDM VE                                   & VJP guidance                           & DDIM-VE, $\eta{=}1.0$ \\
Super-res ($4{\times}$) & MPGD    & 100  & $\zeta{=}30.0$  & EDM VE                                   & manifold projection                    & DDIM-VE, $\eta{=}0.85$ \\
Super-res ($4{\times}$) & Palette & 100  & ---             & EDM VE                                   & ---                                    & training-based; per-task EPS backbone \\
\midrule
Gaussian deblur & DPS             & 250  & $\zeta'{=}1.5$  & EDM VE                                   & likelihood grad, $\zeta'/\Vert r\Vert$ & Euler, $S_{\text{churn}}{=}0$ \\
Gaussian deblur & DAPS            & 500  & mcmc $9{\times}10^{-4}$ & EDM VE, $\sigma_{\max}{=}30$, $\sigma_{\min}{=}0.1$ & nested ODE + Langevin       & annealing${=}100$, ode${=}5$, mcmc${=}100$, Euler \\
Gaussian deblur & DDNM            & 100  & ---             & EDM VE                                   & per-freq Wiener correction             & DDIM-VE, $\eta{=}0.85$, Wiener $\epsilon{=}10^{-3}$ \\
Gaussian deblur & $\Pi$GDM        & 100  & $\zeta{=}1.0$   & EDM VE                                   & VJP guidance                           & DDIM-VE, $\eta{=}1.0$ \\
Gaussian deblur & MPGD            & 100  & $\zeta{=}5.0$   & EDM VE                                   & manifold projection                    & DDIM-VE, $\eta{=}0.85$ \\
Gaussian deblur & Palette         & 100  & ---             & EDM VE                                   & ---                                    & training-based; per-task EPS backbone \\
\midrule
Motion deblur  & DPS              & 250  & $\zeta'{=}2.5$  & EDM VE                                   & likelihood grad, $\zeta'/\Vert r\Vert$ & Euler, $S_{\text{churn}}{=}0$ \\
Motion deblur  & DAPS             & 500  & mcmc $5{\times}10^{-4}$ & EDM VE                          & nested ODE + Langevin                  & annealing${=}100$, ode${=}5$, mcmc${=}100$, Euler \\
Motion deblur  & DDNM             & 100  & ---             & EDM VE                                   & per-freq Wiener correction             & DDIM-VE, $\eta{=}0.85$, Wiener $\epsilon{=}10^{-3}$ \\
Motion deblur  & $\Pi$GDM         & 100  & $\zeta{=}1.0$   & EDM VE                                   & VJP guidance                           & DDIM-VE, $\eta{=}1.0$ \\
Motion deblur  & MPGD             & 100  & $\zeta{=}9.0$   & EDM VE                                   & manifold projection                    & DDIM-VE, $\eta{=}0.85$ \\
Motion deblur  & Palette          & 100  & ---             & EDM VE                                   & ---                                    & training-based; per-task EPS backbone \\
\bottomrule
\end{tabular}
}
\end{table}

\clearpage
\section{Metric Definitions}
\label{app:metrics}
\paragraph{CRPS.} For a scalar target $z$ and a predictive distribution with CDF $F$, $\mathrm{CRPS}(F, z) = \int_\R (F(u) - \mathbf{1}\{u \geq z\})^2 \dd u$. For our multivariate settings we report the average per-coordinate CRPS, in pixel space (CRPS-pixel) and in the Inception feature space used for FID (CRPS-inception), each averaged over the evaluation images and over the posterior samples drawn for each observation~\cite{gneiting2014}.

\paragraph{MMD.} With a Gaussian kernel $k(u, v) = \exp(-\|u - v\|^2 / (2\ell^2))$, we estimate $\mathrm{MMD}^2(P, Q)$ between the EPS sample distribution $P$ and the empirical distribution $Q$ of the evaluation ground truths using the unbiased U-statistic of \citet{gretton2012kernel}. We report MMD-pixel (kernel in pixel space) and MMD-inception (kernel in Inception feature space). The bandwidth $\ell$ is chosen via the median heuristic on the combined sample.

\stopcontents[appendices]

\end{document}